%% file: main.tex
\documentclass[a4paper]{article}
\usepackage[margin=1.3in]{geometry} 

\usepackage[round]{natbib}
\usepackage[title]{appendix}

\input{commands.tex}

\title{Invariant Policy Learning: A Causal Perspective}

\author{Sorawit Saengkyongam,\,Nikolaj Thams,\, Jonas Peters, and Niklas Pfister}
\date{University of Copenhagen, Denmark \\[2ex]
Emails: \{ss, thams, jonas.peters, np\}@math.ku.dk}

\begin{document}

\maketitle

\begin{abstract}
Contextual bandit and reinforcement learning algorithms have been successfully used in various interactive learning systems such as online advertising, recommender systems, and dynamic pricing. However, they have yet to be widely adopted in high-stakes application domains, such as healthcare. 
One reason may be that existing approaches assume that the underlying mechanisms are static in the sense that they do not change over different environments. In many real-world systems, however, the mechanisms are subject to shifts across environments which may invalidate the static environment assumption. In this paper, we take a step toward tackling the problem of environmental shifts considering the framework of offline contextual bandits. We view the environmental shift problem through the lens of causality and propose multi-environment contextual bandits that allow for changes in the underlying mechanisms. We adopt the concept of invariance from the causality literature and introduce the notion of policy invariance. We argue that policy invariance is only relevant if unobserved variables are present and show that, in that case, an optimal invariant policy is guaranteed to generalize across environments under suitable assumptions. Our results establish concrete connections among causality, invariance, and contextual bandits.
\end{abstract}

\section{Introduction}
\label{sec:introduction}

The problem of 
learning decision-making policies lies at the heart of learning systems. To adopt these learning systems in high-stakes application domains such as personalized medicine or autonomous driving, it is crucial that the learned policies are reliable even in environments that have never been encountered before. In this paper, we consider the problem of learning policies that are robust with respect to shifts across environments. We consider this question in the setup of offline  contextual  bandits, which provides a mathematical framework for tackling the above learning problems. 

While recent studies in offline contextual bandits \citep{dudik2011doubly, BJ13, swaminathan2015counterfactual, swaminathan2015self, kallus2018balanced, athey2021policy, zhou2022offline} offer theoretical results and novel methodologies for policy learning from offline data, 
they primarily focus on a fixed-environment setting (from now on, we will refer to this as the equal distribution assumption) in which the underlying mechanisms do not change over time or over different environments. In practice, however, shifts between environments often occur, possibly invalidating the equal distribution assumption. In healthcare, for example, datasets from different hospitals may not come from the same underlying distribution. As a result, a learning agent that ignores environmental shifts may fail to generalize beyond the environments it was trained on.

In the supervised learning context, the environmental shift problem has been studied  under different names, such as domain generalization, covariate shift adaptation, distributional robustness or out-of-distribution generalization
\citep{sugiyama2012machine, muandet2013domain, NEURIPS2018_1d94108e, arjovsky2019invariant, christiansen2020causal}. In domain generalization, the goal is to develop learning algorithms that are robust to changes in the test distribution. Thus, a fundamental problem is how to characterize such changes. A promising direction relies on a causal framework to describe the changes through the concept of interventions \citep{Schoelkopf2012icml, Rojas2016, Magliacane2018, arjovsky2019invariant, christiansen2020causal}. 
A key insight is that while purely predictive methods perform best if test and training distributions coincide, 
causal models generalize to
arbitrarily strong interventions on the covariates because of the modularity property of structural causal models (see e.g., \cite{pearl2009causality}).

In real-world applications knowledge of the underlying causal graph and structural discrepancies between environments may not be available. 
In recent years, invariance-based methods have been exploited to learn the causal structure from data \citep{PBM16, Pfister2018jasa, HeinzeDeml2017}. In invariant causal prediction \citep{PBM16}, for example, one assumes that the data are collected from different environments, each of which describes different underlying data-generating mechanisms, and uses this heterogeneity to learn the causal parents of an outcome variable $Y$.
The underpinning assumption is the invariance assumption, which posits the existence of a set of covariates $X$ in which the mechanism between $X$ and $Y$ remains constant. A model based on such invariant covariates is guaranteed to generalize to all unseen environments.

Our paper delineates an explicit connection among causality, invariance, and the environmental shift problem in the context of contextual bandits.
We develop a causal framework for characterizing the environmental shift problem, and provide a practical and theoretically sound solution based on the proposed framework. 

Our contributions are threefold.
First, we propose a multi-environment contextual bandit framework that represents mechanisms underlying a contextual bandit problem by structural causal models (SCMs; \cite{pearl2009causality}). The framework allows for changes in environments and thereby relaxes the equal distribution %
assumption. We define environments as different perturbations on the underlying SCM, and we evaluate the policy according to its worst-case performance in all environments. 
Second, using the proposed framework, we generalize the invariance assumption used in methods such as invariant causal prediction and define invariance properties for policies that, under certain assumptions, guarantee generalizability to unseen environments. Third, we develop an offline method for testing invariance under distributional (policy) shifts, and provide an algorithm for finding an optimal invariant policy. In addition, we highlight a setting in which causality and invariance are not necessary for solving the environmental shift problem. This insight takes us closer to understanding what causality can offer in offline contextual bandits. 

The remainder of our paper is organized as follows. Sections~\ref{sec:related_work} and \ref{sec:off_cb}  briefly review related work and introduces the offline contextual bandit problem. Section~\ref{sec:multi-env} 
formally defines a causal framework for multi-environment contextual bandits and the main objective of our problem's formulation. %
Drawing on the proposed framework, Section~\ref{sec:inv-policy} introduces invariance properties for policies and provides the main theoretical contributions underpinning our solution for the environmental shift problem. Section~\ref{sec:learn_opt_policy} discusses the assumptions required to learn invariant policies from offline data and presents an algorithm for learning an optimal invariant policy. Section~\ref{sec:simulation} provides simulation experiments that empirically verify our theoretical results. In Section~\ref{sec:warfarin}, we apply our framework to a warfarin dosing study. %

\subsection{Related Work}\label{sec:related_work}

Our work is most closely related to the line of work studying invariance and generalizability for prediction tasks in i.i.d.\ settings 
mentioned above \citep{Rojas2016, Magliacane2018, arjovsky2019invariant, christiansen2020causal, pfister2021SR}.
The environmental shift problem is also related to the problem of transportability in causal inference \citep{pearl2011transportability, bareinboim2014transportability, bareinboim2016causal, subbaswamy2019preventing, lee2020generalized, correa2020general} which aims to generalize causal findings from source environments to a target environment. Our work differs from the transportability literature: we neither assume prior knowledge of the underlying causal graph nor of the structural differences between environments. Instead, we only assume that invariances in the observed environments are preserved in the target environment. Furthermore, while the goal in transportability is to derive whether and how one can identify a causal quantity (e.g., an interventional distribution) in the target environment based on data from the source environment, our goal is to learn worst-case optimal policies based on the source environments.

Graphical models have been used in reinforcement learning to represent the underlying Markov Decision Processes (MDP) under the framework of factored MDPs. Such methods, however, focus mainly on providing efficient planning algorithms rather than generalizing to a new environment \citep{kearns1999efficient, guestrin2003efficient, guestrin2001multiagent, jonsson2006causal}.
Although some recent studies have explored the use of causality and invariance for tackling the environmental shift problem in contextual bandits and, more generally, reinforcement learning \citep{zhang2020invariant, sonar2020invariant}, the actual roles and benefits of causality and invariance remain unclear and under-explored.

Our framework differs from the framework of causal bandits \citep{LS18, LLR16, YA18, deKroon2020causal}. While causal bandits exploit causal knowledge (either assumed to be known a priori or estimated from data)
for improving the finite sample performance in a single environment, our framework focuses on modeling distributional shifts and the ability to generalize to new environments.
Another line of work has addressed the problem of policy evaluation and learning under unobserved confounding between the action and the reward variables \citep{bareinboim2015bandits, sen2017contextual, tennenholtz2020off, kallus2020confounding, tennenholtz2021bandits}. In contrast, we consider the complementary problem of unobserved confounding between the covariates and the reward variables (see Section~\ref{sec:inv-policy}). 

\subsection{Offline Contextual Bandits}
\label{sec:off_cb}
We briefly review the offline contextual bandit problem  \citep{beygelzimer2009offset, NIPS2010_c0f168ce}, considering a setup in which some of the covariates (also known as context variables) are unobserved. More precisely, we assume that the covariates can be partitioned into observed and unobserved variables $X \in \mathcalbf{X}$ and $U \in \mathcalbf{U}$. Here, $\mathcalbf{X}$ and $\mathcalbf{U}$ are metric spaces; the reader may think of $\mathcalbf{X}\subseteq\mathbb{R}^d$ and $\mathcalbf{U}\subseteq\mathbb{R}^p$. As in the standard contextual bandit setup \citep{LZ07}, for each round, we assume that the system generates a covariate vector $(X, U)$ and reveals only the observable $X$ to an agent. From the observed covariates $X$, the agent selects an action $A \in \mathcal{A}$ according to a policy $\pi: \mathcalbf{X} \xrightarrow{} \Delta(\mathcal{A})$ that maps the observed covariates to the probability simplex $\Delta(\mathcal{A})$ over the set of actions $\mathcal{A}$. (In this work, we assume $\mathcal{A}$ to be finite). Adapting commonly used notation, we write, for all $x \in \mathcalbf{X}$ and $a \in \mathcal{A}$, $\pi(a|x) := \pi(x)(a)$. The agent then receives a reward $R$ depending on the chosen action $A$, and on both the observed and unobserved covariates $(X, U)$.

In the classical setting, one assumes that the covariates are drawn i.i.d.\ from a joint distribution $\P_{X,U}$
(an assumption we will relax when introducing multi-environment contextual bandits in Section~\ref{sec:multi-env}) 
and that the rewards are drawn from a conditional distribution $\P_{R \mid X,U,A}$. 
The agent is evaluated based on the performance of its policy $\pi$ which is measured by the policy value:
\begin{equation*}
    V(\pi) \coloneqq  \EX_{(X,U) \sim \P_{X,U}}\EX_{A \sim \pi(X)}\EX_{R \sim \P_{R \mid X, U, A}}\big[R\big].
\end{equation*}

The agent is now given a fixed training dataset that is collected offline: it consists of $n$ rounds from one or more different policies, i.e., $D \coloneqq \{(X_i, A_i, R_i, \pi_i(X_i))\}_{i=1}^n$, where $A_i \sim \pi_i(X_i)$\footnote{We assume knowledge of the initial policy $\pi_i$ to ease our presentation and focus our contribution on the environmental shifts problem. Our theoretical results and algorithms remain unchanged even if the initial policy is unknown and needs to be estimated from the offline data (see Appendix~\ref{app:details-resampling} for more details).} for all $i\in\{1,\dots,n\}$.
The goal of the agent is then to find a policy $\pi$ that maximizes the policy value over a given policy class $\Pi$, i.e., $\pi^* \in \argmax_{\pi \in \Pi} V(\pi)$.

As mentioned, this setting assumes
that the environment in which we deploy the agent is identical to the environment in which the training dataset was collected. Section~\ref{sec:multi-env} introduces a causal framework for multi-environment contextual bandits, a framework that relaxes the equal distribution assumption.

\section{A Causal Framework for Multi-environment Contextual Bandits} \label{sec:multi-env}
Instead of having a fixed distribution $\P_{X,U}$ over the covariates, we introduce a collection $\mathcal{E}$ of environments such that, in each round, the covariates are drawn from an environment-specific distribution $\P^e_{X,U}$ that depends on the environment $e \in \mathcal{E}$ in that round.

In practice, the agent only observes part of the environments $\mathcal{E}^{\obs} \subseteq \mathcal{E}$ 
and is expected to generalize well to all environments in $\mathcal{E}$ including the unseen environments $\mathcal{E} \setminus \mathcal{E}^{\obs}$. To formalize the problem, we first introduce a model that puts assumptions on how environments change the distributions of $X$, $U$ and $R$. Specifically, an environment $e$ can only perturb the distribution of the reward $R$ through altering the distribution of the covariates $X$ and $U$. This constraint makes it possible to generalize information learned from one set of environments to another. In this formulation -- even though the full conditional distribution of the reward $\P^{\pi, e}_{R \mid X, U, A}$ is assumed to be fixed across environments -- the observable distribution $\P^{\pi, e}_{R \mid X, A}$ after marginalizing out the unobserved $U$ may change from one environment to another (see, e.g., Figure~\ref{fig:example1b})

Formally, the assumptions are constructed via an underlying class of SCMs indexed by the environment and policy.\footnote{
 Readers familiar with the standard notion of SCMs may think about
 an SCM with a source node $E$.  
 $\mathcal{S}(\pi, e)$ then corresponds to
 an intervention 
 on the action variable (change of policy) and on some of the observed covariates variables (change of environment).
 Here, we consider fixed environments, so that we do not have to consider them as random draws from an underlying distribution; see also \cite{dawid2002}.}
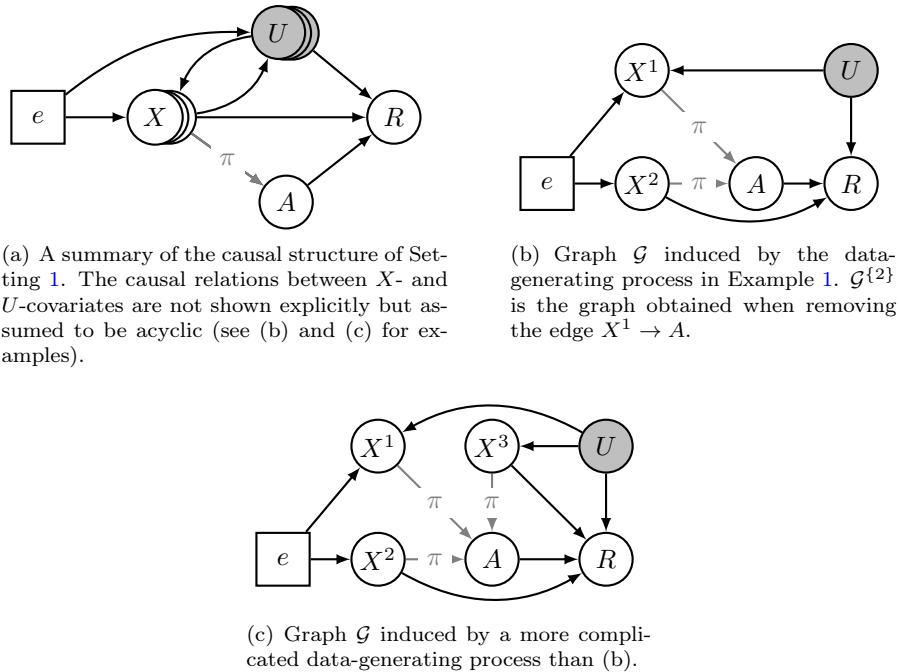
\begin{figure*}[t!]
    \centering
    \subfigure[A summary of the causal structure of Setting~\ref{setting:scmfixed}. The causal relations between $X$- and $U$-covariates are not shown explicitly but assumed to be acyclic (see (b) and (c) for examples).
    ]{
    \input{figs/fig_setting1}
    \label{fig:graph-setting1}
    }
    \qquad
    \subfigure[
    Graph $\mathcal{G}$ induced by the data-generating process in Example~\ref{ex:cf_cb}.
     $\mathcal{G}^{\{2\}}$ is the graph obtained when removing the edge $X^1 \rightarrow A$. 
    ]{
        \input{figs/fig_example1a}
        \label{fig:example1b}
    }
    \qquad
    \subfigure[Graph $\mathcal{G}$ induced by a more complicated data-generating process than (b).]{
        \input{figs/fig_example1b}
        \label{fig:example1c}
    }
    \caption{
    Graphs summarizing different data-generating models. White and grey circles represent observed and hidden variables, respectively.
    (b) Here, $\{X^2\}$ is $d$-invariant, 
    because $R \ci_{\mathcal{G}^{\{2\}}} e \mid X^S$,
    see Definition~\ref{def:inv_sets_NEW}.
    Any set $S$ that contains $X^1$ is not $d$-invariant 
    because of the open path $e \rightarrow X^1 \leftarrow U \rightarrow R$. 
    In practice, we do not assume that the structure is known but test for invariances \eqref{eq:h_0_NEW} from data. This requires testing under distributional shifts: even though $\{X^2\}$ is $d$-invariant, \eqref{eq:h_0_NEW} may not hold for a policy $\pi$ that depends on $X^1$ and $X^2$ because of the path $e \rightarrow X^1 \rightarrow A \rightarrow R$. %
    (c) A more complex model, where the environments do not act on all $X$ variables. Although $U$ has an edge into $X^3$, the subset $\{X^2, X^3\}$ is still a $d$-invariant set -- there is no edge from $e$ to $X^3$. Again, every subset of variables containing $X^1$ is not $d$-invariant. (In fact, in examples (b) and (c), $X^1$ is a strongly non-$d$-invariant variable, see Definition~\ref{def:conf-removing}, and cannot be part of a $d$-invariant set.) 
    }
    \label{fig:ex1-setting1}
\end{figure*}
\begin{setting}[Multi-environment (acyclic) SCMs for bandits]
\label{setting:scmfixed}
Let $\mathcalbf{X} = \mathcal{X}^1\times\hdots\times\mathcal{X}^d$ and $\mathcalbf{U} = \mathcal{U}^1\times\hdots\times\mathcal{U}^p$ be products of metric spaces, $\mathcal{A} = \{a^1, \dots, a^k\}$ a discrete action space, $\Pi := \{\mathcalbf{X} \xrightarrow{} \Delta(\mathcal{A})\}$ the set of all policies, and $\mathcal{E}$ a collection of environments.
For all $\pi\in\Pi$ and all $e\in\mathcal{E}$ we consider
the following SCMs,
\begin{equation}\label{eq:setting1-scm}
\mathcal{S}(\pi, e):\quad
\begin{cases}
U\coloneqq s_e(X, U,\epsilon_U)\\
X\coloneqq h_e(X, U, \epsilon_X)\\
A\coloneqq g_{\pi}(X, \epsilon_A) \\
R\coloneqq f(X, U, A, \epsilon_R),
\end{cases}
\end{equation}
where $(X, U, A, R)\in\mathcalbf{X}\times\mathcalbf{U}\times\mathcal{A}\times\mathbb{R}$, $(s_e)_{e \in \mathcal{E}}$, $(h_e)_{e \in \mathcal{E}}$,
and $f$ are measurable functions, $\epsilon=(\epsilon_U, \epsilon_X,\epsilon_A,\epsilon_R)$ is a random vector with independent components and a distribution $Q_{\epsilon}=Q_{\epsilon_U}\otimes Q_{\epsilon_X}\otimes Q_{\epsilon_A}\otimes Q_{\epsilon_R}$, and $g_{\pi}$ and $Q_{\epsilon_A}$ are such that for all $x\in\mathcalbf{X}$ it holds that $g_{\pi}(x, \epsilon_A)$ is a random variable on $\mathcal{A}$ with distribution $\pi(x)$.
Figure~\ref{fig:graph-setting1} visualizes the 
coarse-grained
structure of this setting. $U, X,$ and $A$ should be thought of as random vectors. Accordingly, $h_e$, for example, is a function with a multivariate output; it is a short-hand notation in the sense that a component of $h_e$ does not need to depend on all $X$, for example. In particular, we assume that the graph $\mathcal{G}$ (defined below) corresponding to the SCMs is acyclic, see Figure~\ref{fig:example1b}~and~\ref{fig:example1c} for an example.

We assume there exists a probability measure $\mu$ on $\mathcalbf{X}\times\mathcalbf{U}\times\mathcal{A}\times\mathbb{R}$ such that for all $\pi\in\Pi$ and all $e\in\mathcal{E}$ the SCM $\mathcal{S}(\pi, e)$ induces a unique distribution $\P^{\pi,e}$ over $(X, U, A, R)$ (see \cite{bongers2016foundations} for details) which is dominated by $\mu$ and marginally has full support on $\mathcalbf{X}$. We denote the corresponding density by $p^{\pi,e}$ and the corresponding expectations by $\EX^{\pi,e}$. 
Whenever a probability, density, or expectation does not depend on $\pi$, we omit $\pi$ and write $\EX^{e}[X]$ rather than $\EX^{\pi,e}[X]$, for example.
\end{setting}
Some remarks regarding 
Setting~\ref{setting:scmfixed}
are in order:
(1) We only use the SCMs as a flexible way of modeling the changes in the joint distribution with respect to the environment $e$ and the policy $\pi$. In particular, we do not use it to model any 
further 
intervention 
distributions
that do not correspond to a change of policy or environment.
(2) In practice, the precise 
form of the SCMs is unknown. 
Indeed, we will neither assume knowledge of the structural equations nor complete knowledge of the graph structure, except that the constraints induced by \eqref{eq:setting1-scm} hold.
(3) The assumption of a dominating measure for all environments ensures that we can always assume the existence of densities while also switching across environments. In particular, this avoids any measure-theoretic difficulties regarding conditional distributions. (4) The assumption that the induced distributions over $X$ have full support in all environments
ensures that the generalization problem when moving from $\mathcal{E}^{\obs}$ to $\mathcal{E}$ does not involve any extrapolation. 
Additionally, it ensures that conditional expectations such as $\EX^{\pi, e}[R\mid X=x]$ can be uniquely defined for all $x\in\mathcalbf{X}$ as integrals of the conditional densities.
(5) The environments are modelled fixed (and not random). However, we could also treat the environments as random variables which can be considered a special case of the fixed-environment setting (see Appendix~\ref{app:scmrandom}). 

We now introduce the graph $\mathcal{G}$ over the variables $(X^1,\ldots,X^d, U^1,\ldots,U^p, A, R)$ that visualizes the structure of the SCMs $\mathcal{S}(\pi,e)$ (for all $\pi \in \Pi$ and $e \in \mathcal{E}$).
We draw edges corresponding to the assignments in~\eqref{eq:setting1-scm}, add edges from all $X$ to $A$ and add an environment node, which has edges into all variables whose assignments are not the same in all environments. This is similar to the selection diagrams in \cite{pearl2011transportability}, with the difference that we consider multiple environments.

More precisely, $\mathcal{G}$ is constructed as follows: Each coordinate of the variables $(X, U, A, R)$ corresponds to a node.
The nodes are connected according to the assignments, that is, we draw a directed edge from a variable $B$ to a variable $C$ if, for at least one environment $e\in\mathcal{E}$, the variable $B$ appears on the right-hand side of the assignment of variable $C$ (see Figure~\ref{fig:example1b} for an example).
Let $\mathcal{I}_X \subseteq \{1, \ldots, d\}$ and $\mathcal{I}_U \subseteq \{1, \ldots, p\}$ index the variables $X^j$ and $U^{\ell}$ for which the structural assignments $X^j \coloneqq h^j_e(X, U, \epsilon_X)$ and $U^{\ell} \coloneqq s^{\ell}_e(X, U, \epsilon_U)$ in \eqref{eq:setting1-scm} vary with $e$, i.e., where there exist $e,f \in \mathcal{E}$ such that $h^j_e \neq h^j_f$ or $s^{\ell}_e \neq s^{\ell}_f$, respectively. The environments $\mathcal{E}$ correspond to perturbations on variables $X^{\mathcal{I}_X}$ or $U^{\mathcal{I}_U}$, which implies that for each $e \in \mathcal{E}$ the distribution $\P^{\pi, e}(X^{\mathcal{I}_X}, U^{\mathcal{I}_U} \mid U^{\{1,\dots,p\} \setminus \mathcal{I}_U}, X^{\{1,\dots,d\} \setminus \mathcal{I}_X})$ may vary. We augment the graph with a square node $e$ to represent the environments and draw a directed edge from the node $e$ to each of the perturbation targets $X^{\mathcal{I}_X}$ and $U^{\mathcal{I}_U}$. Furthermore, we draw edges from all nodes in $X$ to $A$ and mark them with $\pi$ (to represent their dependence on the policy). This graph $\mathcal{G}$ is assumed to be acyclic, that is, to not contain any directed cycles. 
By the Markov condition, which holds in SCMs \citep{peters2017elements}, the graph $\mathcal{G}$ defined above encodes (conditional) independence statements, which we will see relate to invariance, through the concept of $d$-separation. More precisely, the Markov condition states that any $d$-separation statement in a graph implies conditional independence \citep{pearl2009causality, Lauritzen1990, peters2017elements}. Here, we refer to the standard definition of $d$-separation when not distinguishing between the different types of nodes and denote by $\ci_{\mathcal{G}}$ a $d$-separation statement in a graph $\mathcal{G}$. For completeness, we define $d$-separation in Appendix~\ref{app:d-sep}.

For any $S \subseteq \{1, \ldots, d\}$, we also define $\mathcal{G}^S$ as the subgraph of $\mathcal{G}$, in which, instead of all $X$, only the covariates in $S$ point into $A$:
\begin{equation} \label{eq:GS}
 \mathcal{G}^S := \text{ subgraph of $\mathcal{G}$ without edges $X^{\{1,\dots,d\} \setminus S}$  to  $A$}.
\end{equation}
Neither $\mathcal{G}$ nor $\mathcal{G^S}$ depends on the choice of policy.

We are now ready to define contextual bandits with multiple environments.
\begin{definition}[Multi-environment Contextual Bandits]
\label{def:mecb}
    Assume Setting~\ref{setting:scmfixed}. In a multi-environment contextual bandit setup, before the beginning of each round, the system is in an environment $e \in \mathcal{E}$.
    Then, the system generates a covariate vector $(X, U)$ and reveals only the observable $X$ 
    and the environment label $e$
    to the agent. 
    Based on the observed covariates $X$, the agent selects an action $A$ according to the policy $\pi: \mathcalbf{X} \xrightarrow{} \Delta(\mathcal{A})$. The agent then receives a reward $R$, depending on the chosen action $A$ and on both the observed and unobserved covariates $(X, U)$.
    More precisely, we assume for all $i\in\{1,\ldots,n\}$ that $(X_i, U_i, A_i, R_i)$ are sampled independently according to $\P^{\pi_i, e_i}_{X, U, A, R}$ (see Setting~\ref{setting:scmfixed}). 
    The training data contains data from environments in $\mathcal{E}^{\obs}$.
    When $\abs{\mathcal{E}^{obs}} = \abs{\mathcal{E}} = 1$, the setup reduces to a standard contextual bandit setup.
\end{definition}
In the multi-environment contextual bandit setup, the covariates on different rounds are not identically distributed due to changes in the environments. %
We can thus use this framework to model situations, where the test environments differ from training environments.
We illustrate this setting with the following example, which we will refer back to several times throughout the paper.

\begin{example}
\label{ex:cf_cb}
Consider a linear multi-environment contextual bandit with the following underlying SCMs
\begin{equation*}
\mathcal{S}(\pi, e):\quad
\begin{cases}
U \coloneqq  \epsilon_{U} \\
X^1 \coloneqq  \gamma_e U + \epsilon_{X^1} \\
X^2 \coloneqq  \alpha_e + \epsilon_{X^2} \\
A \coloneqq  g_\pi(X^1, X^2, \epsilon_{A}) \\ %
R \coloneqq  {\begin{cases} 
\beta_1 X^2 + U + \epsilon_{R}, %
 &\mbox{if } A = 0 \\
\beta_2 X^2 - U +\epsilon_{R}, %
 & \mbox{if } A = 1,
\end{cases}
}
\end{cases}
\end{equation*}
where $\epsilon_R, \epsilon_A, \epsilon_{X^1}, \epsilon_{X^2}$ %
are jointly independent noise variables with zero mean, $\gamma_e, \alpha_e \in \mathbb{R}$ for all $e \in \mathcal{E}$, $\beta_1, \beta_2 \in \mathbb{R}$, and $\mathcal{A} = \{0, 1\}$. Figure~\ref{fig:example1b} depicts the induced graph $\mathcal{G}$. In this example, the environments influence the observed covariates in two ways: (a) they change the mean of $X^2$ via $\alpha_e$ and (b) they change the conditional mean of $X^1$ given $U$ via $\gamma_e$, while the rest of the components remain fixed across different environments. Here, the environment-specific coefficient $\gamma_e$ modifies the correlation between the observable $X^1$ and the unobserved variable $U$, and consequently between $X^1$ and the reward $R$. Thus, an agent that uses information from $X^1$ to predict the reward $R$ in the training environments may fail to generalize to other environments. To see this, consider a training environment $e=1$ and a test environment $e=2$ and let $\gamma_1 = 1$, $\gamma_2 = -1$ be the environment-specific coefficients in the training and test environments, respectively. In the training environment, we have a large positive correlation between $X^1$ and $U$, and consequently the agent will learn that the action $A=0$ 
yields a higher expected reward when observing a positive value of $X^1$ (and $A=1$ otherwise). However, the correlation between $X^1$ and $U$ becomes negative (and large in absolute value) in the test environment, which means that the policy that the agent learned from the training environment will now be harmful. We will see in Section~\ref{sec:inv-policy} that a policy that depends on a $d$-invariant set ($\{X^2\}$ in this example) does not suffer from this generalization problem and is guaranteed to generalize across different environments.
\end{example}
A similar structure appears in the medical example discussed in Section~\ref{sec:warfarin}. There, $A$ is the dose of a drug, 
$R$ is a response variable, 
$X$ are observed patient features and $U$ are unobserved genetic factors. The environment $e$ is (a proxy of) the continent on which the data was collected.
\subsection{Distributionally Robust Policies}
To evaluate the performance of an agent across different environments, we define a 
policy value that takes into account  environments. In particular, we focus on the worst-case performance of an agent across environments.
\begin{definition}[Robust Policy Value]
    For a fixed policy $\pi\in\Pi$, and a set of environments $\mathcal{E}$, we define the \emph{robust policy value} $V^{\mathcal{E}}(\pi) \in \mathbb{R}$ as the worst-case expected reward
    \begin{equation}
        V^{\mathcal{E}}(\pi) \coloneqq  \inf_{e \in \mathcal{E}} \EX^{\pi,e}\big[R \big].
    \end{equation} 
\end{definition}
Intuitively, an agent that maximizes the robust policy value is expected to perform well (relative to other agents) in the most harmful environment. The idea of optimizing worst-case performance has been suggested in the reinforcement learning literature \citep{garcia2015comprehensive, amodei2016concrete} to ensure safe behavior of an agent and prevent catastrophic events and has also been used to formulate adversarial training \citep{ijcai2021-591} as well as out-of-distribution generalization  \citep{ye2021towards}.

We now assume that, for several observed environments, we are given an i.i.d.\ sample from a multi-environment contextual bandit, see Definition~\ref{def:mecb}. More precisely, we 
assume to observe $D \coloneqq  \{(X_i, A_i, R_i, \pi_i(X_i), e_i)\}_{i=1}^{n}$, where $e_i \in \mathcal{E}^{\obs}$, $A_i \sim \pi_i(X_i)$, $(X_i, A_i, R_i) \simind \P^{\pi_i, e_i}_{X, A, R}$ for all $i \in \{1,\dots,n\}$. Using only $D$, we aim to solve the following maximin problem\footnote{The maximum can always be attained when $\Pi$ is an unrestricted policy class
and takes a form similar to \eqref{eq:optimal_unconf_policy}.}:
\begin{equation}
    \argmax_{\pi \in \Pi} V^{\mathcal{E}}(\pi).\label{eq:maximin}
\end{equation}
If we do not observe all the environments, solving the maximin problem \eqref{eq:maximin} is impossible without further assumptions. A baseline
approach to this problem is to pool the data from all 
training environments and learn a policy that maximizes the policy value ignoring the environment structure. 
We show in Appendix~\ref{app:setting2_continue} that this is indeed optimal if the observed covariates explain all of the environment based distributional shifts in $R$, e.g., if all relevant covariates have been observed. However, if for example, hidden variables are present, the pooling approach does not necessarily yield
an optimal policy and the learned policy may fail to generalize to unseen test environments.

\section{Invariant Policies for Distributional Robustness}
\label{sec:inv-policy}
We now consider the general case in which the environment shifts may not be explained by the observed covariates. To this end, we introduce $d$-invariant sets and policies, and show that, under Setting~\ref{setting:scmfixed}, the maximin problem \eqref{eq:maximin} can be reduced to finding an optimal $d$-invariant policy given certain assumptions, see Proposition~\ref{prop:inv_policy_NEW} and Theorem~\ref{thm:inv_policy_NEW}.
This becomes particularly relevant if important variables remain unobserved. If all variables are observed, it suffices to pool the observed environments.
\begin{remark}
If there are no hidden variables, one can solve the objective \eqref{eq:maximin} by a standard policy optimization using all covariates $X$, without taking into account further concepts such as invariance or causality.
This statement is made precise and proved as Proposition~\ref{prop:no_uc} in Appendix~\ref{app:setting2_continue}.
\end{remark}
Nevertheless, in more realistic cases (see e.g., Figures~\ref{fig:example1b} and \ref{fig:example1c}), $d$-invariant sets and policies (introduced below) play a central role in solving the distributionally robust objective \eqref{eq:maximin}.

\begin{definition}[$d$-invariant Sets\footnote{The notion of $d$-invariant sets is related to $S$-admissibility in \cite{pearl2011transportability}. We use the term `$d$-invariant' to emphasize that the definition is based on the $d$-separation statement \eqref{eq:inv_con_NEW} and involves the unseen environments. In related contexts, sometimes the term `generalizing' is used \citep{pfister2019learning}. Section~\ref{sec:learn_opt_policy} introduces the invariance hypothesis \eqref{eq:h_0_NEW} that is testable from the observed data and discusses the assumptions required to connect the two conditions.}]
\label{def:inv_sets_NEW}
A subset $S \subseteq \{1,\dots,d\}$ is said to be 
\emph{$d$-invariant} if the following $d$-separation statement holds:
\begin{equation}
    R \ci_{\mathcal{G}^S} e \mid X^S \label{eq:inv_cond_S_NEW},
\end{equation}
where $\mathcal{G}^S$ is defined in \eqref{eq:GS}.
\end{definition}
Our approach relies on the existence of a $d$-invariant set. We therefore make this assumption explicit.
\begin{assumption}\label{assm:inv_set_exists}
    There exists a subset $S \subseteq \{1,\dots,d\}$ such that $S$ is $d$-invariant.
\end{assumption}
Under faithfulness \citep{pearl2009causality}, Assumption~\ref{assm:inv_set_exists} is testable from the observed data (see Section~\ref{sec:learning_invariant_sets}).
Next, we define $d$-invariant policies. For all subsets $S \subseteq \{1, \dots, d\}$, let us denote the set of all policies that depend only on $X^{S}$ by $\Pi^S \coloneqq \{ \pi \in \Pi \mid \exists \pi^S: \mathcalbf{X}^{S} \rightarrow \Delta(\mathcal{A}) \text{ s.t. } \forall x \in \mathcal{X}\,, \pi(\cdot | x) = \pi^S(\cdot | x^S) \} \subseteq \Pi$.
\begin{definition}[$d$-invariant Policies]
\label{def:inv_policies_NEW}
A policy $\pi$ is said to be \emph{$d$-invariant with respect to a subset $S \subseteq \{1,\dots,d\}$}
if $S$ is a $d$-invariant set and $\pi \in \Pi^S$.
\end{definition}
We denote by $\mathbf{S}_{\inv} \coloneqq  \{ S \subseteq \{1, \dots, d\} \mid S \text{ is $d$-invariant}\}$ the collection of all $d$-invariant sets and $\Pi_{\inv} \coloneqq  \{\pi \in \Pi\,|\, 
\exists S \text{ s.t. }
\pi \text{ is $d$-invariant\ w.r.t.\ } S \}$ the collection of $d$-invariant policies.
For now, we assume to have access to the set of $d$-invariant policies $\Pi_{\inv}$. Section~\ref{sec:learn_opt_policy} discusses when and how we can learn $\Pi_{\inv}$ from the observed data.

Because of the hidden variables $U$, the conditional mean $\EX^{\pi, e}[R\mid X = x]$ is not ensured to be stable across the environments. Nevertheless, a $d$-invariant policy ensures that parts of the conditional mean are unchanged across environments.
\begin{lemma}\label{lemma:inv_policy}
Let $S \in \mathbf{S}_{\inv}$ be a $d$-invariant set and $\pi \in\Pi^{S}$. 
It holds for all $e, f \in \mathcal{E}$ and $x \in \mathcalbf{X}^{S}$ that
\begin{equation*}
\EX^{\pi,e}\big[R \mid X^{S} = x \big] = \EX^{\pi,f}\big[R \mid X^{S} = x\big]. \numberthis \label{eq:inv_con_NEW}
\end{equation*}
\end{lemma}
\begin{proof}
    See Appendix~\ref{proof:lemma:inv_policy}.
\end{proof}

For $S \in \mathbf{S}_{\inv}$, Lemma~\ref{lemma:inv_policy} implies that if a policy $\pi \in \Pi^S$ is optimal among $\Pi^S$ in the observed environments, then $\pi$ is also optimal among $\Pi^S$ in all environments (Proposition~\ref{prop:inv_policy_NEW_1}).
With the following assumption, we show in Proposition~\ref{prop:inv_policy_NEW_2} that the same holds when replacing $\Pi^S$ by $\Pi_{\inv}$.

\begin{assumption}
\label{assm:U_causes_R}
    Let $\mathcal{G}$ be the graph of the SCMs in Setting~\ref{setting:scmfixed}. Then, for all $\ell\in\{1,\ldots,p\}$, there must be an edge from $U^\ell$ to $R$ in $\mathcal{G}$.
\end{assumption}

\begin{proposition}
\label{prop:inv_policy_NEW}
Assume Setting~\ref{setting:scmfixed} and Assumption~\ref{assm:inv_set_exists}. 
Then the following statements hold.

\begin{enumerate}[label=(\roman*),
     ref={\theproposition(\roman*)}]
        \item \label{prop:inv_policy_NEW_1}
        Let $S \in \mathbf{S}_{\inv}$ and $\pi_{\opt}^S \in \argmax_{\pi \in \Pi^S} \sum_{e \in \mathcal{E}^{\obs}}\EX^{\pi,e}[R]$. We then have
    \begin{equation*}
         \forall \pi \in \Pi^S:
         \qquad V^{\mathcal{E}}(\pi) \leq V^{\mathcal{E}}(\pi_{\opt}^S). \numberthis \label{eq:inv-policy_NEW_1}
    \end{equation*}
    \item \label{prop:inv_policy_NEW_2}
    Let $\pi^* \in \argmax_{\pi\in\Pi_{\inv}}\sum_{e \in \mathcal{E}^{\obs}}\EX^{\pi,e}[R]$. If Assumption~\ref{assm:U_causes_R} holds, we have 
    \begin{equation*}
         \forall \pi \in \Pi_{\inv}:
         \qquad V^{\mathcal{E}}(\pi) \leq V^{\mathcal{E}}(\pi^*). \numberthis \label{eq:inv-policy_NEW_2}
    \end{equation*}
\end{enumerate}

\end{proposition}
\begin{proof}
    See Appendix \ref{proof:prop:inv_policy}.
\end{proof}
Proposition~\ref{prop:inv_policy_NEW} shows that a $d$-invariant policy that is optimal under the observed environments outperforms all other $d$-invariant policies, even on the test environments. But what about other policies that are not $d$-invariant? 
We will see in Theorem~\ref{thm:inv_policy_NEW} that under certain assumptions on the set $\mathcal{E}$ of environments,
they cannot perform better than the above $\pi^*$
either.

The key argument in the proof of Proposition~\ref{prop:inv_policy_NEW_2} is the identifiability of the optimal $d$-invariant set. Assumption~\ref{assm:U_causes_R} is necessary for this identifiablity: 
if the assumption is violated and there are multiple $d$-invariant sets, one can, in general, not say which of those $d$-invariant sets is optimal with respect to all environments $\mathcal{E}$ (see Appendix~\ref{app:counter_example} for a more detailed discussion). While, without Assumption~\ref{assm:U_causes_R}, the $d$-invariant set that is most predictive on $\mathcal{E}^{\obs}$ is no longer guaranteed to be worst-case optimal, it still satisfies a weaker guarantee shown in Theorem~\ref{thm:inv_policy_1} below.

We now outline the assumptions on the set $\mathcal{E}$ of environments facilitating this result. As seen in Example~\ref{ex:cf_cb}, the crucial difference between a $d$-invariant policy $\pi^{\{2\}}$ (a policy that only depends on $X^2$) and a non-$d$-invariant policy $\pi^{\{1, 2\}}$ (a policy that depends on both $X^1$ and $X^2$) is that $\pi^{\{1, 2\}}$ can use information related to variables confounded with the reward ($X^1$ in this example) that may change across environments. In cases where the environments do not change the system `too strongly' it can therefore happen that using such information is beneficial across all environments. In practice, however, one might not know how strong the test environments can change the system in which case such information can become useless or even harmful. Intuitively, this happens, for example, if environments exist where the non-$d$-invariant confounded variables no longer contain any information about the reward. Formally, we make the following definition.

\begin{definition}[Confounding Removing
Environments]\label{def:conf-removing}
    For $j \in \{1, \ldots, d\}$, we say that the variable $X^j$ is strongly non-$d$-invariant if for all $S \subseteq \{1, \ldots, d\}$
    \begin{align*}
        R \not\ci_{\mathcal{G}^S} e \mid X^{S \cup \{j\}}.
    \end{align*}
        An environment $e \in \mathcal{E}$ is said to be a confounding removing environment if for all $\pi\in\Pi$ it holds that
    \begin{align}
        X^j \ci_{\mathcal{G}^{\pi, e}} U,
    \end{align}
                for all strongly non-$d$-invariant variables $X^j$, where $\mathcal{G}^{\pi, e}$ is the graph induced by the SCM $\mathcal{S}(\pi, e)$.
\end{definition}
The two d-separation statements in Definition~\ref{def:conf-removing} are in different graphs: 
Both graphs $\mathcal{G}^S$ and $\mathcal{G}^{\pi,e}$ are subgraphs of $\mathcal{G}$. The distinction that is important for this definition is that while $\mathcal{G}^S$ contains all edges between the covariates $(X, U)$ that appear in at least one environment, the graph $\mathcal{G}^{\pi,e}$ only contains the edges that are active in the environment $e\in\mathcal{E}$.
Furthermore, to provide more understanding of the strongly non-$d$-invariant variables, we characterize a graphical criterion for such variables in Appendix~\ref{sec:sb_invariance}. There we show that the strongly non-$d$-invariant variables are the variables that are directly affected by $e$ and are confounded with $R$ through $U$, and descendants of such variables. These strongly non-$d$-invariant variables should not be included if one wants to find $d$-invariant sets. For example in Figure~\ref{fig:example1c}, the variable $X^1$ is strongly non-$d$-invariant and the $d$-invariant sets $\{X^2\}$ and $\{X^2, X^3\}$ are the sets that do not contain $X^1$.

To give an example of a confounding removing environment, consider the graph $\mathcal{G}^S$ in Example~\ref{ex:cf_cb} (see Figure~\ref{fig:example1b}).
For any subset $S$ where $\{1\} \subseteq S$ the path $e \rightarrow X^1 \leftarrow U \rightarrow R$ is open, and therefore $X^1$ is strongly non-$d$-invariant.
A confounding removing environment is an environment that removes the incoming edge from $U$ to $X^1$. In such an environment, the variable $X^1$ does not contain any information about the reward $R$. A similar notion of confounding removing environments is used in \cite{christiansen2020causal} in the setting of prediction.

The existence of confounding removing environments implies that at least in some of the environments it is impossible to benefit from a non-$d$-invariant policy. To ensure that one cannot benefit in the worst-case, we therefore introduce the following assumption.

\begin{assumption}
\label{assm:strong_envs}
For all $e \in \mathcal{E}$, there exists $f \in \mathcal{E}$ such that $f$ is a confounding removing environment and it holds that $\P^{e}_X = \P^{f}_X$.
\end{assumption}
To give an example, let $\mathcal{I} \subseteq \{1, \ldots, d\}$ index the variables $X^j$ for which there is an edge from $e$ to $X^j$ in the graph $\mathcal{G}$. If the set $\mathcal{E}$ of environments consists of arbitrary interventions on $X^{\mathcal{I}}$, then Assumption~\ref{assm:strong_envs} is satisfied. 

\begin{theorem}
\label{thm:inv_policy_NEW}
Assume Setting~\ref{setting:scmfixed} and Assumption~\ref{assm:inv_set_exists}. Let $\pi^*$ be an optimal $d$-invariant policy under the observed environments, $\pi^* \in \argmax_{\pi \in \Pi_{\inv}} \sum_{e \in \mathcal{E}^{\obs}}\EX^{\pi,e}[R]$. We then have the following statements.
\begin{enumerate}[label=(\roman*),
     ref={\thetheorem(\roman*)}]
        \item \label{thm:inv_policy_1}
        Let $\pi_{a}$ be the policy that always chooses an action $a \in \mathcal{A}$. We have for all $e \in \mathcal{E}$ that 
    \begin{equation}
        \max_{a \in \mathcal{A}}\EX^{\pi_a, e}[R]\footnote{A (conditional) expectation under $\pi_a$ can also be written in terms of do-notation \citep{pearl2009causality}, e.g., $\forall a \in \mathcal{A}, x \in \mathcal{X}: \EX^{\pi_a, e}[R \mid X=x] = \EX^{e}[R \mid X=x, do(A=a)]$. We use the $\pi_a$ notation to make our presentation consistent.}
\leq 
        \EX^{{\pi}^*, e}[R]. \label{eq:lower_bound_invariant_policy_NEW}
            \end{equation}
    \item \label{thm:inv_policy_2}
    If Assumptions~\ref{assm:U_causes_R}~and~\ref{assm:strong_envs} hold, we have  
\begin{equation}
     \forall \pi \in \Pi:
     \qquad V^{\mathcal{E}}(\pi) \leq V^{\mathcal{E}}(\pi^*). \label{eq:invariant_policy_is_maximin_optimal_NEW}
    \end{equation}
\end{enumerate}
\end{theorem}

\begin{proof}
    See Appendix \ref{proof:thm:inv_policy}.
\end{proof}
The first statement of Theorem~\ref{thm:inv_policy_NEW} implies that in all environments the expected reward under an optimal $d$-invariant policy is larger than any
optimal context-free policy. In other words, the information gained from the $d$-invariant set of covariates (the set that $\pi^*$ depends on) is generalizable across environments in the sense that it is not harmful in any environment. The second statement states that if the environments $\mathcal{E}$ are sufficiently strong (Assumption~\ref{assm:strong_envs}) then an optimal $d$-invariant policy $\pi^*$ maximizes the robust policy value $V^{\mathcal{E}}$.

The above results motivate a procedure to solve
the distributionally robust objective \eqref{eq:maximin}. Proposition \ref{prop:inv_policy_NEW} implies that 
if we consider a policy class containing only the  $d$-invariant policies, the maximin problem reduces to a standard policy optimization problem. Theorem \ref{thm:inv_policy_NEW} shows that an optimal $d$-invariant policy, under Assumption~\ref{assm:strong_envs},
is a solution to the distributionally robust objective. In other words, given a training dataset $D$, we seek to operationalize the following two steps: (a) find the set $\Pi_{\inv}$ of all $d$-invariant policies (Section~\ref{sec:learning_invariant_sets} discusses under which assumptions this is possible), (b) use offline policy optimization to solve $\argmax_{\pi \in \Pi_{\inv}} V^{\mathcal{E}^{\obs}}(\pi)$ on the dataset $D$.

One of the key components of the proposed method is to test whether a policy $\pi$, which may be different from the policy generating the data, is $d$-invariant 
using data obtained from the observed environments $\mathcal{E}^{\obs}$. The following section proposes such a test, discusses the assumptions required to learn the set of $d$-invariant policies, and gives a detailed description of the whole procedure.
\section{Learning an Optimal Invariant Policy}
\label{sec:learn_opt_policy}

\subsection{Learning Invariant Sets}\label{sec:learning_invariant_sets}
Our theoretical results (Proposition~\ref{prop:inv_policy_NEW} and Theorem~\ref{thm:inv_policy_NEW}) in the previous section assume that the set of all $d$-invariant policies $\Pi_{\inv}$ is given. We now turn to the task of learning $\Pi_{\inv}$ which boils down to searching for the collection of all $d$-invariant sets $\mathbf{S}^{\inv}$ using data obtained from the observed environments $\mathcal{E}^{\obs}$. To this end, we first define, for all $S \subseteq \{1,\dots,d\}$, $\pi \in \Pi$ and $\mathcal{E}^{\prime} \subseteq \mathcal{E}$, the null hypothesis
\begin{equation}
\label{eq:h_0_NEW}
    H_0(S,\pi, \mathcal{E}^{\prime}):\,
        \P^{\pi,e}_{R \mid X^{S}} \text{ is the same for all } e \in \mathcal{E^{\prime}}.
\end{equation}
In the case $\mathcal{E}^{\prime} = \mathcal{E}^{obs}$, we refer to $H_0(S, \pi, \mathcal{E}^{obs})$ as $\mathcal{E}^{\obs}$-invariance (which does not consider the unseen environments). Furthermore, we call a set $S$ \emph{invariant} if there exists $\pi \in \Pi^S$ such that $H_0(S, \pi, \mathcal{E}^{obs})$ holds and a policy $\pi$ \emph{invariant with respect to $S$} if $\pi \in \Pi^S$ and $S$ is invariant.
We now state our core assumptions that make learning possible.
\begin{assumption} \label{assm:mean_ff_NEW}
    For all $S \subseteq \{1,\dots, d\}$, the following holds:
    \begin{enumerate}[label=(\roman*),
     ref={\theassumption(\roman*)}]
    \item \label{assm:mean_ff1_NEW}
        $\exists \pi \in \Pi^S: H_0(S,\pi, \mathcal{E})$ true  $\implies R \ci_{\mathcal{G}^S} e \mid X^S$
    \item \label{assm:mean_ff2_NEW} $\forall \pi \in \Pi^S: H_0(S,\pi, \mathcal{E}^{\obs})$ true $\implies H_0(S,\pi, \mathcal{E})$ true
    \end{enumerate}
\end{assumption}
Assumption~\ref{assm:mean_ff1_NEW} connects the conditional distribution invariance used in the null hypothesis \eqref{eq:h_0_NEW} to the $d$-invariance condition given in \eqref{eq:inv_cond_S_NEW} (The reversed implication follows by Lemma~\ref{lemma:markov_cond}, Appendix~\ref{proof:lemma:inv_policy}.)
This assumption is a special case 
of the faithfulness assumption \citep{pearl2009causality} which is a fundamental assumption in causal discovery methods (e.g., \cite{glymour2019review}) that, in linear SCMs, holds with probability one if the linear coefficients are drawn from a distribution that is absolutely continuous with respect to Lebesgue measure \citep{meek1995, Spirtes2000}. Assumption~\ref{assm:mean_ff2_NEW} ensures that any invariance found in the observed environments $\mathcal{E}^{\obs}$ can be generalized to all environments $\mathcal{E}$. Implicitly, it requires that the observed environments are sufficiently heterogeneous\footnote{For example, if the observed environments are identical, we clearly would not be able to distinguish $d$-invariant sets from other sets using the observed data. Assumption~\ref{assm:mean_ff2_NEW} prevents such cases.}. This type of assumption is also at the core of other invariance-based methods \citep{rojas2018invariant, Magliacane2018, arjovsky2019invariant, pfister2021SR}.

At first glance, Assumption~\ref{assm:mean_ff1_NEW} suggests that we have to check the hypothesis $H_0(S, \pi, \mathcal{E})$ for all $\pi \in \Pi^S$ to conclude whether or not $S$ is $d$-invariant. Fortunately, as shown in Proposition~\ref{prop:inv_set_NEW}, we actually only need to check the null hypothesis for a single $\pi \in \Pi^S$.
\begin{proposition}
\label{prop:inv_set_NEW}
Assume Setting~\ref{setting:scmfixed} and Assumption~\ref{assm:mean_ff_NEW}. Then, for all subsets $S \subseteq \{1, \dots, d\}$ and for all policies $\pi,\tilde{\pi} \in \Pi^S$, it holds that
\begin{equation}
    H_0(S,\pi, \mathcal{E}) \text{ true}
    \iff
    H_0(S,\tilde{\pi}, \mathcal{E}) \text{ true}.
\end{equation}
\end{proposition}
\begin{proof}
    See Appendix \ref{proof:prop:inv_set}.
\end{proof}

Assumption~\ref{assm:mean_ff_NEW} and Proposition~\ref{prop:inv_set_NEW} make the learning problem tractable.
The task of testing whether a set $S$ is $d$-invariant boils down to testing the $\mathcal{E}^{\obs}$-invariance hypothesis $H_0(S,\pi^S, \mathcal{E}^{\obs})$ for a single $\pi^S \in \Pi^S$. We therefore have the flexibility to choose any $\pi^S$ from $\Pi^S$ to test the hypothesis (called the test policy). We discuss strategies for choosing the test policy in Section~\ref{sec:target_test}.

Testing $H_0(S,\pi^S, \mathcal{E}^{\obs})$ for $\pi^S\in\Pi^S$ by directly checking for a change in the conditional distributions across environments in the observed data is, however, not in general possible. This is because the observed data may have been generated based on an initial policy $\pi^0$ that does not satisfy $\pi^0 \in \Pi^S$. It can therefore happen that $H_0(S,\pi^S, \mathcal{E}^{\obs})$ is true but $H_0(S,\pi^0, \mathcal{E}^{\obs})$ is not.

We illustrate this point using the example graph $\mathcal{G}$ given in Figure~\ref{fig:example1b}. For a policy depending only on $S=\{2\}$ the environment $e$ is d-separated from $R$ given $X^{\{2\}}$ in $\mathcal{G}^{\{2\}}$, which implies that $\{2\}$ is $d$-invariant, and in particular that $H_0(\{2\}, \pi^{\{2\}}, \mathcal{E}^{\obs})$ is true by the Markov property (see Lemma~\ref{lemma:markov_cond} in Appendix~\ref{proof:lemma:inv_policy}). However, if the initial policy $\pi^0$ depends on both $X^1$ and $X^2$, then the path $e \rightarrow X^1 \rightarrow A \rightarrow R$ in Figure~\ref{fig:example1b} is open, which implies, by Assumption~\ref{assm:mean_ff_NEW}, that $H_0(\{2\}, \pi^{\{1,2\}},  \mathcal{E}^{\obs})$ is not true.\footnote{In the same example, when conditioning on $\{1,2\}$, the path $e \rightarrow X^1 \leftarrow U \rightarrow R$ is also open, which shows that $S = \{1, 2\}$ is not a $d$-invariant set.
}

Thus, in general, we cannot directly test the $\mathcal{E}^{\obs}$-invariance hypothesis of a set $S$ by using the observed data that were generated by the initial policy. Instead, we need to test $H_0(S, \pi^S, \mathcal{E}^{\obs})$ for a policy $\pi^S \in \Pi^S$ that is different from the data-generating policy $\pi^0$ (by Proposition~\ref{prop:inv_set_NEW} it suffices to test a single policy).
As we detail in the following section, we can do so by applying an off-policy test for invariance by resampling the data to mimic the policy $\pi^S$. 

\subsection{Testing Invariance under Distributional Shifts}\label{sec:resampling}
Consider a fixed set $S \subseteq \{1, \ldots, d\}$ and a pre-specified test policy $\pi^S \in \Pi^S$ (see Section~\ref{sec:target_test} for how to choose $\pi^S$). 
To test the hypothesis $H_0(S, \pi^S, \mathcal{E}^{\obs})$, we apply the off-policy test from \cite{thams2021statistical}, which draws a target sample from $\pi^S$ by resampling the offline data -- drawn from $\pi^0$ -- and then tests the invariance in this target sample. More formally, let $\mathcal{E}^{\obs} \coloneqq \{e_1, \dots, e_L\}$ and suppose that for every $e_j \in \mathcal{E}^{\obs}$ a dataset $D^{e_j}$ consisting of $n_e$ observations $D^{e_j} = \{(X_i^{e_j}, A_i^{e_j}, R_i^{e_j}, \pi^0(A_i^{e_j}|X_i^{e_j}))\}_{i=1}^{n_{e_j}}$ is available. For each environment $e_j$, we draw a weighted resample $D^{e_j, \pi^S}$ of $D^{e_j}$ using the weighted resampling procedure introduced in \cite{thams2021statistical}.\footnote{Importance weighting is not applicable here because the test statistics of an invariance test cannot be expressed in terms of weighted averages. See also the discussion in \cite{thams2021statistical}.}
We then apply an invariance test $\varphi^S(D^{e_1, \pi^S}, \ldots, D^{e_L, \pi^S})$ to the resampled data, to test the $\mathcal{E}^{\obs}$-invariance hypothesis $H_0(S, \pi^S, \mathcal{E}^{\obs})$.
An invariance hypothesis test $\varphi^S$ is a function (into $\{0,1\}$) that takes data from environments $e_1, \ldots, e_L$, each of size $m_{e_i}$, and tests whether $S$ is invariant. Here, $\varphi^S = 1$ indicates that we reject the hypothesis of invariance. We detail a concrete test $\varphi^S$ in Section~\ref{sec:target_test}.
In Appendix~\ref{app:details-resampling}, we provide details on the resampling scheme, that is,  a formal definition of $D^{e_j, \pi^S}$  and show that the theoretical guarantees on the asymptotic level proved in \cite{thams2021statistical} also extend to our application.

\subsection{Algorithm for Invariant Policy Learning} \label{sec:algo_inv_policy}
The previous sections discuss finding invariant subsets $S$. We now discuss how to employ this in an algorithm that learns an optimal invariant policy.
We assume that we are given an off-policy optimization algorithm $\var{off\_opt}$ that takes as input a sample $D\coloneqq (D^{e_1}, \ldots, D^{e_L})$ and a policy space $\Pi$, and returns an optimal policy $\pi^*$ and its estimated expected reward $\hat{\EX}^{\pi^*}(R)$. 

Here, we present one choice of $\var{off\_opt}$ that we use in the experimental section; our approach can also be applied with other off-policy optimization algorithms.
Given a policy space $\Pi^S$, we consider an optimal policy of the form 
\begin{equation}
    \pi^S(a \mid x)\coloneqq \mathbbm{1}\big[ a = \argmax_{a^\prime \in \mathcal{A}} Q^{S}(x,a^\prime) \big], \label{eq:policy_off_opt}
\end{equation}
where $Q^{S}(x, a) \coloneqq \frac{1}{L}\sum_{\ell=1}^{L} \EX^{\pi_a, e_{\ell}}[R \mid X^S = x]$ denotes the pooled conditional mean under the policy that always selects an action $a$.\footnote{In our framework, changing the policy corresponds to intervening on the underlying SCM (see Setting~\ref{setting:scmfixed}). The expression $Q^S(x,a)$ is derived from expectations under such interventions and can thus be considered a causal quantity.}

Let $\pi^0$ be an initial policy generating the sample $D$. By our assumption in Setting~\ref{setting:scmfixed}, the policy $\pi^0$ depends only on the observed covariates $X$. We therefore have that for all $S \subseteq \{1, \dots, d \}$ the pooled conditional mean $Q^S(x,a)$ is identifiable for all $a \in \mathcal{A}$ and $x \in \mathcal{X}^S$ as shown in Lemma~\ref{lemma:q_identification} below.
\begin{lemma}\label{lemma:q_identification}
Let $S \in \mathbf{S}_{\inv}$ be a $d$-invariant set. It holds for all $ x \in \mathcal{X}^S$ and all $a \in \mathcal{A}$ that
\begin{equation}\label{eq:q_estimator}
    Q^S(x,a) = \frac{1}{L}\sum_{\ell = 1}^{L}\EX^{\pi^0, e_{\ell}}\big[\tfrac{R}{\pi^0(A \mid X)} \mid X^S = x, A = a\big].
\end{equation}
\end{lemma}
\begin{proof}
    See Appendix \ref{proof:lemma:q_identification}.
\end{proof}
Here, we express the causal quantity $Q^S(x,a)$ entirely in terms of expecations under the observed policy $\pi^0$ by using reweighting. Equivalently, one can also express $Q^S(x,a)$ with the backdoor adjustment formula \citep{pearl2009causality}. While the two formulations are equivalent, the resulting estimators are different (see the discussion in Appendix~\ref{app:backdoor_vs_rew}).

We propose to estimate $Q^S$ by a weighted least squares approach in which we consider a parameterized function class $\{f_{\theta}:\mathcalbf{X}^S \times \mathcal{A} \xrightarrow{} \mathbb{R} \mid \theta \in \Theta^S \}$ and assume that there exists a unique $\theta^S_0 \in \Theta^S$ such that for all $x\in\mathcal{X}^S$ and $a\in\mathcal{A}$ it holds that $Q^S(x,a) = f_{\theta^S_0}(x,a)$. That is, we consider
\begin{equation}
    \hat{\theta}^S_n \coloneqq \argmin_{\theta \in \Theta^S} \frac{1}{L}\sum_{\ell = 1}^L \frac{1}{n_{e_\ell}}\sum_{i = 1}^{n_{e_\ell}}  \frac{(f_{\theta}(A^{e_\ell}_i, {X^{e_\ell}_i}^{S}) - R^{e_\ell}_i)^2}{\pi^0(A^{e_\ell}_i \mid X^{e_\ell}_i)} , \label{eq:weighted_regressor}
\end{equation}
where $n \coloneqq (n_{e_1},\dots,n_{e_L})$. We then plug the estimate $\widehat{Q}^S\coloneqq f_{\hat{\theta}^S_n}$ into \eqref{eq:policy_off_opt} to obtain our (estimated) optimal policy. Proposition~\ref{prop:q_consistency} shows that, under some regularity conditions, $\hat{\theta}^S_n$ is a consistent estimate of $\theta^S_0$.
\begin{proposition}\label{prop:q_consistency}
    Assume Setting~\ref{setting:scmfixed} and Assumption~\ref{assm:inv_set_exists}. Let $S \in \mathbf{S}^{\inv}$ be a $d$-invariant set. Assume that
    \begin{itemize}
        \item[(i)] $\Theta^S$ is compact,
        \item[(ii)] there exists a unique $\theta^S_0 \in \Theta^S$ s.t. $\forall x \in \mathcal{X}^S$, $\forall a \in \mathcal{A}: Q^S(x, a) = f_{\theta^S_0}(x, a)\,\mu\text{-a.s.}$,
        \item[(iii)] $\forall x \in \mathcal{X}^S, \forall a \in \mathcal{A}:$ $\theta \rightarrow f_{\theta}(x, a)$ is continuous on $\Theta^S$,
        \item[(iv)] $\forall e \in \mathcal{E}^{\obs}:\EX^{\pi^0,e}[\sup_{\theta \in \Theta^S} (R-f_{\theta}(X, A))^2] < \infty$,
        \item[(v)] $\exists \delta > 0\text{ s.t. } \forall x \in \mathcal{X},\forall a \in \mathcal{A}: \pi^0(a|x) \geq \delta$.
    \end{itemize}
    Then, $\hat{\theta}^S_n$ is a consistent estimate of $\theta^S_0$, i.e., $\|\hat{\theta}^S_n - \theta^S_0\|_{\infty} \rightarrow 0$ in probability as  $n_{e_1},\dots,n_{e_L} \rightarrow \infty$.
\end{proposition}
\begin{proof}
    See Appendix \ref{proof:prop:q_consistency}.
\end{proof}

We summarize the overall procedure for learning an optimal invariant policy, see Algorithm~\ref{alg:learn_opt_policy}:
The algorithm iterates over all subsets $S \subseteq \{1, \dots, d\}$ and checks the invariance condition using the off-policy invariance test given in Algorithm~\ref{alg:test_inv}.
The choices of the hypothesis test $\psi^S$ and the test policy $\pi^S$ are discussed in Section~\ref{sec:target_test}. For each iteration, if the set $S$ is invariant, we learn an optimal policy $\pi^*_S$ within the policy space $\Pi^S$ and compute its estimated expected reward $\hat{\EX}^{\pi^*_S}(R)$ using $\var{off\_opt}$. 
Then, the algorithm returns an optimal policy $\pi^*_S$ such that the estimated expected reward $\hat{\EX}^{\pi^*_S}(R)$ is maximized. Lastly, the algorithm returns null if no invariant sets are found.

Algorithm~\ref{alg:learn_opt_policy} requires us to iterate over all subset $S \subseteq \{1,\dots,d\}$ which may be computationally intractable when $d$ is large. We suggest two approaches for reducing the computational complexity of the algorithm. First, one can use a variable screening method (e.g., Lasso regression \cite{tibshirani1996regression}) to filter out the variables that are not predictive of the reward. 
If an optimal invariant set is a subset of the Markov blanket $\MB(R)$ of the reward, applying a variable screening step prior to Algorithm~\ref{alg:learn_opt_policy} would not change the algorithm's output on the population level (see \cite{PBM16, rojas2018invariant, pfister2021SR}). This approach is particularly efficient when the Markov blanket is sparse, that is, $\abs{\MB(R)} \ll d$.

Second, one may apply a greedy search instead of the exhaustive search in Algorithm~\ref{alg:learn_opt_policy}. More specifically, we suggest to follow the greedy search introduced in \cite{rojas2018invariant}. The greedy algorithm starts with an empty set $\hat{S} = \emptyset$. For each iteration, we search over the neighboring sets of the candidate set $\hat{S}$, which are obtained by adding or removing one predictor to or from $\hat{S}$. If any of the neighboring sets are accepted by the invaraince test, we select the one with the highest expected reward. If the test rejects all the neighbors, we select a neighbor that yields the largest p-value of the test.

\begin{algorithm}[t!]
\caption{Learning an optimal invariant policy}
\label{alg:learn_opt_policy}
\SetAlgoLined
\KwIn{data $D = (D^{e_1}, \ldots, D^{e_L})$, off-policy optimization  $\var{off\_opt}$, hypothesis tests and test policies $\{(\psi^S, \pi^S)\}_{S \subseteq \{1,\dots,d\}}$
}
initialize maximum reward $\var{maxR} \leftarrow -\infty$ \;
initialize optimal invariant policy $\pi_{\inv}^* \leftarrow \var{null}$ \;
\tcp{loop over all subsets}
\For{$S \in \mathcal{P}(\{1,\dots,d\})$}{
    \tcp{test for invariance}
    $\var{is\_inv} \leftarrow \var{test\_inv}(D, \pi^S,\psi^S, S)$ \;
        \tcp{(see Algorithm~\ref{alg:test_inv})}
    \tcp{update best invariant set}
    \If{\var{is\_inv}}{
        $\pi^*_S, \hat{\EX}^{\pi^*_S}(R) \leftarrow \var{off\_opt}(D, \Pi^S)$ \;
        \If{$\var{maxR} < \hat{\EX}^{\pi^*_S}(R)$}{
            $\var{maxR} \leftarrow \hat{\EX}^{\pi^*_S}(R)$ \;
            $\pi_{\inv}^* \leftarrow \pi^*_S$ \;
        }
    }
}
\KwOut{optimal invariant policy $\pi_{\inv}^*$}
\end{algorithm}

\begin{algorithm}[t!]
\caption{Testing the invariance of a set $S$ with given test policy $\pi^S$}
\label{alg:test_inv}
\SetAlgoLined
\SetKwFunction{invcond}{$\var{test\_inv}$}
\SetKwProg{Fn}{Function}{:}{}
\Fn{\invcond{{\normalfont data $D = (D^{e_1}, \ldots, D^{e_L})$, test policy $\pi^S$,
hypothesis test $\psi^S$, target set $S$}
}}{
 \tcp{resampling according to $\pi^S$}
 \For{$e = e_1, \ldots, e_L$}{
  \For{$i=1$ \KwTo $\abs{D^e}$}{ 
    compute weights: $r_i^e \leftarrow \dfrac{\pi^S(a_i^e \mid x_i^{e,S})}{\pi^{0}(a_i^e \mid x_i^e)}$ \;
    }
    choose resampling size $m_e$ with GOF-heuristic in \cite{thams2021statistical} \;
    draw $D^{e, \pi^S} \coloneqq (D^{e}_{i_1}, \ldots, D^{e}_{i_{m_e}})$ from $D^{e}$ with prob.\ $\propto \prod_{\ell=1}^{m_e}r_{i_\ell}^e$ \;
 }
 $D^{\pi^S} \leftarrow (D^{e_1, \pi^S},\ldots, D^{e_L, \pi^S})$\;
 \tcp{verifying invariance condition}
 $\var{is\_invariant} \leftarrow \psi^S(D^{\pi^S})$ \;
 \KwRet\ $\var{is\_invariant}$
}
\end{algorithm}

\subsection{Specifications of the Target Test}
\label{sec:target_test}

The resampling procedure detailed in Algorithm~\ref{alg:test_inv} requires a hypothesis test for the $\mathcal{E}^{\obs}$-invariance null hypothesis that has power against the alternatives. We discuss one such test in Section~\ref{sec:inv_res_dist_test} below. Moreover, in Sections~\ref{sec:optimizing-power}~and~\ref{sec:target_policy:uniform}, we discuss two choices of the test policy that aim to improve the power of the resampling test.

\subsubsection{Invariant residual distribution test}
\label{sec:inv_res_dist_test}

We now detail a test $\varphi^S$ to test $\mathcal{E}^{\obs}$-invariance in the target sample. 
We first pool data from all environments into one dataset and estimate the conditional $\EX^{\pi^S}[R \mid X^S]$ using any prediction method (such as linear regression or a neural network). 
We then test whether the residuals $R - \EX^{\pi^S}[R\mid X^S]$ are equally distributed across the environments $e \in \mathcal{E}$, i.e., we split the sample back into $L$ groups (corresponding to the environments) and test whether the residuals in these groups are equally distributed (see also \cite{PBM16}, for example). 
We then define $\varphi^S$ to be the composition of these operations, that is, $\varphi^S$ returns $1$ if the test for equal distribution of the residuals is rejected.

In the simulation and the warfarin case study (Section~\ref{sec:simulation}~and~\ref{sec:warfarin}), we use the Kruskal-Wallis test \citep{kruskal1952use} to test whether the residuals have the same mean across environments; this test holds pointwise asymptotic level for all $\alpha \in (0,1)$ (see Proposition~\ref{thm:SIR-consistency-new} in Appendix~\ref{app:details-resampling}).
To obtain power against more alternatives, one could also use other tests, such as a two-sample kernel test with maximum mean discrepancy \citep{gretton2012kernel} and then correct for the multiple testing using Bonferroni-corrections (see also \cite{Rojas2016}, for example).

\subsubsection{Optimizing the test policy for power}\label{sec:optimizing-power}
To check whether a subset $S$ is invariant, we only need to test the $\mathcal{E}^{\obs}$-invariance for a single policy $\pi \in \Pi^S$ (see Proposition~\ref{prop:inv_set_NEW}).
This provides us with a degree of freedom that we can leverage. Intuitively, the non-invariance may be more easily detectable in some test policies compared to others. We can therefore try to find a policy that gives us the strongest signal for detecting non-invariance. 
We maximize the power of the test by minimizing the $p$-value of the test. In a population setting, this would return small $p$-values for  non-invariant  sets, whereas for invariant sets one would not be able to make the $p$-values arbitrarily small, since they are uniformly distributed. In a finite sample setting, this type of power optimization can lead to overfitting (which would break any level guarantees); to avoid this we use sample splitting.

As presented in Section~\ref{sec:resampling}, for each environment $e$, we obtain a target sample $D^{e,\pi^S}$ 
from a test policy $\pi^S$ by resampling the sample $D^e$ that was generated under the policy $\pi^{0}$, and then test $\mathcal{E}^{\obs}$-invariance in the target sample. The probabilities for obtaining the reweighted sample conditioned on the original sample are given by the importance weights, see Appendix~\ref{app:details-resampling}. Here, we optimize the ability to detect non-invariance over a parameterized subclass of $\Pi^S$,
\begin{align*}
    \Pi^\Theta_S \coloneqq \{\pi_\theta^S \mid \theta \in \Theta\},
\end{align*}
where $\Theta = \bigtimes_{a \in \mathcal{A}} \mathbb{R}^{|S|}$ 
and $\pi_\theta^S$ is a linear softmax policy, i.e., for all $x^S\in \mathbb{R}^{|S|}$ and $a \in \mathcal{A}$:
\begin{align*}
    \pi_\theta^S(a|x^S) = \frac{\exp\left(\theta_a^\top x^S\right)}{\sum_{a'}\exp\left(\theta_{a'}^\top x^S\right)}.
\end{align*}
This is the parameterization we chose in the experiments below, but other choices work, too.

To check for the $\mathcal{E}^{\obs}$-invariance condition of a subset $S$, the idea is then to find a policy $\pi_\theta^S \in \Pi^\Theta_S$ such that, in expectation, the test power is maximized, i.e., we need to solve the following optimization problem:
\begin{equation*}
    \argmax_{\theta \in \Theta} \EX\big[\var{pw}(D^{\pi_\theta^S})\mid D\big],
\end{equation*}
where $D \coloneqq (D^{e_1}, \ldots, D^{e_L})$ 
is all the observed data and $\var{pw}$ is a function that takes as input the reweighted sample $D^{\pi_\theta^S}$ and outputs the power of the test. 
Since we condition on $D$, the expectation is only with respect to the resampling of $D^{\pi_\theta^S}$.
For many invariance tests, the test power $\var{pw}(D^{\pi_\theta^S})$ cannot be directly obtained, but one can 
minimize the $p$-value of the test instead.
This motivates the objective function
\begin{equation}
    \argmin_{\theta\in\Theta} \EX\big[\var{pv}(D^{\pi_\theta^S})\mid D\big], \label{eq:p-value_op}
\end{equation}
where $\var{pv}$ is a function that takes as input the reweighted sample $D^{\pi_\theta^S}$ and outputs the p-value of the test. 
We then employ gradient-based optimization algorithms to solve the above optimization problem, where the gradient is derived using the log-derivative. More precisely, let $J(\theta) \coloneqq  \EX\big[\var{pv}(D^{\pi_\theta^S}) \mid D \big]$ be our objective function which now depends on the parameters $\theta$. The gradient of the objective function $J(\theta)$ can be derived as follows
\begin{align*}
    \nabla J(\theta) &= \nabla \EX\big[\var{pv}(D^{\pi_\theta^S})\mid D\big] \\
    &= \nabla \sum_{d} \P(D^{\pi_\theta^S} = d \mid D) \var{pv}(d) \\
    &= \sum_{d} \P(D^{\pi_\theta^S} = d \mid D) \nabla \log \P(D^{\pi_\theta^S} = d \mid D) \var{pv}(d) \\
    &= \EX\big[ \nabla \log \P(D^{\pi_\theta^S} \mid D) \var{pv}(D^{\pi_\theta^S}) \mid D \big].
\end{align*}
This expectation can be estimated by drawing repeated resamples $D^{\pi_\theta^S}$,
where $\P(D^{\pi_\theta^S}\mid D)$ is determined
by the resampling weights.
In practice, we apply stochastic gradient descent \citep{zhang2004solving},
i.e., at each iteration of the optimization we compute the gradient only from a single resample. 
As we argue in Appendix~\ref{sec:power-opt-REPL},
we can further speed up the optimization process substantially by a minor modification to the resampling weights, corresponding to sampling with replacement instead of distinct weights. 

The optimization yields a policy $\pi^*_{\theta}$ that approximately satisfies $\pi^*_{\theta}\in \argmin_{\pi_{\theta} \in \Pi^S} J(\theta)$. We can then use $\pi^*_{\theta}$ as a test policy for testing the invariance of $S$.
Lastly, to preserve the level of the statistical test, we split the original sample into two halves, perform the power optimization procedure on one half, and verify the invariance condition on the other half. The algorithm is presented in Algorithm~\ref{alg:power_opt} in Appendix~\ref{app:algorithm-power-opt}. We only use the approximation of the resampling weights for the power optimization and use the actual resampling weights
for the final resampling, so the level guarantee of Proposition~\ref{thm:SIR-consistency-new} in Appendix~\ref{app:details-resampling} still holds.

\subsubsection{Using a uniform target distribution}
\label{sec:target_policy:uniform}
Since the procedure in Section~\ref{sec:optimizing-power} may be computationally challenging, especially if the algorithm is repeated many times as in Section~\ref{sec:simulation}. 
A computationally simpler approach is for each $a \in \mathcal{A}$ to test invariance under the test policy $\pi_a \in \Pi^{\emptyset}$, which always chooses the action $a$, and then combine the resulting $p$-values using Bonferroni corrections \citep{dunn1961multiple}. Beyond computational simplicity, this has an additional benefit: 
Across environments there may be a cancelling effect of the difference in means due to different dependencies on the action in each environment. By testing the invariance of the conditional mean of the reward in each action, such cancelling effects are accounted for.

\subsection{Learning Causal Ancestors under Distributional Shifts}\label{sec:off_icp}
Sections~\ref{sec:learning_invariant_sets}~and~\ref{sec:resampling} discuss an approach to learn invariant sets from off-policy data. The learned invariant sets are then used to find an optimal invariant policy as discussed in Section~\ref{sec:algo_inv_policy}. Besides learning an optimal invariant policy, one can further use the proposed off-policy invariance test to analyze the causal structure. More specifically, the learned invariant sets allow us to look for potential observed causal ancestors $\AN(R)$\footnote{Formally, $\AN(R)\subseteq\{1,\ldots,d\}$ is defined as the set of indicies $j$ for which 
there is a directed path from $X^j$ to $R$
in $\mathcal{G}$.} of $R$ by taking the intersection of the accepted sets. This approach is similar to invariant causal prediction \citep{PBM16}, except that here, we employ the off-policy invariance test to account for the distributional shift between the initial and the test policies, and allow for hidden variables.

Now we outline a method for finding $\AN(R)$ from the offline data obtained from multiple environments 
$D^{e_1},\ldots,D^{e_L}$. For all $e_j \in \mathcal{E}^{\obs}$ and $S \subseteq \{1,\dots,d\}$, let us denote by $D^{e_j, \pi^S}$ a weighted resample of $D^{e_j}$, and $\psi^S$ an invariance test for the $\mathcal{E}^{\obs}$-invariance hypothesis $H_0(S, \pi^S, \mathcal{E}^{\obs})$ as discussed in Section~\ref{sec:resampling} and Appendix~\ref{app:details-resampling}. For ease of presentation, we assume that $n_{e_1} = \dots = n_{e_L} \eqqcolon n$. Then, 
we propose to estimate the causal ancestors of $R$ by
\begin{equation}
    \hat{S}^n_{\AN} \coloneqq \bigcap_{S: \psi^S(D^{e_1, \pi^S}, \dots, D^{e_L, \pi^S}) = 0} S. \label{eq:off_icp}
\end{equation}
We detail the whole procedure in Algorithm~\ref{alg:off_icp} in Appendix~\ref{app:alg:off_icp}. 
Proposition~\ref{prop:off_icp} shows that this method 
controls the probability of wrongly selecting an incorrect variable.

\begin{proposition}\label{prop:off_icp}
    Assume Setting~\ref{setting:scmfixed}, and that $\mathbf{S}^{\inv}$ is non-empty. Let $\hat{S}^n_{\AN}$ be the estimated set of causal ancestors given in \eqref{eq:off_icp} and assume that the invariance tests $\psi^S$ used in \eqref{eq:off_icp} have pointwise asymptotic level $\alpha \in (0, 1)$. It then holds that
    \begin{equation}
        \liminf_{n \to \infty} \P(\hat{S}^n_{\AN} \subseteq \AN(R)) \geq 1 - \alpha.
    \end{equation}
\end{proposition}
\begin{proof}
    See Appendix \ref{proof:prop:off_icp}.
\end{proof}

\section{Simulation Experiments}
\label{sec:simulation}
To verify our theoretical findings we perform two simulation experiments, where we consider a linear multi-environment contextual bandit setting similar to Example~\ref{ex:cf_cb} with the following SCM $\mathcal{S}(\pi, e)$ (which induces the graph shown in Figure~\ref{fig:example1b}):
\begin{equation*}
\begin{gathered}
    U \coloneqq \epsilon_U, \quad
    X^1 \coloneqq \gamma_e U + \epsilon_{X^1}, \quad
    X^2 \coloneqq \alpha_e + \epsilon_{X^2}, \\
    A \sim \pi(A \mid X^1, X^2), \quad
    R \coloneqq \beta_{A,1}X^2 + \beta_{A,2}U + \epsilon_R,
\end{gathered}
\end{equation*}
where $\epsilon_U, \epsilon_{X^1}, \epsilon_{X^2}, \epsilon_{R} \sim \mathcal{N}(0,1)$, $A$ takes values in the space $\{a_1, \ldots, a_L\}$, $\gamma_e$ and $\alpha_e$ are parameters that depend on the environment $e$, and $\beta_{a_1, 1}, \dots, \beta_{a_L, 1}, \beta_{a_1, 2}, \dots, \beta_{a_L, 2}$ are  parameters that are fixed across environments. Appendix~\ref{app:sim:datagen} contains details on how the parameters are chosen in the experiments. The code for all the experiments is available at \url{https://github.com/sorawitj/invariant-policy-learning}.

\subsection{Generalization and Invariance}
\label{sec:exp-oracle}
We first consider an oracle setting, where we know a priori which subsets are invariant. From our data-generating process, it follows that $\{X^2\}$
is the only invariant set. We then compare an invariant policy which depends only on $X^2$ with a policy that uses both $X^1$ and $X^2$. We train both policies on a dataset of size $10'000$ obtained from multiple training environments under a fixed initial policy $\pi^0$ (see Appendix~\ref{app:sim:init_policy}). In both cases,
we employ a weighted least squares
to estimate the expected reward $\EX[R \mid A, X^S]$, where $S$ is the subset that the policy uses. The policy then takes a greedy action w.r.t.\ the estimated expected reward, i.e., $\argmax_{a} \hat{\EX}[R \mid A=a, X^S]$ (see Section~\ref{sec:algo_inv_policy}). 
Then we evaluate both policies on multiple unseen environments
and 
compute the regret with respect to the policy that is optimal in each of the unseen environments. Figure~\ref{fig:oracle} 
shows the results.
Each data point represents the evaluation on an unseen environment. The $y$-axes show the regret value and the $x$-axes 
display the distance from each unseen environment to the training environments (the distance is computed as the $\ell^2$-distance between the average value of the pairs $(\gamma_{e_{tr}}, \alpha_{e_{tr}})$ in the training environments and the pair $(\gamma_e, \alpha_e)$ in the unseen test environment). 
The plot shows that the worst-case behavior of the invariant policy is smaller than the non-invariant one. In particular, for environments different from the training environments the gain can be significant. This empirically supports our result 
of Theorem~\ref{thm:inv_policy_NEW}.
\begin{figure}[t!]
    \centering
    \includegraphics[scale=.48]{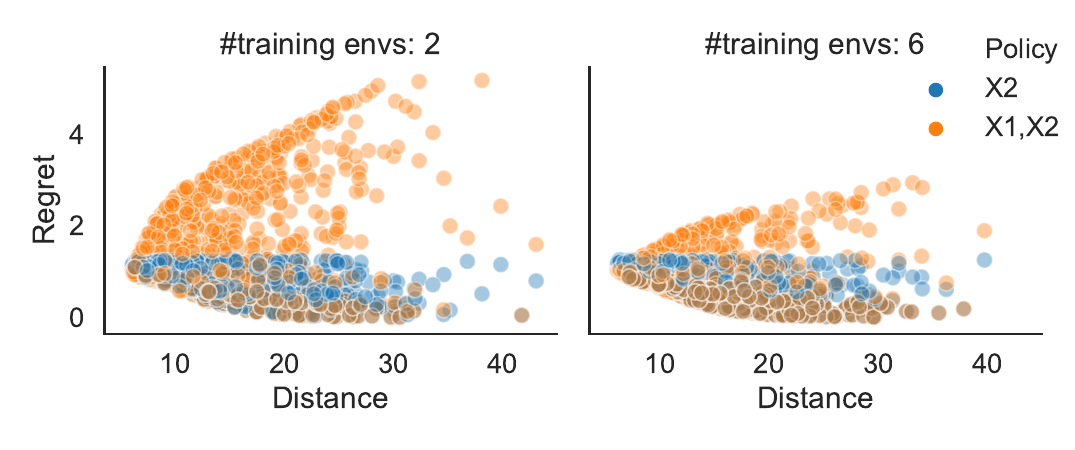}
    \caption{The generalization performance (in terms of regret) of the policy based on an invariant set $\{X^2\}$ and the policy based on a non-invariant set $\{X^1,X^2\}$. The left and the right plot show the results when the training environments consist of two and six different environments, respectively. In both cases, the worst-case regret for the invariant policy is upper bounded while this is not the case for the non-invariant policy.}
    \label{fig:oracle}
\end{figure}
\subsection{Learning Invariant  Policies}\label{sec:experiment_learning_invariant_policies}
In practice, we do not know in advance which sets are invariant.
We now aim to find an invariant policy from a dataset generated under an initial policy $\pi^0$ which takes both $X^1$ and $X^2$ as input.
To do so, we employ the method
proposed in Section~\ref{sec:resampling}
for testing invariance under distributional shifts. 
More precisely, we generate a dataset of size $n$ from multiple training environments under the initial policy $\pi^0$ and apply the off-policy invariance test (see Section~\ref{sec:target_test})
to verify the invariance property of each subset in $\{\varnothing, \{X^1\}, \{X^2\}, \{X^1, X^2\}\}$. We repeat the experiment $500$ times and plot the acceptance rates at various sample sizes ($n = 1'000, 3'000, 9'000, 27'000, 81'000$)
(these numbers denote the total sample size, that is, number of observations, summed over all environments).
The resulting acceptance rates are shown in Figure~\ref{fig:invarint_test}. 
Our method yields high acceptance rates for the set $\{X^2\}$, which  
indeed is invariant,
while the acceptance rates for other sets gradually decrease as the sample size increases. Furthermore, we can see that our test is more powerful when the number of training environments increases (keeping the total number of observations fixed). Our test is conservative (the acceptance rate is above the 95\% level in the left plot) because the target test is not exact (the true conditional expectation is not given). In Appendix~\ref{app:invariant_test_true_conditional}, we conduct the same experiment with an exact test, using the true conditional expectation, which shows the correct level.
\begin{figure}[t!]
    \centering
    \includegraphics[scale=.455]{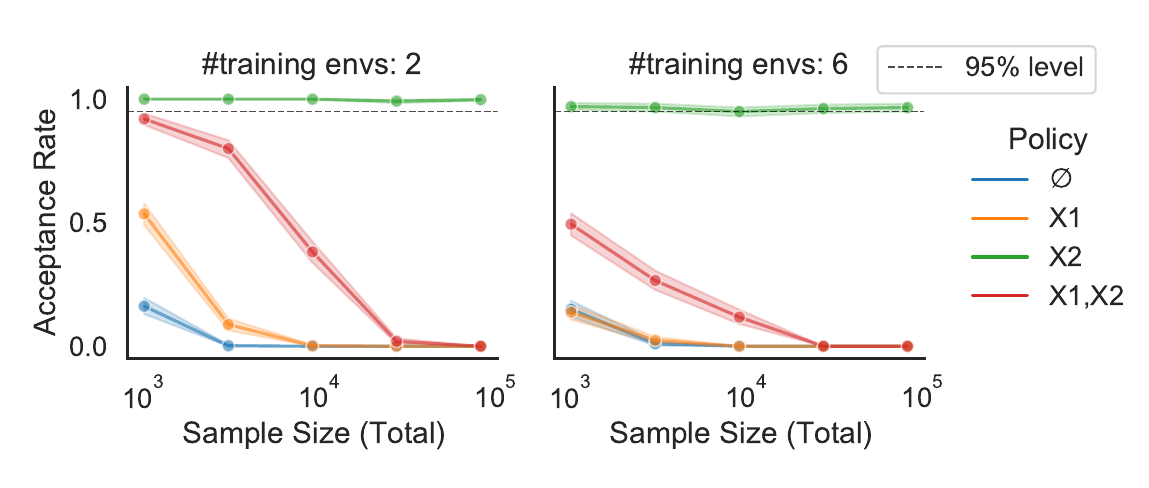}
    \caption{Acceptance rates for the off-policy invariance test proposed in Section~\ref{sec:resampling} for varying sample sizes. 
    With increasing sample size, only the invariant set $\{X^2\}$ is accepted. Here, more environments (right) seems to yield higher test power than fewer environments (left). 
    }
    \label{fig:invarint_test}
\end{figure}

\section{Warfarin Dosing Case Study}\label{sec:warfarin}
We evaluate our proposed approach on the clinical task of warfarin dosing. Warfarin is a blood thinner medicine prescribed to patients at risk of blood clots. The appropriate dose of warfarin varies from patient to patient depending on various factors such as demographic and genetic information \citep{international2009estimation}. Our case study is based on the International Warfarin Pharmacogenetics Consortium (IWPC) dataset \citep{international2009estimation} which consists of $5'700$ patients who were treated with warfarin, collected from 21 research groups on 4 continents. The IWPC dataset contains the optimal dose of warfarin for each of the patients as well as their information on demographic characteristics, clinical and genetic factors. The warfarin dosing problem has been used in a number of previous works evaluating off-policy learning algorithms \citep{kallus2018policy, NEURIPS2018_f337d999, zenati2020counterfactual}. Similarly to these works, we formulate the warfarin dosing problem as a multi-environment contextual bandit problem as follows.
\begin{itemize}
    \item The covariates $(X)$ are patient-level features including demographic, clinical and genetic factors.
    \item The actions $(A)$ are recommended warfarin doses output by a policy. We discretize the actions into three equal-sized buckets (low, medium, high) based on the quantiles of the optimal warfarin dose.
    \item The reward $(R)$ depends on the recommended dose and the optimal dose: For each patient $i$, the reward $R_i(a)$ for an action $a \in \{\text{low}, \text{medium}, \text{high}\}$ is computed as
    \begin{equation}
        R_i(a) \coloneqq \abs{Y_i - m(a)}, \label{eq:exp:reward_function}
    \end{equation}
    where $Y_i$ is the optimal warfarin dose for a patient $i$ and $m(a)$ is a median value of the optimal warfarin doses within the bucket $a$. Here, we assume that neither the reward function nor the optimal warfarin doses are known to the agent. Instead, for each patient $i$, only the reward for the action $A_i$ is observed, i.e., $R_i \coloneqq R_i(A_i)$.
    \item The environments $(\mathcal{E})$ are proxies for continents.
    The continent information is not directly contained in the dataset, but we create proxies for the continent by clustering the 21 research groups into 4 clusters based on their proportion of the patients' race within each group. We believe that the resulting clusters roughly correspond to 4 different continents.
\end{itemize}
To reduce the search space, we select the top 10 features that are most predictive for the optimal warfarin dose using the permutation feature importance method \citep{breiman2001random}. The top 10 features include 4 demographic variables, 4 clinical factors, and 2 genetic factors. 

We consider two experimental setups to illustrate the benefits of our invariant learning approach. In the first setup, we directly
apply our method to the IWPC dataset. Here, including invariance does not seem necessary in that our method performs similarly to other baselines (but not worse). It does, however, generate some causal insight into the problem. The second setup is a semi-real setting, where we introduce an artificial, non-invariant confounder.

\begin{figure}[t!]
    \centering
    \includegraphics[scale=.5]{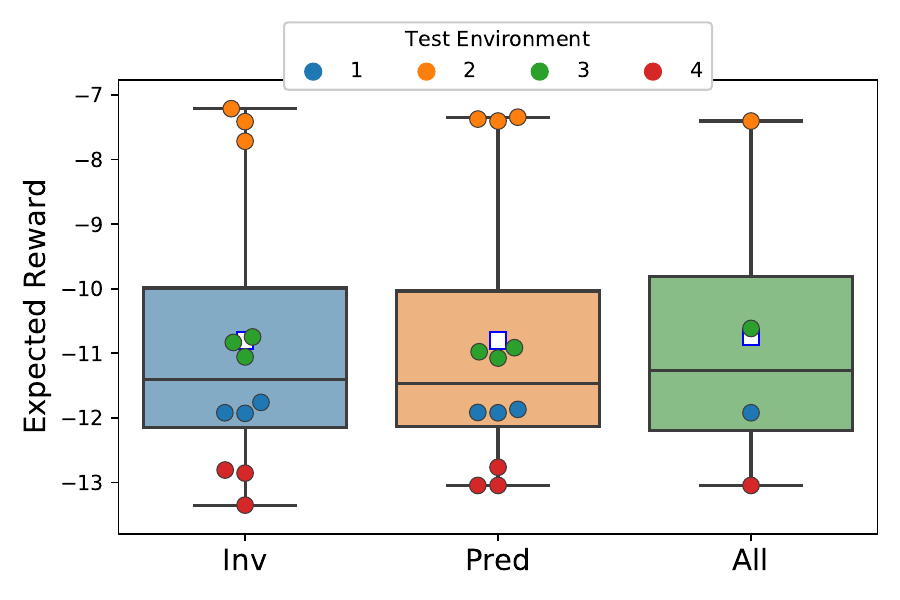}
    \caption{Empirical results on the original dataset. Each point represents the expected reward of a policy on the corresponding test environment. The square points represent the mean value of the expected rewards. In this setup, all candidate methods yield similar performances on all of the test environments.
    This result indicates that the test environments may not be significantly different from the training environments.}
    \label{fig:warfarin_no_hidden}
\end{figure}

We now outline our first experimental setup and the results. We first generate training data $\{(X_i, A_i, R_i, e_i)\}_{i=1}^n$ by drawing actions $A_i$ from a policy $\pi^0 \in \Pi^{\text{BMI}}$ that is constructed from linear regression $Y_i \approx f(X^\text{BMI}_i)$ of the optimal dose onto the BMI (see Appendix~\ref{app:warfarin:init_policy} for more details). 
\subsection{Candidate Methods}
\label{sec:warfarin:candidate_methods}
Using the generated training data, we empirically compare the performance of the following policy learning methods:
\begin{itemize}
    \item Invariant Policy Learning (Inv): This is our proposed method. We first perform the off-policy invariance test using the test described in Section~\ref{sec:target_policy:uniform} to search for potential invariant sets. We then take the top 20 sets with the largest p-values $\mathbf{S}^{20}_{\inv}$ as the candidate invariant sets. For each $S$ in $\mathbf{S}^{20}_{\inv}$, we fit the policy optimization algorithm described in Section~\ref{sec:algo_inv_policy} with $X^S$ as the covariates (the same algorithm is also used in other candidate methods below). Lastly, we select the top 3 sets that yield the largest expected rewards (computed using 5-fold cross-validation).
    \item Predictive Policy Learning (Pred): This method serves as a baseline for policy learning that solely maximizes the expected reward. 
    For each subset $S$, we fit the policy optimization algorithm with $X^S$ as the covariates. We then take the policies corresponding to the top 3 sets with the largest expected rewards.
    \item All Set Policy Learning (All): This method serves as another baseline where we take all of the patient's features and fit the policy optimization algorithm.
\end{itemize}
\subsection{Evaluation Setup \& Results}
\label{sec:warfarin:eval}
We compare the policy learning methods using the following `leave-one-environment-out' evaluation procedure.
\begin{enumerate}
    \item Select $e \in \mathcal{E} = \{1,\dots,4\}$ as a test environment. Split the training data into $D^{\tst} \coloneqq \{(X_i, A_i, R_i, e_i)\}_{i=1}^{n_{\tst}}$, where $e_i = e$ and $D^{\tr} \coloneqq \{(X_i, A_i, R_i, e_i)\}_{i=1}^{n_{\tr}}$, 
    where $e_i \in \{1,\dots,4\}\setminus\{e\}$.
    \item Using $D^{\tr}$, train the policies with candidate methods detailed in Section \ref{sec:warfarin:candidate_methods}.
    \item Evaluate the fitted policies by computing the expected reward on $D^{\tst}$ using the true reward function \eqref{eq:exp:reward_function}.
\end{enumerate}
We repeat the above procedure for each $e \in \mathcal{E}$ and
display the evaluation result in Fig~\ref{fig:warfarin_no_hidden}. The performances of all candidate methods are similar.
Even though the proposed invariant approach does not yield a higher reward compared with the baselines, it does not worsen the performance, either. This suggests that we can gain the stability benefit of an invariant  policy without having to sacrifice predictiveness. Indeed, the stability benefit could prevent the learned invariant policy from being suboptimal when a new test environment is sufficiently different from the training environments as we show in Section~\ref{sec:warfarin:semi-real}

\subsection{Analyzing Invariant Sets}
In addition to learning an optimal invariant policy, we can
use the invariance-based approach to further analyze the dependence between the patient's features and the reward as discussed in Section~\ref{sec:off_icp}. In particular, we apply the off-policy invariant causal prediction algorithm (see Algorithm~\ref{alg:off_icp} in Appendix~\ref{app:alg:off_icp}) to find potential causal ancestors of the reward.
On this dataset, with a confidence level of 5\%, the algorithm returns the empty set, which can happen if the covariates are highly correlated, for example \cite{HeinzeDeml2017}.
Nonetheless, we can still extract more information by obtaining the defining sets (see Section~2.2 in \cite{HeinzeDeml2017}). The resulting defining set of size 2 is 
 \{Race, VKORC1\} (see Appendix~\ref{app:warfarin:defining_set} for more details on the variables). These variables are potential causal ancestors in the sense that at least one variable in these sets is a causal ancestor. 
 
\begin{figure}[t!]
    \centering
    \includegraphics[scale=.5]{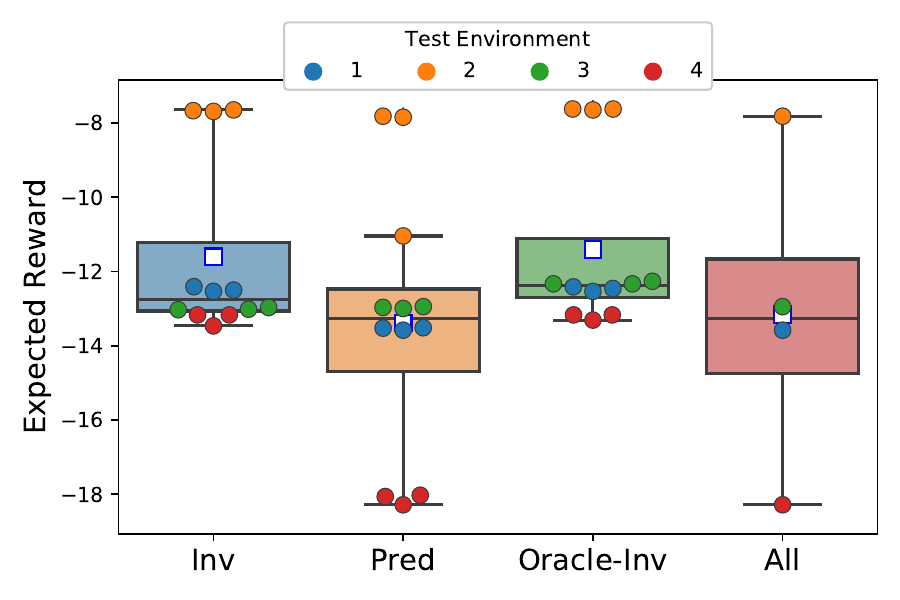}
    \caption{Empirical results on policy learning with a non-invariant predictor (see Section~\ref{sec:warfarin:semi-real}). Each point represents the expected reward of a policy on the corresponding test environment. 
    In this setup, our proposed method (Inv) outperforms the two baselines (Pred and All) that ignore the environment structure, while approaching the performance of the invariant oracle (Oracle-Inv).}
    \label{fig:warfarin_hidden}
\end{figure}

\subsection{Semi-real experiment} %
\label{sec:warfarin:semi-real}
To further illustrate the benefits of the invariance-based learning approach, we consider a semi-real setup where we introduce hidden variables and a non-invariant predictor. We remove the two genetic factors from the patient's features and create a non-invariant predictor that depends on those two factors as follows.

We first fit a linear regression to estimate the optimal warfarin dose from the genetic factors and denote the resulting coefficients by $\beta$. To mimic environmental perturbations, we  perturb $\beta$ depending an environment $e \in \mathcal{E}$ resulting in $\beta_e \coloneqq \gamma_e \beta$, where $\gamma_e$ is an environment-specific parameter. We define the non-invariant predictor in the environment $e \in \mathcal{E}$ as $X^{\text{n-inv}} \coloneqq {X^{G}}^{\top} \beta_e$, where $X^G$ are the two genetic features. We then add $X^{\text{n-inv}}$ as part of the patient's features and remove $X^{G}$.
The training data are generated in a similar fashion as in the first setup, except that the initial policy does not only depend on the BMI score $X^\text{BMI}$ but also on the non-invariant predictor $X^{\text{n-inv}}$.

In addition to the candidate methods described in Section~\ref{sec:warfarin:candidate_methods}, we introduce an additional baseline for this setup.
\begin{itemize}
    \item Oracle invariant Policy (Oracle-Inv): By construction, we know that $X^{\text{n-inv}}$ is a strongly non-$d$-invariant variable (see Definition~\ref{def:conf-removing}). This method serves as an oracle version of the invariant policy learning method by searching for the top 3 sets that do not contain $X^{\text{n-inv}}$ such that their corresponding policies yield the largest expected reward (the procedure is similar to the Pred method with $X^{\text{n-inv}}$ being removed).
\end{itemize}

We evaluate the candidate methods using a similar procedure as described in Section~\ref{sec:warfarin:eval}. Figure~\ref{fig:warfarin_hidden} illustrates the evaluation result. Our proposed method (Inv) yields a higher expected reward than the two baselines on most of the test environments. 
This is because the two baselines ignore the environment structure and use information from $X^{\text{n-inv}}$ in their resulting policies, while the invariant method uses the invariance test to remove this non-invariant proxy variable. Furthermore, the performance of our proposed method is almost on par with the invariant oracle (Oracle-Inv), except for the test environment $e = 3$,
in which our approach is unable to ignore the 
non-invariant predictor, possibly because the non-invariance that would be implied by 
Assumption~\ref{assm:mean_ff_NEW} may not be strong enough (for our test) when $\mathcal{E}^{\obs} = \{1, 2, 4\}$.

\section{Conclusion}
\label{sec:conclusion}
This paper tackles
the problem of environmental shifts in offline contextual bandits from a causal perspective. We introduce a framework for multi-environment contextual bandits that is based on structural causal models and frame the environmental shift problem as a distributionally robust objective over environments that are induced by different perturbations on the covariates.
We prove that if there are no unobserved confounders, taking into account causality and invariance is not necessary for obtaining the distributionally robust policies. However, causality and invariance can become relevant when not all variables are observed. To tackle settings with unobserved confounders, we adapt invariance-based ideas from causal inference to the proposed framework and introduce the notion of invariant policies. Our theoretical results show that under certain assumptions an invariant policy that is optimal on the training environments is also optimal on all unseen environments, and therefore distributionally robust. We further provide a method for finding invariant policies based on an off-policy invariance test. It can be combined with any existing policy optimization algorithm to learn an optimal invariant policy. We believe that our contributions shed some light on what causality can offer in contextual bandit and, more generally, in reinforcement learning problems.

For future work, there are several directions that would be interesting to investigate. One direction is to explore the use of invariance-based ideas in the adaptive setting, in which the goal of an agent is to
optimally adapt to a changing environment. Learning agents may require fewer and safer explorations in a new environment if they carry over invariance information from previous environments. It may further be possible to extend invariance-based ideas from the contextual bandit setting to the full reinforcement learning problem with long-term consequences and state dynamics. Although some previous works have explored this direction \citep{zhang2020invariant, sonar2020invariant}, we believe that the connections with respect to causality and invariance are not yet fully understood. In the i.i.d.\ setting, recent work has investigated trading off invariance and predictability \citep{rothenhausler2021anchor, pfister2021SR, jakobsen2020distributional, oberst2021regularizing, saengkyongam2022exploiting}. We believe that a similar idea can be applied to contextual bandit and reinforcement learning problems. Lastly, if one can gain additional knowledge of the test environments, one may aim to optimize objectives other than the worst-case performance which could lead to a different class of generalization guarantees.

This paper considers invariance as a dichotomous property and could be  a first step towards 
using invariance-based
ideas for building safer and more robust adaptive learning systems.

\section*{Acknowledgments}
SS, NT, and JP were supported by a research grant (18968) from VILLUM FONDEN and JP was, in addition, supported by the Carlsberg Foundation. NP was supported by a research grant (0069071) from Novo Nordisk Fonden. We thank Steffen Lauritzen for helpful discussions.

\bibliography{refs}
\clearpage

\begin{appendices}
\section{Pearl's $d$-separation}\label{app:d-sep}

\begin{definition}[Pearl’s $d$-separation \citep{pearl2009causality, peters2017elements}]
    Let $\mathcal{G}$ be directed acyclic graph (DAG) with nodes $\mathbf{V}$. Let $V_i, V_m \in \mathbf{V}$ and $\mathbf{S} \subseteq \mathbf{V} \setminus \{V_i, V_m\}$. A path between nodes $V_i$ and $V_m$ is said to be \emph{blocked by a set $\mathbf{S}$} if there exists a node $V_k \in \mathbf{V}$ such that one of the following holds:
    \begin{enumerate}
        \item $V_k \in \mathbf{S}$ and 
            \begin{align*}
                & V_{k-1} \rightarrow V_k \rightarrow V_{k+1} \\
                \text{or } & V_{k-1} \leftarrow V_k \leftarrow V_{k+1} \\
                \text{or } & V_{k-1} \leftarrow V_k \rightarrow V_{k+1},
            \end{align*}
        \item neither $V_k$ nor any of its descendants is in $\mathbf{S}$ and $$V_{k-1} \rightarrow V_k \leftarrow V_{k+1}.$$
    \end{enumerate}
    For any three disjoint subsets $\mathbf{A}, \mathbf{B}, \mathbf{S} \subseteq \mathbf{V}$ of nodes in $\mathcal{G}$, we say that $\mathbf{A}$ and $\mathbf{B}$ are \emph{$d$-separated} by $\mathbf{S}$, denoted by $\mathbf{A} \ci_{\mathcal{G}} \mathbf{B} \mid \mathbf{S}$ if every path between nodes in $\mathbf{A}$ and $\mathbf{B}$ is blocked by $\mathbf{S}$.
\end{definition}
(This formulation is taken from \cite{peters2017elements}.)

\section{Policy Learning without Unexplained Environment Shifts}\label{app:setting2_continue}

This section presents an assumption under which
it is not beneficial to explicitly take into account the environment structure. In particular, simply pooling the data from all training environments and applying a standard value-based policy learning algorithm yields a solution to \eqref{eq:maximin}. 
This result sheds light on the role of causality and invariance in contextual bandits and reinforcement learning. The relevant assumption is the following.
\begin{assumption}
\label{assm:no_uc}
Assume that $R \ci_{\mathcal{G}} e \mid X$.
\end{assumption}
In words, under Assumption~\ref{assm:no_uc}, the influence of the environment $e$ on the reward $R$ can be fully explained by the observed covariates $X$. This assumption holds, for example, if there are no hidden confounders between $X$ and $R$ that are directly affected by $e$.

The following proposition shows that under Assumption~\ref{assm:no_uc} there is a population optimal policy that does not depend on the environments. In particular, this optimal policy can be learned from data obtained in any observed subset of the environments $\mathcal{E}^{\obs} \subseteq \mathcal{E}$.

\begin{proposition}
\label{prop:no_uc}
    Assume Setting \ref{setting:scmfixed} and Assumption~\ref{assm:no_uc}.
    Let $\mathcal{E}^{\obs} \subseteq \mathcal{E}$ be a non-empty subset of observed environments and $\pi^*\in\Pi$ be a policy such that for all $x\in\mathcalbf{X}$ and all $a\in\mathcal{A}$ 
        \begin{equation}
    \label{eq:optimal_unconf_policy}
        \pi^*(a | x) > 0 \implies 
        a  \in \argmax_{a^\prime \in \mathcal{A}} Q^{\mathcal{E}^{\obs}}(x,a^\prime),
    \end{equation}
    where 
 $Q^{\mathcal{E}^{\obs}}(x,a) \coloneqq  \frac{1}{\abs{\mathcal{E}^{\obs}}}\sum_{e \in \mathcal{E}^{\obs}} \EX^{\pi_{a}, e}[R \mid X=x]$
    and $\pi_{a}$ is the policy that always selects $a$. Then,
    \begin{equation*}
        \pi^* \in \argmax_{\pi \in \Pi} V^{\mathcal{E}}(\pi),
    \end{equation*}
    i.e., $\pi^*$ is a solution to the maximin problem  \eqref{eq:maximin}.
\end{proposition}
\begin{proof}
See Appendix \ref{proof:thm:no_confound}.
\end{proof}
This type of generalization is well-established in the context of regression. In the contextual bandit setting the value function $\EX^{\pi, e}[R]$ changes across environments, so instead one needs to use that the Q-function $Q^{e}(x, a)=\EX^{\pi_a,e}[R\mid X=x]$ does not change across environments $e\in\mathcal{E}$ and then argue that this implies that the optimal policy remains the same in each environment.
Proposition~\ref{prop:no_uc} suggests that we can estimate an optimal policy by pooling the data from training environments and applying a standard value-based policy learning algorithm. This is indeed the case. 

Let $\widehat{Q}_n$ be an estimator of the conditional mean $\EX^{\pi_a}[R\mid X]$ that is based on $n$ independent observations $(X_i,A_i,R_i)$ from potentially different environments. The following proposition shows that such an approach indeed yields a consistent estimate of an optimal policy given that $\widehat{Q}_n$ is consistent.
\begin{proposition}
\label{prop:no_confound_NEW}
    Assume Setting~\ref{setting:scmfixed}
    and Assumption~\ref{assm:no_uc}. Let $\widehat{Q}_n$ be a uniformly consistent estimator of $Q^{\mathcal{E}^{\obs}}$, that is, 
    for all $a\in\mathcal{A}$ it holds that
    \begin{equation*}
        \lim_{n\rightarrow\infty}\EX_{D}\left[\sup_{x\in\mathcalbf{X}}\big\vert \widehat{Q}_n(x, a)-Q^{\mathcal{E}^{\obs}}(x,a)\big\vert\right]=0,
    \end{equation*}
    where $\EX_{D}$ is an expectation over the $n$ observations $(X_i, A_i, R_i)$ used to estimate $\widehat{Q}_n$.
    Let $\widehat{\pi}_n$ be any policy that maximizes $\widehat{Q}_n$, i.e.,
    for all $x\in\mathcalbf{X}$ and all $a\in\mathcal{A}$ it holds that
        \begin{equation*}
        \widehat{\pi}_n(a | x) > 0 \implies 
        a  \in \argmax_{a^\prime \in \mathcal{A}} \widehat{Q}_n(x,a^\prime).
    \end{equation*}
    Then, %
    the robust policy value converges towards its optimal value, that is
    \begin{equation*}
        \lim_{n\rightarrow\infty}\EX_{D}\left[\big\vert V^{\mathcal{E}}(\widehat{\pi}_n)-\max_{\pi \in \Pi} V^{\mathcal{E}}(\pi)\big\vert\right]=0.
    \end{equation*}
\end{proposition}
\begin{proof}
    See Appendix \ref{proof:prop:no_confound}.
\end{proof}
The same argument would work if instead of pooling, one considers only a single environment. In practice, however, one would make use of all available data. Whether it is possible to construct a uniformly consistent estimator $\widehat{Q}_n$ depends on the model class that can be assumed in the structural assignment of $R$, and on the policy used in generating the observations. For example, in the case of additive confounding and noise such as %
$f(X,U,A,\epsilon_R)=f_1(X, A)+f_2(U,\epsilon_R)$ 
with $f_1$ and $f_2$ in some function classes %
and a policy $\pi$ that has full support, (i.e., $\forall a \in \mathcal{A}, x \in \mathcal{X}: \pi(a \mid x) > 0$), 
one can consider a least squares estimator of the form
\begin{align*}
    \widehat{Q}^{\mathcal{E}^{\obs}}_n \coloneqq  \argmin_{f_1 %
    } \frac{1}{n} \sum_{i=1}^{n}(f_1(X_i, A_i) - R_i)^2.
\end{align*}
The assumptions of Proposition~\ref{prop:no_confound_NEW} are then satisfied
under further constraints on the function class and noise distributions, e.g., linear functions, Gaussian noise, and bounded domains.

\section{Backdoor Adjustment and Reweighting Approaches} \label{app:backdoor_vs_rew}

In our work, we take a reweighting approach for estimating the causal estimand (as presented in Section~\ref{sec:algo_inv_policy}) as opposed to using the backdoor adjustment formula. We have decided to use the reweighting approach for several reasons.
\begin{enumerate}[label=(\roman*)]
    \item To test the $\mathcal{E}^{\obs}$-invariance hypothesis, the reweighting approach allows us to use the resampling procedure \citep{thams2021statistical} that allows to apply arbitrary invariance tests while perserving pointwise asymptotic level. In contrast, it is not immediate how one can develop an asymptotically valid test based directly on the formula given by the backdoor adjustment.
    \item It has been shown in several works (see e.g., \cite{jung2020wbe}, \cite{jung2020wer}) that the reweighting approach may be more efficient in many settings than estimators derived from the backdoor adjustment formula.
\end{enumerate}

It has been shown in \cite{jung2020wer} that the two formulations imply equivalent identification results for causal effects. More specifically, any causal effect that is identified by do-calculus \citep{pearl1995causal} (e.g., by the backdoor adjustment formula) can also be identified via empirical risk minimization in a reweighted distribution. For completeness, we make this connection explicit for the setting considered in this work.

Consider the observed variables $(X, A, R)$ in Setting~\ref{setting:scmfixed} with a fixed initial policy $\pi^0$ in a fixed environment $e \in \mathcal{E}$. For notational convenience, we denote the density $p^{\pi^0,e}$ simply by $p$ and assume, without affecting the generality of the result, that all the variables are discrete. Fix $S \subseteq \{1,\dots,d\}$ and define $N \coloneqq \{1,\dots,d\} \setminus S$ as the complement of $S$. We are interested in the causal quantity $p^{\pi_a}(r | x^S)$ or equivalently, using do-notation, $p(r | \doo(a), x^S)$. (Note that the causal quantity $p(r | \doo(a), x^S)$ is the same for any $\pi^0$). Since the graph $\mathcal{G}$ in Setting~\ref{setting:scmfixed} satisfies (i) $X$-variables are non-descendants of $A$ and (ii) $\PA(A) \subseteq X$, we have that $X$ satisfies the backdoor criterion and hence the causal quantity $p(r | \doo(a), x^S)$ is identifiable via the adjustment formula (see \cite[Section 3.5]{pearl2016primer})
\begin{equation}
    p(r|\doo(a),x^S) = \sum_{x^N} p(r|a,x^N,x^S)p(x^N|x^S). \label{eq:bd}
\end{equation}
We will now show that $p(r|\doo(a),x^S)$ can also be expressed in terms of a reweighted distribution. To this end, we define a weighting factor $r(x^N, x^S, a) \coloneqq \frac{p(a)}{p(a|x^N, x^S)}$, and define a target distribution $q(x^N, x^S, a, r) \coloneqq r(x^N, x^S, a)p(x^N, x^S, a, r)$.
The following statements hold.
\begin{itemize}
\item Using an appropriate factorization, we get
\begin{align*}
    q(r,a,x^N,x^S) &= r(x^N, x^S, a)p(x^N, x^S, a, r) \\
    &= p(r|a,x^N,x^S)p(a)p(x^N|x^S)p(x^S). \numberthis \label{eq:ipw1}
\end{align*}
\item From \eqref{eq:ipw1}, we have
\begin{align*}
    q(x^S) &= \sum_{r, a, x^N} p(r|a,x^N,x^S)p(a)p(x^N|x^S)p(x^S) \\
    &= p(x^S). \numberthis \label{eq:ipw2}
\end{align*}
\item From \eqref{eq:ipw1}, we have
\begin{align*}
    q(a) &= \sum_{r, x^N, x^S} p(r|a,x^N,x^S)p(a)p(x^N|x^S)p(x^S) \\
    &= p(a). \numberthis \label{eq:ipw3}
\end{align*}
\item From \eqref{eq:ipw1} and \eqref{eq:ipw2}, we have
\begin{align*}
    q(r,a,x^N|x^S) &= q(x^N, x^S, a, r)\frac{1}{q(x^S)} \\
    &= p(r|a,x^N,x^S)p(a)p(x^N|x^S)p(x^S)\frac{1}{p(x^S)} \\
    &= p(r|a,x^N,x^S)p(a)p(x^N|x^S). \numberthis \label{eq:ipw4}
\end{align*}
\item From \eqref{eq:ipw1} and \eqref{eq:ipw2}, we have
\begin{align*}
    q(a|x^S) &= \frac{1}{q(x^S)}q(a,x^S) \\
    &= \frac{1}{p(x^S)} \sum_{r,x^N} q(r,a,x^N,x^S) \\
    &= \frac{1}{p(x^S)} \sum_{r,x^N} p(r|a,x^N,x^S)p(a)p(x^N|x^S)p(x^S) \\
    &= p(a) \\
    &= q(a). \numberthis \label{eq:ipw5}
\end{align*}
\end{itemize}
Now, we show that the backdoor adjustment formula for $p(r|do(a),x^S)$, given in \eqref{eq:bd}, can be expressed in terms of a conditional density in the target distribution $q$ as follows
\begin{align*}
    p(r|do(a),x^S) &= \sum_{x^N} p(r|a,x^N,x^S)p(x^N|x^S) && \text{by \eqref{eq:bd}} \\
    &= \sum_{x^N} p(r|a,x^N,x^S)p(x^N|x^S)\frac{p(a)}{p(a)} \\
    &= \frac{1}{p(a)}\sum_{x^N} q(r, a, x^N|x^S) && \text{by \eqref{eq:ipw4}}\\
    &= \frac{1}{p(a)}q(r, a|x^S) \\
    &= \frac{1}{q(a)}q(r, a|x^S) && \text{by \eqref{eq:ipw3}} \\
    &= q(r|a, x^S). && \text{by \eqref{eq:ipw5}}
\end{align*}

\section{Proofs}
\label{app:proofs}

\subsection{Proof of Proposition~\ref{prop:no_uc}}
\label{proof:thm:no_confound}
\begin{proof}
Let $e\in\mathcal{E}$, $a\in\mathcal{A}$ and $x\in\mathcalbf{X}$ be arbitrary. By the Markov property (see Lemma~\ref{lemma:markov_cond}) we have that $\EX^{\pi_a, e}\big[R \mid X=x \big]$ does not depend on the environment. This, in particular, implies that for all $e \in \mathcal{E}$, all $x \in \mathcalbf{X}$ and all $a \in \mathcal{A}$, it holds that
\begin{align}
    Q^{\mathcal{E}^{\obs}}(x,a) 
    &= \tfrac{1}{\abs{\mathcal{E}^{\obs}}}\sum_{f \in \mathcal{E}^{\obs}} \EX^{\pi_a, f}[R \mid X=x] \nonumber\\
    &= \EX^{\pi_a, e}[R \mid X=x].\label{eq:prop_no_u_indofe}
\end{align}
We thus have for all policies
$\pi\in\Pi$ and for all $x\in\mathcalbf{X}$ that
\begin{align*}
    \max_{a \in \mathcal{A}} Q^{\mathcal{E}^{\obs}}&(x,a) \\
    &= \max_{a \in \mathcal{A}} \EX^{\pi_a, e}[R \mid X=x] \\
    &\geq \sum_{a\in\mathcal{A}} \EX^{\pi_a, e}[R \mid X=x]\pi(a \mid x) \\
    &= \sum_{a\in\mathcal{A}} \EX^{\pi, e}[R \mid X=x, A=a]\pi(a \mid x) \\
    &=  \EX^{\pi,e}[R \mid X=x]. \numberthis \label{eq:prop_no_u_maxineq}
\end{align*}
Next, take the expectation over $X$ on both sides to get
\begin{align*}
    \EX^{e}\big[\max_{a \in \mathcal{A}} Q^{\mathcal{E}^{\obs}}(X,a)]
    &\geq \EX^{e}\big[\EX^{\pi,e}[R \mid X] \big]\\
    &= \EX^{\pi, e} \big[ R \big].
\end{align*}
Finally, taking the infimum over $e \in \mathcal{E}$ leads to
\begin{equation}
    \inf_{e\in\mathcal{E}}\EX^{e}\big[\max_{a \in \mathcal{A}} Q^{\mathcal{E}^{\obs}}(X,a)] \geq \inf_{e\in\mathcal{E}} \EX^{\pi,e} \big[ R \big]. \label{eq:optimal_qq}
\end{equation}
Let $\pi^*$ be a policy such that for all $x\in\mathcalbf{X}$ and all $a\in\mathcal{A}$
\begin{equation}
        \pi^*(a | x) > 0 \implies 
        a  \in \argmax_{a^\prime \in \mathcal{A}} Q^{\mathcal{E}^{\obs}}(x,a^\prime).
\end{equation}
Then $\pi^*$ satisfies, for all $e \in \mathcal{E}$, $$\EX^{\pi^*, e} \big[ R \big]=\EX^{e}\big[\max_{a \in \mathcal{A}} Q^{\mathcal{E}^{\obs}}(X,a)].$$ Therefore \eqref{eq:optimal_qq} implies
$$
\pi^* \in \argmax_{\pi \in \Pi} \inf_{e\in\mathcal{E}} \EX^{\pi,e} \big[ R \big],
$$
which completes the proof of Proposition~\ref{prop:no_uc}.
\end{proof}

\subsection{Proof of Proposition~\ref{prop:no_confound_NEW}}
\label{proof:prop:no_confound}
\begin{proof}
Define for all $n\in\mathbb{N}$ the term
\begin{equation*}
    c(n) \coloneqq \max_{a\in\mathcal{A}}\sup_{x\in\mathcal{X}}|Q^{\mathcal{E}^{\obs}}(x,a)-\widehat{Q}_n(x,a)|.
\end{equation*}
As $\mathcal{A}$ is assumed to be finite and because $\widehat{Q}_n$ is assumed to be uniformly consistent, it holds that
\begin{equation}
    \label{eq:thm1_unifrom_consistent}
    \lim_{n\rightarrow\infty}\EX_{D}[c(n)]=0.
\end{equation}
Moreover, as shown in \eqref{eq:prop_no_u_indofe}, in the proof of Proposition~\ref{prop:no_uc}, we know that for all $e\in\mathcal{E}$, all $a\in\mathcal{A}$ and all $x\in\mathcal{X}$ it holds that
\begin{equation*}
    Q^{\mathcal{E}^{\obs}}(x,a)=\EX^{\pi_a, e}[R\mid X=x].
\end{equation*}
This implies that for all $x\in\mathcalbf{X}$ and all $e\in\mathcal{E}$ it holds that
\begin{align}
    &\EX^{\widehat{\pi}_n, e}[R\mid X=x]\nonumber\\
    &\quad=\sum_{a\in\mathcal{A}}\EX^{\pi_a, e}[R \mid X=x]\widehat{\pi}_n(a|x)\nonumber\\
    &\quad=\sum_{a\in\mathcal{A}}Q^{\mathcal{E}^{\obs}}(x,a)\widehat{\pi}_n(a|x)\nonumber\\
    &\quad=\sum_{a\in\mathcal{A}}\widehat{Q}_n(x,a)\widehat{\pi}_n(a|x)\nonumber\\
    &\qquad\quad +\sum_{a\in\mathcal{A}}(Q^{\mathcal{E}^{\obs}}(x,a)-\widehat{Q}_n(x,a))\widehat{\pi}_n(a|x).\label{eq:pluginhere}
\end{align}
Each of the sums only contains one terms, since $\widehat{\pi}_n$ puts all mass on a single action.
Next, observe that
\begin{align}
    &\left\vert\sum_{a\in\mathcal{A}}(Q^{\mathcal{E}^{\obs}}(x,a)-\widehat{Q}_n(x,a))\widehat{\pi}_n(a|x)\right\vert\nonumber\\
    &\quad\leq\sum_{a\in\mathcal{A}}\left\vert Q^{\mathcal{E}^{\obs}}(x,a)-\widehat{Q}_n(x,a)\right\vert\widehat{\pi}_n(a|x)\nonumber\\
    &\quad\leq c(n)\label{eq:part1plugin}
\end{align}
and
\begin{align}
    &\sum_{a\in\mathcal{A}}\widehat{Q}_n(x,a)\widehat{\pi}_n(a|x)\nonumber\\
    &\quad=\max_{a\in\mathcal{A}}\widehat{Q}_n(x,a)\nonumber\\
    &\quad=\max_{a\in\mathcal{A}}Q^{\mathcal{E}^{\obs}}(x,a)\label{eq:part2plugin}\\
    &\qquad\qquad+ (\max_{a\in\mathcal{A}}\widehat{Q}_n(x,a)-\max_{a\in\mathcal{A}}Q^{\mathcal{E}^{\obs}}(x,a)).\nonumber
\end{align}
Using 
\eqref{eq:pluginhere}, \eqref{eq:part1plugin} and \eqref{eq:part2plugin}  together with the triangle inequality yields
\begin{align*}
    &\left\vert \EX^{\widehat{\pi}_n, e}[R\mid X=x] - \max_{a\in\mathcal{A}}Q^{\mathcal{E}^{\obs}}(x,a)\right\vert\\
    &\quad=\Big\vert\max_{a\in\mathcal{A}}\widehat{Q}_n(x,a)-\max_{a\in\mathcal{A}}Q^{\mathcal{E}^{\obs}}(x,a)\\
    &\qquad\qquad + \sum_{a\in\mathcal{A}}(Q^{\mathcal{E}^{\obs}}(x,a)-\widehat{Q}_n(x,a))\widehat{\pi}_n(a|x)\Big\vert\\
    &\quad\leq 2c(n).
\end{align*}
This in particular implies that for all $e\in\mathcal{E}$ and all $x\in\mathcalbf{X}$ it holds that
\begin{equation*}
    \max_{a\in\mathcal{A}}Q^{\mathcal{E}^{\obs}}(x,a)-2c(n)\leq \EX^{\widehat{\pi}_n, e}[R \mid X=x]
\end{equation*}
and that
\begin{equation*}
    \EX^{\widehat{\pi}_n, e}[R \mid X=x] \leq \max_{a\in\mathcal{A}}Q^{\mathcal{E}^{\obs}}(x,a)+2c(n).
\end{equation*}
Taking the expectation over $X$ and the infimum over $\mathcal{E}$ in both inequalities leads to
\begin{equation*}
  V^{\mathcal{E}}(\pi^*) - 2c(n)\leq V^{\mathcal{E}}(\widehat{\pi}_n)
  \leq V^{\mathcal{E}}(\pi^*) + 2c(n),
\end{equation*}
where $\pi^*$ is the policy defined in~\eqref{eq:optimal_unconf_policy}. Finally, we use \eqref{eq:thm1_unifrom_consistent} and Proposition~\ref{prop:no_uc} to get that
    \begin{equation*}
    \lim_{n\rightarrow\infty}\EX_{D}\left[\vert V^{\mathcal{E}}(\widehat{\pi}_n)- \max_{\pi \in \Pi} V^{\mathcal{E}}(\pi)\vert\right]
    \leq \lim_{n\rightarrow\infty}\EX_{D}[4c(n)]=0.
\end{equation*}
This completes the proof of Proposition~\ref{prop:no_confound_NEW}.
\end{proof}

\subsection{Proof of Lemma~\ref{lemma:inv_policy}}
\label{proof:lemma:inv_policy}
The key argument in the proof of Lemma~\ref{lemma:inv_policy} is a Markov property that we formulate as a lemma below.
\begin{lemma}[Extended Markov Property]
\label{lemma:markov_cond}
Assume Setting~\ref{setting:scmfixed}.
For all subsets $S \subseteq \{1,\dots, d\}$, it holds for all $Z\in\{U^1, \dots, U^p, R\}$ that
\begin{equation*}
\begin{gathered}
    Z \ci_{\mathcal{G}^S} e \mid X^{S} \\
    \Longrightarrow \\
    \begin{multlined}[t] \forall \pi\in\Pi^S:
    \P^{\pi, e}_{Z \mid X^S} \text{ is the same for all } e \in \mathcal{E},
    \end{multlined}
\end{gathered}
\end{equation*}
where the symbol $\ci_{\mathcal{G}}$ denotes d-separation in the graph $\mathcal{G}$.
\end{lemma}
Using Lemma~\ref{lemma:markov_cond}, the proof of Lemma~\ref{lemma:inv_policy} goes as follows.
\begin{proof}
Let $S^{\inv}$ be a $d$-invariant set and $\pi^{\inv} \in \Pi^{S^{\inv}}$ be a $d$-invariant policy with respect to $S^{\inv}$.
By Definition~\ref{def:inv_sets_NEW}, we have $R \ci_{\mathcal{G}^{S^{\inv}}} e \mid X^{S^{\inv}}$. It then holds by Lemma~\ref{lemma:markov_cond} for all $x \in \mathcal{X}^{S^{\inv}}$ and all $e, f \in \mathcal{E}$ that
\begin{equation*}
\EX^{\pi^{\inv},e}\big[R \mid X^{S^{\inv}} = x \big] = \EX^{\pi^{\inv},f}\big[R \mid X^{S^{\inv}} = x\big].
\end{equation*} 
\end{proof}

\subsubsection{Proof of Lemma~\ref{lemma:markov_cond}}
\begin{proof}
Lemma~\ref{lemma:markov_cond} corresponds to a global Markov property in the augmented graph (including the non-random environment index). Such results are well-established and used in settings in which $\mathcal{E}$ is finite, for example in influence diagrams \citep{dawid2002}. The result, however, also holds for more general, even uncountable $\mathcal{E}$.

To prove this, we first fix $S \subseteq \{1,\dots, d\}$, $\pi\in\Pi^S$ and $Z\in\{U,R\}$. Furthermore, let $e\in\mathcal{E}$, let $\Sigma$ be the discrete $\sigma$-algebra on $\mathcal{E}$ and let $\nu_e:\Sigma\rightarrow [0,1]$ be a probability measure that puts non-zero mass on $\{e\}$.
We can then replace the environment indicator in the SCM $\mathcal{S}(\pi,e)$ with a random variable $E$ with distribution $\nu_e$. This induces a joint distribution over $(E, X, U, A, R)$ that is globally Markov with respect to the graph $\mathcal{G}^S$, where $e$ is now replaced by $E$ (see \cite{pearl2009causality} Thm 1.4.1 or \cite{Lauritzen1990}). Additionally, it satisfies that $(X, U, A, R) \mid E=e$ has the same distribution as the distribution induced by $\mathcal{S}(\pi,e)$. 
Therefore the d-separation $Z \ci_{\mathcal{G}^S} E \mid X^{S}$ (which is implied by $Z \ci_{\mathcal{G}^S} e \mid X^{S}$) implies that the joint distribution $(E, X, U, A, R)$ satisfies the following conditional independence
\begin{equation}
    \label{eq:ci_ze}
    Z \ci E \mid X^{S}.
\end{equation}
Next, denote by $p^{\pi}$ the density of $(E, X, U, A, R)$ with respect to a product measure with the discrete measure as the $E$-component and for all $e\in\mathcal{E}$ denote by $p^{\pi, e}$ the induced density of $\mathcal{S}(\pi, e)$. Then, by construction of the densities and using the conditional independence in \eqref{eq:ci_ze} it holds that
for all $x\in\mathcalbf{X}^S$, all $z\in\operatorname{supp}(Z)$ and all $f\in\mathcal{E}$ with $\nu_e(f)>0$ that
\begin{align*}
    p^{\pi,f}(z\mid X^S=x)
    &=p^{\pi}(z\mid X^S=x, E=f)\\
    &=p^{\pi}(z\mid X^S=x)\\
    &\eqqcolon w_z(x),
\end{align*}
The function $w_z$ 
therefore no longer depends on the environment $f$ nor on $\nu_e$. Since $\nu_e(e)>0$, this in particular implies that for all $x\in\mathcalbf{X}^S$ and all $z\in\operatorname{supp}(Z)$ it holds that
\begin{equation*}
    p^{\pi, e}(z\mid X^S=x)=w_z(x).
\end{equation*}
As this construction works for all $e\in\mathcal{E}$, this completes the proof of Lemma~\ref{lemma:markov_cond}.
\end{proof}

\subsection{Stable Blanket and Invariance}
\label{sec:sb_invariance}
In this section, if not explicitly stated otherwise, all causal relations such as parents, descendants, ancestors etc. refer to the graph $\mathcal{G}$. Moreover, we use the convention that $k\in\text{DE}(X^k)$, where $\text{DE}(X^k)\subseteq\{1,\ldots,d\}$ denotes only the $X$-variable descendants of $X^k$. 
We first define the strongly non-$d$-invariant set: 
\begin{equation*}
   S_{\SNI}\coloneqq \{j\in\{1,\ldots,d\}\mid \exists k\in\text{CI}: j\in\DE(X^k)\},\\
\end{equation*}
where $\text{CI}$ are confounded and directly intervened on nodes (i.e., for $k\in\text{CI}$ there exists $\ell\in\{1,\ldots,p\}$ such that $e\rightarrow X^k \leftarrow \cdots\leftarrow U^{\ell}\rightarrow\cdots \rightarrow R$ in $\mathcal{G}$) and define $S_{\I} \coloneqq\{1,\ldots,d\}\setminus S_{\SNI}$. 
Furthermore, we define $S_R\subseteq\{1,\ldots,d\}$ to be the set of $X$-variables such that $j\in S_R$ if and only if $X^j\rightarrow R$ in $\mathcal{G}$ or there is a directed path $X^j\rightarrow \cdots\rightarrow R$ in $\mathcal{G}$, where $\cdots$ consists of $U$-variables.
The following Lemma will serve as a basis for our proofs of Proposition~\ref{prop:inv_policy_NEW} and Theorem~\ref{thm:inv_policy_NEW}.

\begin{lemma}[properties of $S_{\I}$]
\label{lemma:stable_invariant}
Assume Setting~\ref{setting:scmfixed} and Assumption~\ref{assm:U_causes_R}. Then, for all $S\in\mathbf{S}_{\inv}$, it holds that $S\subseteq S_{\I}$ and if a $d$-invariant set exists, it holds that $S_R\subseteq S_{\I}$, $S_{\I}$ is $d$-invariant and
\begin{equation*}
    j\in S_{\SNI}
    \quad\Longleftrightarrow\quad
    X^j \text{ is strongly non-$d$-invariant.}
\end{equation*}

\end{lemma}
\begin{proof}
    The proof is divided into four parts (S.1, S.2, S.3 and S.4):
    \begin{enumerate}[label=\textbf{S.\arabic*}] 
        \item \label{step1_lemm2} We prove that if $S \in \mathbf{S}_{\inv} $ then $S \subseteq S_{\I}$ by contraposition. Let $S \subseteq \{1,\dots,d\}$ be a subset such that there exists $j \in S$ but $j \in S_{\SNI}$. This implies that there exist $k\in\{1,\ldots,d\}$ and $\ell\in\{1,\ldots,p\}$ such that $e\rightarrow X^k \leftarrow \cdots\leftarrow U^{\ell}\rightarrow\cdots \rightarrow R$ in $\mathcal{G}$ and $j\in\text{DE}(X^k)$. Since $j\in\text{DE}(X^k)$, the path $e\rightarrow X^k \leftarrow \cdots\leftarrow U^{\ell}\rightarrow\cdots \rightarrow R$ is open given $X^S$, and therefore $R \nci_{\mathcal{G}}\, e \mid X^{S}$. By Definition~\ref{def:inv_sets_NEW}, this implies that $S$ is not $d$-invariant, leading to a contradiction. 
        \item \label{step2_lemm2}
        In this step, we prove that if a $d$-invariant set exists, it holds that $S_R \subseteq S_{\I}$. We prove this by contraposition.
            Assume that there exists $j \in S_R$ such that $j \in S_{\SNI}$. This implies that there exist $k\in\{1,\ldots,d\}$ and $\ell\in\{1,\ldots,p\}$ such that $e\rightarrow X^k \leftarrow \cdots\leftarrow U^{\ell}\rightarrow\cdots \rightarrow R$ in $\mathcal{G}$ and $j\in\text{DE}(X^k)$. Now, we construct a contradiction by showing that this would imply that no $d$-invariant set exists. Let $S \subseteq \{1,\dots,d\}$ be an arbitrary set. There are two possibilities,
            \begin{enumerate}
                \item[(a)] $j \in S$: Using the same argument as in \ref{step1_lemm2}, we have that $S$ is not a $d$-invariant set.
                \item[(b)] 
                $j \notin S$: Since $j \in S_R$ but $j \in S_{\SNI}$ there exists a directed path (using that $j\in\text{DE}(X^k)$)
                \begin{equation*}
                    e\rightarrow X^k\rightarrow\underbrace{\cdots}_{\text{part 1}}\rightarrow
                    X^j\rightarrow\underbrace{\cdots}_{\text{part 2}}\rightarrow R, 
                \end{equation*}
                where part~2 either has length zero or consists only of $U$-variables (by definition of $S_R$). The only way this path can be blocked by $X^S$ is if either $k$, $j$ or one of the variables in part~1 are contained in $S$. However, if this is the case the path $e\rightarrow X^k \leftarrow \cdots\leftarrow U^{\ell}\rightarrow\cdots \rightarrow R$ is open given $X^S$. Since the edges from $X$ to $A$ are not relevant in this case, this in particular means that $R \nci_{\mathcal{G}^S}\, e \mid X^{S}$, which by Definition~\ref{def:inv_sets_NEW} implies that $S$ is not $d$-invariant.
            \end{enumerate}
            As these are the only two possibilities, we have shown that no $d$-invariant set exists, which is a contradiction. Therefore $S_R\subseteq S_{\I}$.
        \item \label{lemm2_main_proof} Now we prove that if a $d$-invariant set exists, then $S_{\I}$ is $d$-invariant. In this step, all the graphical statements are understood to be taken in $\mathcal{G}^{S_{\I}}$. 
        Let $\rho$ be an arbitrary path from $e$ to $R$ in $\mathcal{G}^{S_{\I}}$. We first consider the two trivial cases. (i) $\rho$ enters $R$ through $A$, i.e., that it has the form
        $$e\rightarrow\cdots X^j\rightarrow A\rightarrow R.$$
        By construction of $\mathcal{G}^{S_{\I}}$ this path can only be in $\mathcal{G}^{S_{\I}}$ if $j\in S_{\I}$ which implies that it is blocked by $X^{S_{\I}}$. (ii) $\rho$ enters $R$ directly from $U$-variables. Let $U^{\ell}$ be the $U$-variable on $\rho$ that is closest to $e$, then $\rho$ has the form $e \rightarrow U^{\ell} \cdots \rightarrow R$. By Assumption~\ref{assm:U_causes_R}, there must be an edge from $U^{\ell}$ to $R$ and hence $\rho$ simplifies to $e \rightarrow U^{\ell} \rightarrow R$. Because $U^{\ell}$ is unobserved, any set $S \subseteq \{1,\dots,d\}$ would then not be $d$-invariant which contradicts to the assumption that there is a $d$-invariant set.
        Next we consider more involved cases, assume that $\rho$ enters $R$ either through a $U$- or $X$-variable. Let $U^{\ell}$ be the $U$-variable on $\rho$ that is closest to $e$ and $X^j$ be the $X$-variable on $\rho$ that is closest to $U^{\ell}$. We consider the two following cases:
        \begin{enumerate}
            \item[(1)] $U^{\ell}$ does not exist: This implies that $\rho$ does not contain any unobserved variables $U$ and hence $\rho$ can enter $R$ only through an X-variable. By \ref{step2_lemm2}, we have $S_R \subseteq S_{\I}$ and hence it holds that $\rho$ is blocked by $X^{S_{\I}}$.
            \item[(2)] $U^{\ell}$ exists: $\rho$ has the form 
            \begin{equation*}
                    \rho:\quad e\rightarrow X^r\underbrace{\cdots}_{\text{part 1}}U^{\ell} \underbrace{\cdots}_{\text{part 2}}\rightarrow R, 
                \end{equation*}
            where part 1 could be of length zero or it could consist of further $X$-variables and part 2 could be of length zero or it could consist of further $X$-or $U$-variables.
            By Assumption~\ref{assm:U_causes_R}, we have that there must be an edge from $U^{\ell}$ to $R$ and hence there exists another path
            \begin{equation*}
                    \tilde{\rho}:\quad e\rightarrow X^r\underbrace{\cdots}_{\text{part 1}}U^{\ell}\rightarrow R, 
            \end{equation*}
            where part 1 corresponds to the part 1 from path $\rho$. It suffices to show that $\tilde{\rho}$ is blocked by $X^{S_{\I}}$: whenever $\tilde{\rho}$ is blocked by $X^{S_{\I}}$, $\rho$ is blocked by $X^{S_{\I}}$ too (as $U^{\ell} \notin X^{S_{\I}}$. We now consider the following three cases for $\tilde{\rho}$:
        \begin{enumerate}
            \item[(i)] $\tilde{\rho}:\quad e\rightarrow\cdots X^j\rightarrow U^{\ell} \rightarrow R$,
            \item[(ii)] $\tilde{\rho}:\quad e\rightarrow\cdots \rightarrow X^j\leftarrow U^{\ell} \rightarrow R$,
            \item[(iii)] $\tilde{\rho}:\quad e\rightarrow\cdots X^k\leftarrow X^j\leftarrow U^{\ell} \rightarrow R$,
        \end{enumerate}
        in each case the $\cdots$ can also be of length zero. \\
        \textbf{Case (i):} 
        We show by contradiction that $\tilde{\rho}$ is blocked by $X^{S_{\I}}$. Assume $\tilde{\rho}$ is open given $X^{S_{\I}}$, i.e., $j \in S_{\SNI}$. Let $S \subseteq \{1,\dots,p\}$ be an arbitrary subset. If $j \in S$, then by the definition of $S_{\SNI}$ there exists $k\in\{1,\ldots,d\}$ and $c \in\{1,\ldots,p\}$ such that $e\rightarrow X^k \leftarrow \cdots\leftarrow U^{c} \rightarrow\cdots \rightarrow R$ in $\mathcal{G}$ and $j \in \DE(X^k)$ and hence $R \nci_{\mathcal{G}^{S_{\I}}} e\mid X^{S}$. If $j \notin S$, then the path $\tilde{\rho}$ is open given $X^S$ and hence $R \nci_{\mathcal{G}^{S_{\I}}} e\mid X^{S}$. Therefore, there is no $d$-invariant set which contradicts to the fact that a $d$-invariant set exists. \\
        \textbf{Case (ii):} In this case, $X^j$ is a collider on $\tilde{\rho}$. Assume $\tilde{\rho}$ has the form $e\rightarrow X^j\leftarrow U^{\ell}\rightarrow R$. This implies that $\DE(X^j) \subseteq S_{\SNI}$ and hence $\tilde{\rho}$ is blocked by $X^{S_{\I}}$. Thus, in order for $\tilde{\rho}$ to be open given $X^{S_{\I}}$ it must have the form $e\rightarrow\cdots X^k \rightarrow X^j\leftarrow U^{\ell}\rightarrow R$. Now, we consider the following two cases separately:
        \begin{enumerate}
            \item[(a)]  $k\in S_{\I}$: This directly implies that $\tilde{\rho}$ is blocked by $X^{S_{\I}}$.
            \item[(b)] $k\not\in S_{\I}$: By definition of $S_{\I}$ it holds that $\text{DE}(X^{k})\cap S_{\I}=\varnothing$. Hence, also $\text{DE}(X^j)\cap S_{\I} = \emptyset$ which since $X^j$ is a collider implies that $\tilde{\rho}$ is blocked by $X^{S_{\I}}$.
        \end{enumerate}
        We have therefore shown that in Case~(ii) the path $\tilde{\rho}$ is blocked by $X^{S_{\I}}$.\\
        \textbf{Case (iii):} In this case, let $X^c$ be the collider closest to $X^j$ on $\tilde{\rho}$. Again we consider two cases:
        \begin{enumerate}
            \item[(a)] $j\in S_{\I}$: This directly implies that $\tilde{\rho}$ is blocked by $X^{S_{\I}}$.
            \item[(b)] $j\not\in S_{\I}$: Since $\text{DE}(X^c)\subseteq\text{DE}(X^j)$, this implies that $\text{DE}(X^c)\cap S_{\I}=\varnothing$. Hence, the path $\tilde{\rho}$ is blocked by $X^{S_{\I}}$.
        \end{enumerate}
        We have therefore shown that in Case~(iii) the path $\tilde{\rho}$ is blocked by $X^{S_{\I}}$. Combining all cases, we have shown that any path $\tilde{\rho}$ from $e$ to $R$ is blocked by $X^{S_{\I}}$ in $\mathcal{G}^{S_{\I}}$.
        \end{enumerate}
    \item It remains to show that
        $$j\in S_{\SNI}
    \quad\Leftrightarrow\quad
    X^j \text{ is strongly non-$d$-invariant.}$$
    We show each direction separately. First, let $j\in S_{\SNI}$. By the definition of $S_{\SNI}$ it holds that there exists $k\in\{1,\ldots,d\}$ and $\ell\in\{1,\ldots,p\}$ such that $e\rightarrow X^k \leftarrow \cdots\leftarrow U^{\ell}\rightarrow\cdots \rightarrow R$ in $\mathcal{G}$ and $j \in \DE(X^k)$. As this path does not involve $A$ it is contained in $\mathcal{G}^{S}$ for all subsets $S\subseteq\{1,\ldots,d\}$. Moreover, since $X^k$ is a collider and the only $X$-variable on this path, it holds that this path will be open given $X^{S\cup\{j\}}$ for all subsets $S\subseteq\{1,\ldots,d\}$. Therefore, $X^j$ is strongly non-$d$-invariant.
    Next, to show the reverse direction let $j\in S_{\I}$. Then, by \ref{lemm2_main_proof} it holds that $S_{\I}$ is $d$-invariant. So in particular $R \ci_{\mathcal{G}^{S_{\I}}} e\mid X^{S_{\I}}$, which since $j\in S_{\I}$ implies that $j$ is not strongly non-$d$-invariant.
    \end{enumerate}
    This completes the proof of Lemma~\ref{lemma:stable_invariant}.
\end{proof}
As shown in Lemma~\ref{lemma:stable_invariant} the set $S_{\I}$ is $d$-invariant and contains all $d$-invariant sets if a $d$-invariant set exists. It will be used in the proofs of Proposition~\ref{prop:inv_policy_NEW} and Theorem~\ref{thm:inv_policy_NEW} to find the optimal $d$-invariant policy, as it encodes all invariant available information about the reward. The set $S_{\I}$ is related to stable blankets as defined in  \cite{pfister2021SR}.

\subsection{Proof of Proposition~\ref{prop:inv_policy_NEW}}
\label{proof:prop:inv_policy}
\begin{proof}
As before, we define $\pi_{a}(a^\prime \mid x) \coloneqq  \mathbbm{1}\big[ a^\prime = a \big]$ as the policy that always selects the action $a$. We first show an intermediate result and then turn to the proof of Proposition~\ref{prop:inv_policy_NEW}.
\begin{lemma}\label{lemma:argmax_invariant}
Let $S \in \mathbf{S}_{\inv}$ be an invariant set. It holds for all $e,f \in \mathcal{E}$ that
\begin{equation}
    \argmax_{\pi\in\Pi^S} \EX^{\pi, e}[R]=\argmax_{\pi\in\Pi^S} \EX^{\pi, f}[R].
\end{equation}
\end{lemma}
\begin{proof}
    Let $S \in \mathbf{S}_{\inv}$ be a $d$-invariant set and $e \in \mathcal{E}$ be an environment. By the same arguments as in \eqref{eq:prop_no_u_maxineq} we have, for all $\pi^S \in\Pi^S$, $x\in\mathcalbf{X}^{S}$ it holds that
\begin{equation*}
    \max_{a \in \mathcal{A}} \EX^{\pi_{a}}[R \mid X^{S} =x ] \geq  \EX^{\pi^S}[R \mid X^{S}=x],
\end{equation*}
where we drop the environment $e$ in the expectations as $S$ is $d$-invariant (see Lemma~\ref{lemma:inv_policy}).
Taking the expectation over $X^{S}$ on both sides yields
\begin{align*}
    \EX^{e}\big[\max_{a \in \mathcal{A}} \EX^{\pi_{a}}[R \mid X^{S} ] \big]
    &\geq \EX^{e}\big[\EX^{\pi^S}[R \mid X^{S}] \big]\\
    &= \EX^{\pi^S, e} [ R ]. \numberthis \label{eq:max_a_geq_policy}
\end{align*}
Let $\bar{\pi}\in\Pi^S$ be a policy that satisfies for all $a\in\mathcal{A}$ and for $\mu$-a.e. $x\in\mathcalbf{X}^{S}$
    \begin{equation}
    \label{eq:invariant_explicit_form_S}
        \bar{\pi}(a | x) > 0 \implies 
        a  \in \argmax_{a^\prime \in \mathcal{A}} \EX^{\pi_{a'}}\big[R \mid X^S = x \big].
    \end{equation}
We have that
\begin{align*}
    \EX^{\bar{\pi}, e}[R]
    &=\EX^{e}\big[\EX^{\bar{\pi}}[R\mid X^S]\big]\\
    &=\EX^{e}\big[\max_{a\in\mathcal{A}}\EX^{\pi_a}[R\mid X^S]\big].
\end{align*}
By \eqref{eq:max_a_geq_policy}, we then have that $\bar{\pi}\in\argmax_{\pi\in\Pi^S} \EX^{\pi, e}[R]$.
We now show the reverse direction, i.e., if $\pi^*\in\argmax_{\pi\in\Pi^S} \EX^{\pi, e}[R]$, then $\pi^*$ satisfies \eqref{eq:invariant_explicit_form_S}. 

Let $\pi^*\in\argmax_{\pi\in\Pi^S} \EX^{\pi, e}[R]$. By \eqref{eq:max_a_geq_policy}, we again have
\begin{equation*}
    \EX^{\pi^*, e}[R]=\EX^{e}\big[\max_{a\in\mathcal{A}}\EX^{\pi_a}[R\mid X^{S}]\big].
\end{equation*}
Since, for all $e \in \mathcal{E}$, the distribution of $X$ has full support (by the assumption in Setting~\ref{setting:scmfixed}), $\pi^*$ satisfies for all $a\in\mathcal{A}$ and for $\mu$-a.e. $x\in\mathcalbf{X}^{S}$
\begin{equation}
    \pi^*(a | x) > 0 \implies 
    a  \in \argmax_{a^\prime \in \mathcal{A}} \EX^{\pi_{a'}}\big[R \mid X^{S} = x \big]. \label{eq:explicit_optimal_invariant_S}
\end{equation}
Thus, the policy $\pi^*$ satisfies \eqref{eq:explicit_optimal_invariant_S} if and only if $\pi^*\in\argmax_{\pi\in\Pi^S} \EX^{\pi, e}[R]$. Furthermore, since \eqref{eq:explicit_optimal_invariant_S} does not depend on $e$, it then holds for all $e, f \in \mathcal{E}$ that
\begin{equation}
    \argmax_{\pi\in\Pi^S} \EX^{\pi, e}[R]=\argmax_{\pi\in\Pi^S} \EX^{\pi, f}[R]. \label{eq:argmax_invariant_no_e_S}
\end{equation}
\end{proof}

Now we prove the first statement of Proposition~\ref{prop:inv_policy_NEW}. To this end, let $$\pi^S_{\obs}\in\argmax_{\pi\in\Pi^S} \sum_{e \in \mathcal{E}^{\obs}} \EX^{\pi, e}[R].$$
By Lemma~\ref{lemma:argmax_invariant}, we have for all $e,f\in\mathcal{E}$ that $$\argmax_{\pi\in\Pi^S} \EX^{\pi, e}[R]=\argmax_{\pi\in\Pi^S} \EX^{\pi, f}[R].$$
So in particular, for all $e\in\mathcal{E}$, it holds that $$\pi^S_{\obs}\in\argmax_{\pi\in\Pi^S} \EX^{\pi, e}[R].$$ Thus  it holds for all $e\in\mathcal{E}$, all $S\in\mathbf{S}_{\inv}$ and all $\pi^S \in \Pi^S$ that
    \begin{equation*}
        \EX^{\pi^S_{\obs}, e}[R]
        \geq \EX^{\pi^S, e} [ R ].
    \end{equation*}
    Taking the infimum over $e \in \mathcal{E}$ on both sides yields
    \begin{equation*}
        V^{\mathcal{E}}(\pi^S_{\obs}) = \inf_{e \in \mathcal{E}} \EX^{\pi^S_{\obs}, e}[R] \geq \inf_{e \in \mathcal{E}} \EX^{\pi^S, e}[R] = V^{\mathcal{E}}(\pi^S).
    \end{equation*}
    Because this inequality holds for all $S \in \mathbf{S}_{\inv}$ and all $\pi^S \in \Pi^S$, this implies
    \begin{equation}
    \label{eq:properties_of_optimal_pooled_policy_S}
        \forall \pi \in \Pi^S: \quad V^{\mathcal{E}}(\pi^S_{\obs}) \geq  V^{\mathcal{E}}(\pi).
    \end{equation}
    This completes the proof of Proposition~\ref{prop:inv_policy_NEW_1}.

Next we prove the second statement of Proposition~\ref{prop:inv_policy_NEW}. The proof arguments are largely similar to those of the first statement except here we make use of the set $S_{\I}$ defined in Section~\ref{sec:sb_invariance} which is a superset of all other $d$-invariant sets. The proof is divided into two steps (S.1, S.2)
\begin{enumerate}[label=\textbf{S.\arabic*}] 
    \item In the first step, we derive an upper bound on the expected reward of arbitrary $d$-invariant policies.
    By Lemma~\ref{lemma:stable_invariant}, it holds that $S \subseteq S_{\I}$ for all $S \in \mathbf{S}_{\inv}$. We then have, for all $S\in\mathbf{S}_{\inv}$, $a\in\mathcal{A}$, and $e \in \mathcal{E}$, that
    \begin{equation}
    \EX^{\pi_a, e}[R \mid X^S, X^{S_{\I} \setminus S}] = \EX^{\pi_a, e}[R \mid X^{S_{\I}}], \label{eq:S_NI_property}
    \end{equation}
This is closely related to the predictiveness property of stable blankets (see \cite{pfister2021SR}).

Now, we expand the conditional expectation and get for all $e\in\mathcal{E}$ that
\begin{align*}
    &\EX^{e}\big[\max_{a \in \mathcal{A}} \EX^{\pi_{a}, e}[R \mid X^{S} ] \big] \\
    &= \EX^{e}\big[\max_{a \in \mathcal{A}} \EX^{e}_{X^{S_{\I}\setminus S}}\big[\EX^{\pi_{a}, e}[R \mid X^{S}, X^{S_{\I}\setminus S}]\big] \big],
\intertext{and by Jensen's inequality,}
    &\leq \EX^{e}\big[\EX^{e}_{X^{S_{\I}\setminus S}}\big[ \max_{a \in \mathcal{A}}\EX^{\pi_{a}, e}[R \mid X^{S}, X^{S_{\I}\setminus S}]\big] \big]\\
    &= \EX^{e}\big[ \max_{a \in \mathcal{A}}\EX^{\pi_{a}, e}[R \mid X^{S}, X^{S_{\I}\setminus S}] \big],\\
\intertext{and by \eqref{eq:S_NI_property},}
    &= \EX^{e}\big[\max_{a \in \mathcal{A}}\EX^{\pi_{a}, e}[R \mid X^{S_{\I}}]\big].
\end{align*}
Combining this with \eqref{eq:max_a_geq_policy}, we have
\begin{equation}
    \EX^{e}\big[ \max_{a \in \mathcal{A}}\EX^{\pi_{a}, e}[R \mid X^{S_{\I}}] \big] \geq \EX^{\pi^S, e} [ R ]. \label{eq:lower_bound_policy_SNI}
\end{equation}

\item In the second step, we are now ready to prove the main result of the statement. To this end, let $$\pi^*\in\argmax_{\pi\in\Pi_{\inv}} \sum_{e \in \mathcal{E}^{\obs}} \EX^{\pi, e}[R].$$
Since \eqref{eq:lower_bound_policy_SNI} holds for all $S\in\mathbf{S}_{\inv}$, we have for all $e \in \mathcal{E}$ that 
\begin{equation*}
    \argmax_{\pi \in \Pi^{S_{\I}}} \EX^{\pi,e}[R] = \argmax_{\pi \in \Pi_{\inv}} \EX^{\pi,e}[R].
\end{equation*}
Then, By Lemma~\ref{lemma:argmax_invariant} we have for all $e,f\in\mathcal{E}$ that $$\argmax_{\pi\in\Pi_{\inv}} \EX^{\pi, e}[R]=\argmax_{\pi\in\Pi_{\inv}} \EX^{\pi, f}[R].$$
So in particular, for all $e\in\mathcal{E}$, it holds that $$\pi^*\in\argmax_{\pi\in\Pi_{\inv}} \EX^{\pi, e}[R].$$ Thus  it holds for all $e\in\mathcal{E}$, all $S\in\mathbf{S}_{\inv}$ and all $\pi^S \in \Pi^S$ that
    \begin{equation*}
        \EX^{\pi^*, e}[R]
        \geq \EX^{\pi^S, e} [ R ].
    \end{equation*}
    Taking the infimum over $e \in \mathcal{E}$ on both sides yields
    \begin{equation*}
        V^{\mathcal{E}}(\pi^*) = \inf_{e \in \mathcal{E}} \EX^{\pi^*, e}[R] \geq \inf_{e \in \mathcal{E}} \EX^{\pi^S, e}[R] = V^{\mathcal{E}}(\pi^S).
    \end{equation*}
    Because this inequality holds for all $S \in \mathbf{S}_{\inv}$ and all $\pi^S \in \Pi^S$, this implies
    \begin{equation}
    \label{eq:properties_of_optimal_pooled_policy}
        \forall \pi \in \Pi_{\inv}: \quad V^{\mathcal{E}}(\pi^*) \geq  V^{\mathcal{E}}(\pi).
    \end{equation}
    This completes the proof of Proposition~\ref{prop:inv_policy_NEW_2}.
\end{enumerate}
\end{proof}

\subsection{Proof of Theorem~\ref{thm:inv_policy_NEW}}
\label{proof:thm:inv_policy}
\begin{proof}

We first prove the first statement of Theorem~\ref{thm:inv_policy_NEW}. 
Fix a policy
\begin{equation*}
    \pi^*\in\argmax_{\pi\in\Pi_{\inv}} \sum_{e \in \mathcal{E}^{\obs}} \EX^{\pi, e}[R].
\end{equation*}
Using the same argument as we made in Section~\ref{proof:prop:inv_policy}(\ref{lemm2_main_proof}), we get that for all $e\in\mathcal{E}$ it holds that
\begin{equation}
\label{eq:optimal_policy_expansion}
    \EX^{\pi^*, e}[R] = \EX^{e}\big[\max_{a \in \mathcal{A}}\EX^{\pi_{a}}[R \mid X^{S_{\I}}]\big].
\end{equation}
Hence by Jensen's inequality it holds that
\begin{align*}
    \EX^{\pi^*, e}[R] 
    &\geq \max_{a \in \mathcal{A}} \EX^{e}\big[\EX^{\pi_{a}}[R \mid X^{S_{\I}}]\big] \\
    &= \max_{a \in \mathcal{A}} \EX^{\pi_{a}, e}[R].
\end{align*}
This completes the proof of the first statement.

Next, we prove the second statement of Theorem~\ref{thm:inv_policy_NEW}. To do so, we use the following lemma, which is proved in Section~\ref{sec:proof_upperbound} below.

\begin{lemma}[Upper bound]
\label{lemma:ub-ninv}
Assume Setting~\ref{setting:scmfixed}, Assumptions~\ref{assm:U_causes_R},~\ref{assm:mean_ff_NEW} and~\ref{assm:strong_envs}, and that $\mathbf{S}_{\inv}\neq\varnothing$. Let $\pi \in \Pi \setminus \Pi_{\inv}$ be an arbitrary non-$d$-invariant policy. Then it holds that
\begin{equation*}
    V^{\mathcal{E}}(\pi) \leq \inf_{e \in \mathcal{E}} \EX^{e}_{X^{S_{\I}}}\big[\max_{a \in \mathcal{A}} \EX^{\pi_{a}}\big[R \mid X^{S_{\I}} \big]\big].
\end{equation*}
\end{lemma}
\noindent To finish the proof of Theorem~\ref{thm:inv_policy_NEW}, fix again a policy 
\begin{equation*}
    \pi^*\in\argmax_{\pi\in\Pi_{\inv}} \sum_{e \in \mathcal{E}^{\obs}} \EX^{\pi, e}[R].
\end{equation*}
Then, by Proposition~\ref{prop:inv_policy_NEW_2}, it holds that
\begin{equation}
    \label{eq:proof_thm2_part1}
    \forall \pi\in\Pi_{\inv}:\quad
    V^{\mathcal{E}}(\pi) \leq V^{\mathcal{E}}(\pi^*).
\end{equation}
Furthermore, Lemma~\ref{lemma:ub-ninv} together with \eqref{eq:optimal_policy_expansion} implies that
\begin{equation}
    \label{eq:proof_thm2_part2}
    \forall \pi\in\Pi\setminus\Pi_{\inv}:\quad
    V^{\mathcal{E}}(\pi) \leq V^{\mathcal{E}}(\pi^*).
\end{equation}
Combining \eqref{eq:proof_thm2_part1} and \eqref{eq:proof_thm2_part2} concludes the proof of Theorem~\ref{thm:inv_policy_NEW}.
\end{proof}

\subsubsection{Proof of Lemma \ref{lemma:ub-ninv}}\label{sec:proof_upperbound}
\begin{proof}

Recall the terminology and notation from Section~\ref{sec:sb_invariance}. 
The proof can be split into two parts:
\begin{enumerate}
    \item We first prove that if $e\in\mathcal{E}$ is a confounding removing environment it holds for all $\pi\in\Pi$ that
    \begin{equation}
        \label{eq:d-sepration_claim}
        \forall j\in S_{\SNI}: R \ci_{\mathcal{G}^{\pi,e}} X^j \mid X^{S_{\I}},A.
    \end{equation}
\item We then prove the upper bound using step 1) as the main argument.
\end{enumerate}

\noindent \textbf{Step 1)} Let $e\in\mathcal{E}$ be a confounding removing environment and fix $j\in S_{\SNI}$ and $\pi\in\Pi$.
By Lemma~\ref{lemma:stable_invariant} it holds that $X^j$ is strongly non-$d$-invariant.
Therefore, since $e$ is a confounding removing environment, we get that
\begin{equation}
\label{eq:unconfounded_j}
    X^j\ci_{\mathcal{G}^{\pi,e}}U.
\end{equation}
Now, let $\rho$ be an arbitrary path from $X^j$ to $R$ in $\mathcal{G}^{\pi, e}$. We consider the following (separate) cases that can occur:
\begin{enumerate}
    \item[(a)] $\rho$ enters $R$ through $A$: Then the path $\rho$ is blocked by $X^{S_{\I}}$ and $A$ because $A$ is not a collider and hence blocks $\rho$.
    \item[(b)] $\rho$ only contains $A$ and $X$-variables and enter $R$ through $X$-variables: Then there exists $k\in\{1,\ldots,d\}$ such that $\rho$ ends with $X^k\rightarrow R$. This implies that $k\in S_R$ since $\mathcal{G}^{\pi,e}$ is a sub-graph of $\mathcal{G}$. Furthermore, since by Lemma~\ref{lemma:stable_invariant} (recall that $\mathbf{S}_{\inv}\neq\varnothing$) $S_R\subseteq S_{\I}$, this implies that $k\in S_{\I}$. Hence, $\rho$ is blocked by $X^{S_{\I}}$ and $A$ because $X^k$ is not a collider.
    \item[(c)] $\rho$ contains at least one $U$-variable: Let $\ell\in\{1,\dots,p\}$ such that $U^{\ell}$ is the $U$-variable closest to $X^j$ on $\rho$, i.e., $\rho$ has the form
    $$\underbrace{X^j\cdots U^{\ell}}_{\gamma}\cdots\rightarrow R.$$
    Now, by \eqref{eq:unconfounded_j} it holds that $\gamma$ is blocked (given the empty set) in $\mathcal{G}^{\pi, e}$ and by construction it only consists of $X$-variables (except $U^\ell$). Therefore, there must be at least one collider on $\gamma$. Let $X^k$ be the collider closest to $U^{\ell}$ and let $X^m$ (this could be $X^j$) the variable that comes right before $X^k$ on $\gamma$, i.e.,
    $$X^j\cdots X^m\rightarrow X^k\leftarrow\cdots U^{\ell}.$$
    We consider two cases:
    \begin{enumerate}
        \item[(i)] First, assume that $\text{DE}(X^k)\cap S_{\I}\neq\varnothing$ (in $\mathcal{G}^{\pi,e}$), then it holds, by the definition of $S_{\I}$ and since $\mathcal{G}^{\pi, e}$ is a subgraph of $\mathcal{G}$, that $m\in S_{\I}$ as well (otherwise none of the descendants of $X^k$ could be in $S_{\I}$ as $\text{DE}(X^k)\subset\text{DE}(X^m)$). However, $X^m$ is not a collider and therefore $\rho$ is blocked given $X^{S_{\I}}$ and $A$.
        \item[(ii)] Second, assume $\text{DE}(X^k)\cap S_{\I}=\varnothing$, then it in particular holds that $k\in S_{\SNI}$ which by Lemma~\ref{lemma:stable_invariant} implies that $X^k$ is strongly non-$d$-invariant. Hence, because $e$ is a confounding removing environment, it holds that $X^k\ci_{\mathcal{G}^{\pi,e}}U$. However, $X^k$ was selected to be the collider closest to $U^{\ell}$ which means that the part of $\gamma$ from $X^k$ to $U^{\ell}$ is open in $\mathcal{G}^{\pi,e}$ leading to a contradiction.
    \end{enumerate}
\end{enumerate}
We have therefore shown that the path $\rho$ is always blocked given $X^{S_{\I}}$ and $A$. Since $\rho$ was arbitrary this implies that $R \ci_{\mathcal{G}^{\pi,e}} X^j \mid X^{S_{\I}}, A$.

\noindent \textbf{Step 2)}

Now, we are ready to prove the main result. Let $\pi \in \Pi\setminus\Pi_{\inv}$ be an arbitrary non-$d$-invariant policy, 
and let $S\subseteq\{1,\dots,d\}$ such that $\pi\in\Pi^S$.
We have
\begin{align*}
    &V^{\mathcal{E}}(\pi) \\
    &= \inf_{e \in \mathcal{E}} \EX^{\pi,e}\big[R\big], \\
\intertext{by the tower property of conditional expectation,}
    &= \inf_{e \in \mathcal{E}} \EX^{e}_{X^{S_{\I}}, X^{S \setminus S_{\I}}}\left[\EX^{\pi,e}\big[R \mid X^{S_{\I}},  X^{S \setminus S_{\I}}\big]\right] \\
    &= \begin{multlined}[t] \inf_{e \in \mathcal{E}} \EX^{e}_{X^{S_{\I}}, X^{S \setminus S_{\I}}} \bigg[ \int \EX^{\pi_a,e}\big[R \mid X^{S_{\I}},  X^{S \setminus S_{\I}}\big] \\
    \pi(a| X^S) \,\mu(\mathrm{d}a) \bigg].
    \end{multlined}
\end{align*}
Now, we use Assumption~\ref{assm:strong_envs}. For each $e \in \mathcal{E}$ we choose a confounding removing environment $f(e)$ such that $\P^{\pi,f(e)}_X = \P^{\pi,e}_X$. Because the confounding removing environments are a subset of $\mathcal{E}$, we have
\begin{align*}
    &V^{\mathcal{E}}(\pi) \\
    &= \begin{multlined}[t] \inf_{e \in \mathcal{E}} \EX^{e}_{X^{S_{\I}}, X^{S \setminus S_{\I}}} \bigg[ \int \EX^{\pi_a,e}\big[R \mid X^{S_{\I}},  X^{S \setminus S_{\I}}\big] \\
    \pi(a| X^S) \,\mu(\mathrm{d}a) \bigg]
    \end{multlined} \\
    &\leq \begin{multlined}[t] \inf_{e \in \mathcal{E}} \EX^{f(e)}_{X^{S_{\I}}, X^{S \setminus S_{\I}}} \bigg[ \int \EX^{\pi_a, f(e)}\big[R \mid X^{S_{\I}},  X^{S \setminus S_{\I}}\big] \\
    \pi(a| X^S) \,\mu(\mathrm{d}a) \bigg].
    \end{multlined}
\intertext{Using that $\P^{\pi,f(e)}_X = \P^{\pi,e}_X$, we then have}
    &V^{\mathcal{E}}(\pi) \\
    &\leq \begin{multlined}[t] \inf_{e \in \mathcal{E}} \EX^{e}_{X^{S_{\I}}, X^{S \setminus S_{\I}}} \bigg[ \int \EX^{\pi_a, f(e)}\big[R \mid X^{S_{\I}},  X^{S \setminus S_{\I}}\big] \\
    \pi(a| X^S) \,\mu(\mathrm{d}a) \bigg].
    \end{multlined}
\end{align*}
Next, we use \eqref{eq:d-sepration_claim} which states that for all $j\in\{1,\dots,d\}$ it holds that $R \ci_{\mathcal{G}^{\pi,e}} X^j \mid X^{S_{\I}}, A$.
Then, by the Markov property, we get
\begin{align*}
    &V^{\mathcal{E}}(\pi) \\
    &\leq \begin{multlined}[t] \inf_{e \in \mathcal{E}} \EX^{e}_{X^{S_{\I}}, X^{S \setminus S_{\I}}} \bigg[ \int \EX^{\pi_a,f(e)}\big[R \mid X^{S_{\I}}\big] \\
    \pi(a| X^S) \,\mu(\mathrm{d}a) \bigg],
    \end{multlined}
\intertext{we can then omit $f(e)$ since $S_{\I}$ is a $d$-invariant set (by Lemma~\ref{lemma:stable_invariant} since $\mathbf{S}_{\inv}\neq\varnothing$),}
    &= \begin{multlined}[t] \inf_{e \in \mathcal{E}} \EX^{e}_{X^{S_{\I}}, X^{S \setminus S_{\I}}} \bigg[ \int \EX^{\pi_a}\big[R \mid X^{S_{\I}}\big] \\
    \pi(a| X^S) \,\mu(\mathrm{d}a) \bigg]
    \end{multlined} \\
    &= \begin{multlined}[t]
    \inf_{e \in \mathcal{E}} \EX^{e}_{X^{S_{\I}}} \bigg[ \int \EX^{\pi_a}\big[R \mid X^{S_{\I}} \big] \\
    \EX^{e}_{X^{S \setminus S_{\I}}}\big[\pi(a| X^S)\big] \,\mu(\mathrm{d}a) \bigg],
    \end{multlined}
\intertext{letting $\tilde{\pi}(a | X^{S_{\I}}) \coloneqq \EX^e_{X^{S \setminus S_{\I}}}[\pi(a |X^S)]$,}
    &= \begin{multlined}[t] \inf_{e \in \mathcal{E}} \EX^{e}_{X^{S_{\I}}} \bigg[ \int \EX^{\pi_a}\big[R \mid X^{S_{\I}}\big]
    \tilde{\pi}(a \mid X^{S_{\I}})\,\mu(\mathrm{d}a) \bigg]
    \end{multlined} \\
    &\leq \inf_{e \in \mathcal{E}} \EX^{e}_{X^{S_{\I}}} \Big[ \max_{a \in \mathcal{A}} \EX^{\pi_{a}}\big[R \mid X^{S_{\I}} \big] \Big].
\end{align*}

\end{proof}
\subsection{Proof of Lemma~\ref{lemma:q_identification}}\label{proof:lemma:q_identification}
\begin{proof}
    Let $\pi^0$ be an initial policy generating the training data. For each $e \in \mathcal{E}^{\obs} \coloneqq \{e_1,\dots,e_L\}$ and $S \in \mathbf{S}_{\inv}$, we have that
\begin{align*}
    \EX^{\pi_a, e}[R \mid X^S] &= \EX^e[\EX^{\pi_a, e}[R \mid X] \mid X^S] \\
    &=\EX^e[\EX^{\pi^0, e}\big[\tfrac{\pi_a(A \mid X)}{\pi^0(A \mid X)} R \mid X\big] \mid X^S] \\
    &=\EX^e[\EX^{\pi^0, e}\big[\tfrac{\mathds{1}_{\{A = a\}}}{\pi^0(A \mid X)} R \mid X\big] \mid X^S] \\
    &=\EX^{\pi^0, e}\big[\tfrac{\mathds{1}_{\{A = a\}}}{\pi^0(A \mid X)} R \mid X^S\big] \\
    &=\EX^{\pi^0, e}\big[\tfrac{R}{\pi^0(A \mid X)} \mid X^S, A = a\big]. %
\end{align*}
We therefore have for all $x \in \mathcal{X}^S$ and $a \in \mathcal{A}$ that
\begin{equation*}
    Q^S(x,a) = \frac{1}{L}\sum_{\ell = 1}^{L}\EX^{\pi^0, e_{\ell}}\big[\tfrac{R}{\pi^0(A \mid X)} \mid X^S = x, A = a\big].
\end{equation*}
\end{proof}

\subsection{Proof of Proposition~\ref{prop:q_consistency}}\label{proof:prop:q_consistency}
\begin{proof}
Fix $S\in\mathbf{S}_{\inv}$, let $\tilde{\pi}$ be the uniform random policy given for all $x \in \mathcal{X}$ and all $a \in \mathcal{A}$ by
$$\tilde{\pi}(a|x) =\tfrac{1}{\abs{\mathcal{A}}}$$
and define $\eta(A, X, R, \theta) \coloneqq \frac{1}{\pi^0(A \mid X)} (f_{\theta}(A, X^S) - R)^2$. First, we rewrite the objective function in \eqref{eq:weighted_regressor} in terms of $\eta$ as follows
\begin{align*}
        H_n(\theta) &\coloneqq \tfrac{1}{L}\sum_{\ell = 1}^L \tfrac{1}{n_{e_\ell}}\sum_{i = 1}^{n_{e_\ell}} \tfrac{(f_{\theta}(A^{e_\ell}_i, {X^{e_\ell}_i}^{S}) - R^{e_\ell}_i)^2}{\pi^0(A^{e_\ell}_i \mid X^{e_\ell}_i)} \\
        &=  \tfrac{1}{L}\sum_{\ell = 1}^L \tfrac{1}{n_{e_\ell}}\sum_{i = 1}^{n_{e_\ell}} \eta(A^{e_\ell}_i, {X^{e_\ell}_i}, R^{e_\ell}_i, \theta).
\end{align*}
We now show that (a) $\EX[H_n(\theta)]$ is uniquely minimized at $\theta^S_0$ and (b) $H_n(\theta)$ satisfies the weak uniform law of large numbers. Then, together with (i)~and~(iii), Theorem~2.1 in \cite{NEWEY19942111} implies that $\hat{\theta}_n^S \rightarrow \theta^S_0$ in probability as desired.

First, we show (a).
    Taking the expectation, we have
    \begin{align*}
    \EX\big[H_n(\theta) \big] = \tfrac{1}{L}\sum_{\ell = 1}^L\EX^{\pi^0, e_{\ell}}\big[\eta(A, X, R, \theta)\big]. \numberthis \label{eq:proof_q_consistency_1}
    \end{align*}
    Next, let $e \in \mathcal{E}^{\obs}$ be any observed environment, it holds for all $a \in \mathcal{A}$ and for $\mu$-a.e. $x \in \mathcal{X}^S$ that
    \begin{align}
        \EX^{\tilde{\pi}, e}[R \mid A = a, X^S = x] &= \EX^{\pi_a, e}[R \mid X^S = x]\nonumber \\
        &\overset{(*)}{=} \tfrac{1}{L} \sum_{\ell = 1}^{L}\EX^{\pi_a, e_{\ell}}[R \mid X^S = x] \nonumber\\
        &= Q^S(a,x)\nonumber\\
        &\overset{(**)}{=} f_{\theta_0^S}(a,x)\label{eq:ftheta_conditional_mean}
    \end{align}
    where $(*)$ holds because $X^S$ is $d$-invariant, and $(**)$ holds by (ii). Now, since the conditional mean $\EX^{\tilde{\pi}, e}[R \mid A = a, X^S = x]$ is the (almost surely) unique minimizer of $$\EX^{\tilde{\pi}, e}[(f_{\theta}(A, X^S) - R)^2],$$ \eqref{eq:ftheta_conditional_mean} implies that $\theta^S_0 = \argmin_{\theta \in \Theta^S}\EX^{\tilde{\pi}, e}[(f_{\theta}(A, X^S) - R)^2]$. Furthermore, using that
    \begin{align*}
        \EX^{\tilde{\pi}, e}[(f_{\theta}(A, X^S) - R)^2] &= \EX^{\pi^0, e}[\tfrac{\tilde{\pi}(A \mid X)}{\pi^0(A \mid X)}(f_{\theta}(A, X^S) - R)^2] \\
        &= \tfrac{1}{\abs{\mathcal{A}}}\EX^{\pi^0, e}[\eta(A, X, R, \theta)],
    \end{align*}
    we get that
    \begin{align*}
        \theta^S_0 &= \argmin_{\theta \in \Theta^S} \EX^{\tilde{\pi}, e}[(f_{\theta}(A, X^S) - R)^2] \\
        &= \argmin_{\theta \in \Theta^S} \EX^{\pi^0, e}[\eta(A, X, R, \theta)] \\
        &= \argmin_{\theta \in \Theta^S} \tfrac{1}{L}\sum_{\ell=1}^{L} \EX^{\pi^0, e_{\ell}}[\eta(A, X, R, \theta)] \\
        &= \argmin_{\theta \in \Theta^S} \EX[H_n(\theta)].
    \end{align*}
    Hence, we have shown (a).
    
    Next, we show (b), that is, the objective function $H_n(\theta)$ satisfies the weak uniform law of large numbers. By (v) and
    (iv) we have for all $e \in \mathcal{E}^{\obs}$ that $\EX^{\pi^0,e}[\sup_{\theta\in\Theta^S}\eta(A,X,R,\theta)] < \infty$. Furthermore, define $h_{\ell}(\theta) \coloneqq \tfrac{1}{n_{e_{\ell}}}\sum_{i=1}^{n_{e_{\ell}}}{\eta(A^{e_\ell}_i, {X^{e_\ell}_i}, R^{e_\ell}_i, \theta)}$, then by Lemma~2.4 in \cite{NEWEY19942111} it holds for all $\ell \in \{1,\dots,L\}$ that $\sup_{\theta \in \Theta^S} \norm{h_{\ell}(\theta) - \EX^{\pi^0,e_{\ell}}[\eta(A,X,R,\theta)]}_2 \rightarrow 0$ in probability as $n_{e_{\ell}} \rightarrow \infty$ and hence that $\sup_{\theta \in \Theta^S} \norm{H_{n}(\theta) - \EX[H_{n}(\theta)]}_2 \rightarrow 0$ in probability as  $n_{e_1},\dots,n_{e_L} \rightarrow \infty$, which proves (b).
    
    As explained at the beginning, the result now follows by applying Theorem~2.1 in \cite{NEWEY19942111}. This completes the proof of Proposition~\ref{prop:q_consistency}.
\end{proof}

\subsection{Proof of Proposition~\ref{prop:inv_set_NEW}}
\label{proof:prop:inv_set}
\begin{proof} 
Fix a set $S\subseteq\{1,\dots,p\}$, and let $\pi,\tilde{\pi} \in \Pi^S$. Assume $H_0(S,\pi,\mathcal{E})$ is true. By Assumption~\ref{assm:mean_ff1_NEW}, we have that $R \ci_{\mathcal{G}^S} e \mid X^{S}$. Furthermore, since $\tilde{\pi}\in\Pi^S$
this implies by Lemma~\ref{lemma:markov_cond} that $\P^{\tilde{\pi}, e}_{R \mid X^{S^{\inv}}}$ is the same for all $e \in \mathcal{E}$
which implies that $H_0(S,\tilde{\pi},\mathcal{E})$ is true. This concludes the proof of Proposition~\ref{prop:inv_set_NEW}.
\end{proof}

\subsection{Proof of Proposition~\ref{prop:off_icp}}
\label{proof:prop:off_icp}
\begin{proof}
Let $S^* \coloneqq \AN(R)$ be the set of observed ancestors of $R$. In this proof, all the graphical statements are understood to be taken in $\mathcal{G}^{S^*}$. 
Assume that a $d$-invariant set $S$ exists. 
Then, by Theorem~2 of \cite{tian1998finding}, $S \cap S^*$ is $d$-invariant, too (indeed, $S$ intersected with all ancestors of $R$ is $d$-invariant but as $S$ does not contain any hidden variables, this set equals $S \cap S^*$).

We are now ready to prove the statement of the proposition. From \eqref{eq:off_icp} we have,
\begin{align*}
   &\liminf_{n \to \infty}  \P(\hat{S}^n_{\AN} \subseteq S^*) \\ 
   = &\liminf_{n \to \infty} \P(\bigcap_{S: \psi^S(D^{e_1, \pi^S}, \dots, D^{e_L, \pi^S}) = 0} S \subseteq S^*)  \\
   \geq &\liminf_{n \to \infty} \P(\psi^{S \cap S^*}(D^{e_1, \pi^{S \cap S^*}}, \dots, D^{e_L, \pi^{S \cap S^*}}) = 0)  \\
    \geq &1 - \alpha,
\end{align*}
where the last inequality follows by Proposition~\ref{thm:SIR-consistency-new}. This completes the proof of Proposition~\ref{prop:off_icp}.
\end{proof}

\section{Connection to Random Environments}
It is possible to define multi-environment contextual bandits using random environments.
\label{app:scmrandom}
\begin{setting}[Random Environment Contextual Bandits]
\label{setting:scmrandom}
Let $X = (X^1, \dots, X^d) \in \mathcalbf{X} = \mathcal{X}^1\times\hdots\times\mathcal{X}^d$, $U = (U^1, \dots, U^p) \in \mathcalbf{U} = \mathcal{U}^1\times\hdots\times\mathcal{U}^p$, $A \in \mathcal{A} = \{a^1, \dots, a^k\}$, $R \in \mathbb{R}$, $E \in \mathcal{E}$. For any $\pi \in \{\pi: \mathcalbf{X} \xrightarrow{} \Delta(\mathcal{A})\}$, let $g_{\pi}$ denote 
the function that ensures, for all $x \in \mathcalbf{X}$,
$g_{\pi}(x, \epsilon_A)$ equals $\pi(x)$ in distribution for a uniformly distributed $\epsilon_A$. Now, consider functions $s$, $h$, and $f$, a factorizing distribution $\P_{\epsilon} = \P_{\epsilon_E} \times \P_{\epsilon_U} \times \P_{\epsilon_X} \times \P_{\epsilon_A} \times \P_{\epsilon_R}$ whose $\epsilon_A$ component is uniform, and a structural causal model $\mathcal{S}(\pi)$ given by
\begin{equation*}
\mathcal{S}(\pi):\quad
\begin{cases}
E\coloneqq \epsilon_E\\
U\coloneqq s(X, U, E, \epsilon_U)\\
X\coloneqq h(X, U, E,\epsilon_X)\\
A\coloneqq g_{\pi}(X, \epsilon_A)\\
R\coloneqq f(X, U, A, \epsilon_R).
\end{cases}
\end{equation*}
Assume further that for all $\pi$, the SCM induces a unique distribution over $(E, X, U, A, R)$, which we denote  by $\P^{\pi}$. The structure of the SCM $\mathcal{S}(\pi,e)$ can be also visualized by a graph $\mathcal{G}$ which is constructed in a similar way to the graph in Setting~\ref{setting:scmfixed}, except that the environment becomes one of the variable nodes in this graph.
\end{setting}
\begin{remark}
Setting~\ref{setting:scmrandom}
is a special case of 
Setting~\ref{setting:scmfixed}
in the following sense:
Assume, starting from Setting~\ref{setting:scmrandom}, for all $i\in\{1,\ldots,n\}$ that 
$(X_i, U_i, A_i, R_i, E_i)$, are independent and distributed according to $\P^{\pi_i}_{X, U, A, R, E}$. 
Then, defining $h_e(\cdot,\cdot) := h(\cdot, e, \cdot)$, we  have that, for all $i\in\{1,\ldots,n\}$, $(X_i, U_i, A_i, R_i)$,
are independent and distributed according to $\P^{\pi_i, E_i}_{X, U, A, R}$, using Setting~\ref{setting:scmfixed}. 
\end{remark}

\section{Details for Section~\ref{sec:resampling}}\label{app:details-resampling}
In Section~\ref{sec:resampling}, we propose to use the resampling procedure in \cite{thams2021statistical} to test the hypothesis of invariance under a test policy $\pi^S \in \Pi^S$.

For every $e \in \mathcal{E}^{\obs}$, we have a dataset $D^e$ consisting of $n_e$ observations $D_i^e = (X_i^e, A_i^e, R_i^e, \pi^0(A_i^e|X_i^e))$ is available.\footnote{It is possible to allow for a different initial policy $\pi^0_i$ at each observation $i$. One then needs to define the relative weights 
$r(D_i^e) \coloneqq  {\pi^S(A_i^e|X_i^e)}/{\pi_i^{0}(A_i^e| X_i^e)}$.}
For all $e\in\mathcal{E}^{\obs}$ and all $i\in\{1,\ldots,n_e\}$ define the relative weights as
    \begin{align}\label{eq:relative-weights-new}
        r(D_i^e) \coloneqq  \frac{\pi^S(A_i^e|X_i^e)}{\pi^{0}(A_i^e| X_i^e)}.
    \end{align}
Then, for all $e\in\mathcal{E}^{\obs}$, we draw a weighted resample $D^{e,\pi^S} \coloneqq (D^e_{i_1}, \ldots, D^e_{i_{m_e}})$ of size $m_e$ from $D^e$ with weights
    \begin{align}
        w^e_{i_1, \ldots, i_{m_e}} \coloneqq \begin{cases}\frac{\prod_{\ell=1}^{m_e} r(D_{i_\ell}^e)}{\sum_{\substack{(j_1, \ldots, j_{m_e})\\ \text{ distinct}}} \prod_{\ell=1}^{m_e} r(D_{j_\ell}^e)} & {(i_1, \ldots, i_{m_e})} \text{ distinct} \\
        0 & \text{otherwise.}
        \end{cases}
        \label{eq:SIR-joint-weights-new}
    \end{align}
We then apply an invariance test to the resampled data $D^{e_1, \pi^S}, \ldots, D^{e_L, \pi^S}$. A family of invariance tests $\{\varphi^S\}_{S \subseteq \{1, \ldots, d\}}$ is a collection of functions such that for each $S$, $\varphi^S$ is a function (into $\{0,1\}$) that takes data from environments $e_1, \ldots, e_L$, each of size $m_{e_i}$, and tests whether $S$ is invariant. Here, $\varphi^S = 1$ indicates that we reject the hypothesis of invariance. We say the test has pointwise asymptotic level if for all  invariant sets $S$ and all $\pi\in\Pi^S$ it holds that
\begin{equation*}
    \limsup_{\min\{m_{e_1}, \ldots, m_{e_L}\}\rightarrow\infty}\P^{\pi}(\varphi^S(D^{e_1, \pi^S},\ldots,D^{e_L, \pi^S})=1)\leq\alpha.
\end{equation*}
We state that the overall procedure (resampling and then testing) has asymptotic level as long as the test $\varphi^S$ has asymptotic level.
For simplicity, we assume that $n_{e_1}=\cdots=n_{e_L}\eqqcolon n$ and $m_{e_1}=\cdots=m_{e_L} \eqqcolon m$. The following result follows directly from \cite[Theorem~1]{thams2021statistical}%
\begin{proposition}\label{thm:SIR-consistency-new}
    Let $S\subseteq \{1, \ldots, d\}$ and suppose that for each environment $e_1, \ldots, e_L$, we observe a dataset $D^e$ consisting of $n$ observations $D_i^e = (X_i^e, A_i^e, R_i^e, \pi^0(A_i^e\mid X_i^e))$. 
    Consider $\pi^S \in \Pi^S$ and assume that for all $e \in \mathcal{E}$, $\EX^{\pi^0}[r(D^e_i)^2] < \infty$, where $r$ is defined in \eqref{eq:relative-weights-new}. 
    Let $m = o(\sqrt{n})$ 
    and for all $e$, let $D^{e,\pi^S} \coloneqq (D^e_{i_1}, \ldots, D^e_{i_m})$ be a resample of $D^e$ drawn with weights given by \eqref{eq:SIR-joint-weights-new}. 
    Let $\varphi^S$ be a hypothesis test 
    for invariance of the conditional expectation $\EX^{\pi^S,e}[R\mid X^S]$ that has pointwise asymptotic level $\alpha \in (0,1)$ when $\varphi^S$ is applied to data sampled with $\pi^S$. Applying $\varphi^S$ to the resampled data yields pointwise asymptotic level, that is,
    \begin{align*}
        \limsup_{n\to\infty}\P^{\pi^0}(\varphi^S(D^{e_1, \pi^S}, \ldots, D^{e_L, \pi^S}) = 1) \leq \alpha
    \end{align*}    
    if $S$ is invariant.
\end{proposition}
\begin{proof}
    We only show that this problem with environments can be cast in the setting of \cite{thams2021statistical}, which has no reference to environments. Here, we assume that we have the same number of observations in each environment. The main idea is to create a dataset $D^{\mathcal{E}}$, such that each observation in $D^{\mathcal{E}}$ consists of an observation from each of the environments $D^e$. 
    
    First, we randomly permute the observations within each dataset $D^e$ to obtain a set $\tilde{D}^e$. Then, we construct an auxiliary dataset $D^{\mathcal{E}}$, where the $i$'th observation $D_i^{\mathcal{E}}$ of $D^{\mathcal{E}}$ is the concatenation of the $i$'th observation (after permutation) from each of the environments, $D_i^{\mathcal{E}} \coloneqq (\tilde{D}_i^{e_1}, \ldots, \tilde{D}_i^{e_L})$. 
    
    We can now apply the resampling methodology from \cite{thams2021statistical} to draw a sequence $(D_{i_1}^\mathcal{E}, \ldots, D_{i_m}^\mathcal{E})$ with weights given by 
    \begin{align*}
        w^{\mathcal{E}}_{i_1, \ldots, i_{m}} \coloneqq \begin{cases}\frac{\prod_{\ell=1}^{m} r(D_{i_\ell}^{\mathcal{E}})}{\sum_{\substack{(j_1, \ldots, j_{m})\\ \text{ distinct}}} \prod_{\ell=1}^{m} r(D_{j_\ell}^{\mathcal{E}})} & {(i_1, \ldots, i_{m})} \text{ distinct} \\
        0 & \text{otherwise.}
        \end{cases}
    \end{align*}
    where
    \begin{align*}
        r(D_i^{\mathcal{E}}) \coloneqq \frac{\pi^S(\tilde{A}_i^{e_1} \mid \tilde{X}_i^{e_1})}{\pi^0(\tilde{A}_i^{e_1} \mid \tilde{X}_i^{e_1})} \cdots 
        \frac{\pi^S(\tilde{A}_i^{e_L} \mid \tilde{X}_i^{e_L})}{\pi^0(\tilde{A}_i^{e_L} \mid \tilde{X}_i^{e_L})},
    \end{align*}
    and 
    $\tilde{X}_i^{e}, \tilde{A}_i^{e}$ are the $i$'th observation of $\tilde{D}^e$.
    Because the observations are independent, both within and between environments, the probability of drawing the sample
    $(D_{i_1}^\mathcal{E}, \ldots, D_{i_m}^\mathcal{E}) = ((D_{i_1}^{e_1}, \ldots, D_{i_1}^{e_L}), \ldots,  (D_{i_m}^{e_1}, \ldots, D_{i_m}^{e_L}))$ is equal to the probability of drawing first $m$ observations from $e_1$, $(D_{i_1}^{e_1}, \ldots, D_{i_m}^{e_1})$, and then $m$ from $e_2$ etc. The result then follows directly from \cite{thams2021statistical}.
\end{proof}
In other words, we can test whether $S$ is invariant by resampling the data and 
applying an invariance test on the resampled dataset. Proposition~\ref{thm:SIR-consistency-new} states that this procedure holds level asymptotically. We assume knowledge of the initial policy $\pi^0$ to ease our presentation. We can, in fact, show the pointwise asymptotic validity even if the initial policy $\pi^0$ is unknown and has to be estimated from the offline data (see \cite{thams2021statistical} Theorem 2).

\section{Algorithm: Off-policy Invariant Causal Prediction}\label{app:alg:off_icp}
Below, we present an algorithm for finding the causal ancestors $\AN(R)$ of the reward $R$ under a change in policy.

\begin{algorithm}[!ht]
\caption{Off-policy Invariant Causal Prediction}
\label{alg:off_icp}
\SetAlgoLined
\KwIn{data $D = (D^{e_1}, \ldots, D^{e_L})$, hypothesis tests and test policies $\{(\psi^S, \pi^S)\}_{S \subseteq \{1,\dots,d\}}$
}
initialize the collection of invariant sets  $\mathbf{S}_{\inv} \leftarrow \{\}$\;
\tcp{loop over all subsets}
\For{$S \in \mathcal{P}(\{1,\dots,d\})$}{
    \tcp{test for invariance}
    $\var{is\_inv} \leftarrow \var{test\_inv}(D, \pi^S,\psi^S, S)$ \;
        \tcp{(see Algorithm~\ref{alg:test_inv})}
    \tcp{update the accepted invariant set}
    \If{\var{is\_inv}}{
        add $S$ to $\mathbf{S}_{\inv}$
    }
}
\tcp{get the estimated causal ancestors}
$\hat{S}_{\AN} \leftarrow \bigcap_i \mathbf{S}_{\inv}[i]$ \;
\KwOut{the estimated causal ancestors $\hat{S}_{\AN}$}
\end{algorithm}

\section{Faster power optimization}\label{sec:power-opt-REPL}
In Section~\ref{sec:optimizing-power}, we show that we can optimize the power to detect non-invarince by gradient descent. In particular, the gradient is
$$
\nabla J(\theta) = \EX\big[ \nabla \log \P(D^{\pi_\theta^S} \mid D) \var{pv}(D^{\pi_\theta^S})\big],
$$
where $D^{\pi_\theta^S}$ is a resample of the data $D$ and $\var{pv}$ is a function returning a p-value of our invariance test. $\P(D^{\pi_\theta^S} \mid D)$ is given by Equation~\eqref{eq:SIR-joint-weights-new}, but as discussed in \cite{thams2021statistical}, this may be infeasible to compute if $n$ is very large. 

As a computationally efficient alternative,
\cite{thams2021statistical} proposes 
an approximate
resampling scheme, where a sequence $(i_1, \ldots, i_{m_e})$ (distinct or non-distinct) is sampled with replacement. That is, the weights are given by
\begin{align*}
    w_{\theta,(i_1, \ldots, i_{m_e})} &\coloneqq \frac{\prod_{\ell=1}^{m_e} r_\theta(D_{i_\ell}^e)}{\sum_{(j_1, \ldots, j_m)}\prod_{\ell=1}^{m_e} r_\theta(D_{j_\ell}^e)} \\
    &= \frac{\prod_{\ell=1}^{m_e} r_\theta(D_{i_\ell}^e)}{\left(\sum_{j=1}^{n_e} r_\theta(D_{j}^e)\right)^{m_e}}.
\end{align*}
This expression is much easier to compute than Equation~\eqref{eq:SIR-joint-weights-new}, because the denominator is a sum over $n_e$ terms (instead of $n_e!/(n_e - m_e)!$). In particular, we get
\begin{align*}
    \nabla_\theta \log \P(D^{\pi_\theta^S} \mid D) 
    &= \nabla_\theta \log w_{\theta,(i_1, \ldots, i_{m_e})} \\
    &= \sum_{\ell=1}^{m_e}\nabla_\theta \log r_\theta(D_{i_\ell}^e) - m_e \nabla_\theta \log \sum_{j=1}^{n_e} r_\theta(D_{j}^e).
\end{align*}
Algorithm~\ref{alg:power_opt} splits the data in two halves: we optimize power on the first half of the data and test for invariance on the second half. 
We only use the above approximation for the power optimization, where we need to explicitly compute the normalization constant of the weights. 
    In the second half of Algorithm~\ref{alg:power_opt}, we use Equation~\eqref{eq:SIR-joint-weights-new} (i.e., 
we do not use the 
approximate weights), because
Proposition~\ref{thm:SIR-consistency-new} requires the weights to be those given in Equation~\eqref{eq:SIR-joint-weights-new}. 
If $n$ is so large that we cannot sample by explicitly computing the weights Equation~\eqref{eq:SIR-joint-weights-new}, there are several options for sampling from the scheme without computing the denominator -- see \cite{thams2021statistical} for a variety of approaches.

\section{Invariance test with optimized test policy}\label{app:algorithm-power-opt}
In this section we provide Algorithm~\ref{alg:power_opt}, which tests the invariance of a set by choosing a test policy $\pi^S$ that optimizes the power of the invariance test, as discussed in Section~\ref{sec:optimizing-power}. 
\begin{algorithm}[t!]
\caption{Testing the invariance of a set $S$ with optimization over test policies $\pi^S$}
\label{alg:power_opt}
\SetAlgoLined
\SetKwFunction{invcond}{$\var{test\_inv\_opt\_}\pi$}
\SetKwProg{Fn}{Function}{:}{}
\Fn{\invcond{data $D = (D^{e_1}, \ldots, D^{e_L})$, 
function $\var{pv}$ yielding the p-value of an invariance test, target set $S$, 
learning rate $\gamma$=1e-3, significance level $\alpha$}}{
\tcp{sample splitting}
 \For{$e = e_1, \ldots, e_L$}{
    $n_{e, sp} \leftarrow \operatorname{ceil}(\abs{D^e}/2)$ \;
    $D^{e,1} \leftarrow  \{(x_{i}^e, a_i^e, r_i^e, \pi^0(a_i^e|x_i^e))\}_{i=1}^{n_{e, sp}}$ \;
    $D^{e,2} \leftarrow  \{(x_{i}^e, a_i^e, r_i^e, \pi^0(a_i^e|x_i^e))\}_{i=n_{e,sp}+1}^{\abs{D^e}}$\;
 }
 \tcp{optimizing power}
 Initialize policy parameters $\theta$ \;
 \While{not converged}{
    \For{$e = e_1, \ldots, e_L$}{
        \For{$i=1$ \KwTo $n_{e,sp}$}{ 
        compute weights: $r^e_i \leftarrow \dfrac{\pi_{\theta}^S(a_i^e \mid x_i^{e,S})}{\pi^{0}(a_i^e\mid x_i^e)}$ \;
        }
        choose resampling size $m_e$ with GOF-heuristic in \cite{thams2021statistical} \;
        draw $D^{e,\pi_\theta^S} \coloneqq (D^{e,1}_{i_1}, \ldots, D^{e,1}_{i_{m_e}})$ with replacement from $D^{e,1}$ with probabilities $\propto r_i^e$ \;
    }
    $D^{1, \pi_\theta^S} \leftarrow (D^{e_1, \pi_\theta^S},\ldots, D^{e_L, \pi_\theta^S})$\;
    compute p-value: $\var{pv}(D^{1, \pi_\theta^S})$ \;
    compute gradient: $\nabla \log \P(D^{1, \pi_\theta^S})$ %
    update policy parameters: $\theta \leftarrow \theta - \gamma \var{pv}(D^{1, \pi_\theta^S}) \nabla \log \P(D^{1, \pi_\theta^S})$ \;
    }
 \tcp{verifying invariance condition}
 \For{$e = e_1, \ldots, e_L$}{
  \For{$i=n_{e,sp}+1$ \KwTo $\abs{D^e}$}{ 
    compute weights: $r_i^e \leftarrow \dfrac{\pi_{\theta}^S(a_i^e \mid x_i^{e,S})}{\pi^{0}(a_i^e \mid x_i^e)}$ \;
    }
    choose resampling size $m_e$ with GOF-heuristic in \cite{thams2021statistical} \;
    draw $D^{e, \pi_\theta^{S}} \coloneqq (D^{e, 2}_{i_1}, \ldots, D^{e,2}_{i_{m_e}})$ with replacement from $D^{e,2}$ with probabilities $\propto r_i^e$ \;
 }
 $D^{2, \pi_\theta^S} \leftarrow (D^{e_1, \pi_\theta^S},\ldots, D^{e_L, \pi_\theta^S})$\;
 $\var{is\_invariant} \leftarrow \var{pv}(D^{2, \pi_\theta^S}) \geq \alpha$ \;
 \KwRet\ $\var{is\_invariant}$
}
\end{algorithm}

\section{Simulation Details}
\label{app:sim}
\subsection{Data Generating Process}
\label{app:sim:datagen}
We generate data from the following SCM $\mathcal{S}(\pi, e)$:
\begin{equation*}
\begin{gathered}
    U \coloneqq \epsilon_U, \quad
    X^1 \coloneqq \gamma_e U + \epsilon_{X^1}, \quad
    X^2 \coloneqq \alpha_e + \epsilon_{X^2}, \\
    A \sim \pi(A \mid X^1, X^2), \quad
    R \coloneqq \beta_{A,1}X^2 + \beta_{A,2}U + \epsilon_R,
\end{gathered}
\end{equation*}
where $\epsilon_U, \epsilon_{X^2}, \epsilon_{X^1}, \epsilon_{R} \sim \mathcal{N}(0,1)$, $A$ takes values in the space $\{a_1, \ldots, a_L\}$. In our experiments, we consider $L=3$ and randomly draw the parameters $\beta_{a_1, 1}, \dots, \beta_{a_3, 1}, \beta_{a_1, 2}, \dots, \beta_{a_3, 2}$ from $\mathcal{N}(0,1)$, while the environment-specific parameters $\gamma_e, \alpha_e$ are drawn from $\mathcal{N}(0,4)$. These parameters are then fixed across all experiment runs.
\subsection{Initial Policy}
\label{app:sim:init_policy}
We construct an initial policy $\pi^0$ in Section~\ref{sec:simulation} as follows. First, we generate a training data $D \coloneqq \{(X^1_i, X^2_i, A_i, R_i, e_i)\}_{i=1}^n$ from the uniform random policy and partition the dataset $D$ according to the action values: $D_{a_1}$, \dots, $D_{a_L}$. Then, for each action $a \in \{a_1, a_2, a_3\}$, we fit a linear regression on $D_a$ to estimate the reward $R$ from $X^1$ and $X^2$. Denote the resulting regressor as $f_a$. The initial policy is then constructed as
\begin{equation*}
    \pi^0(A = a \mid X^1, X^2) \propto \exp{\frac{1}{2}f_a(X^1, X^2)}.
\end{equation*}

\subsection{Invariant Test with True Conditional Expectation}
\label{app:invariant_test_true_conditional}
This section contains Figure~\ref{fig:invariant_test_true_conditional}, in which we display acceptance rates for the same experiment as in Section~\ref{sec:experiment_learning_invariant_policies} but with an exact test, using the true conditional expectation. The figure suggests that the procedure indeed holds level.
\begin{figure}[!ht]
    \centering
    \includegraphics[scale=.45]{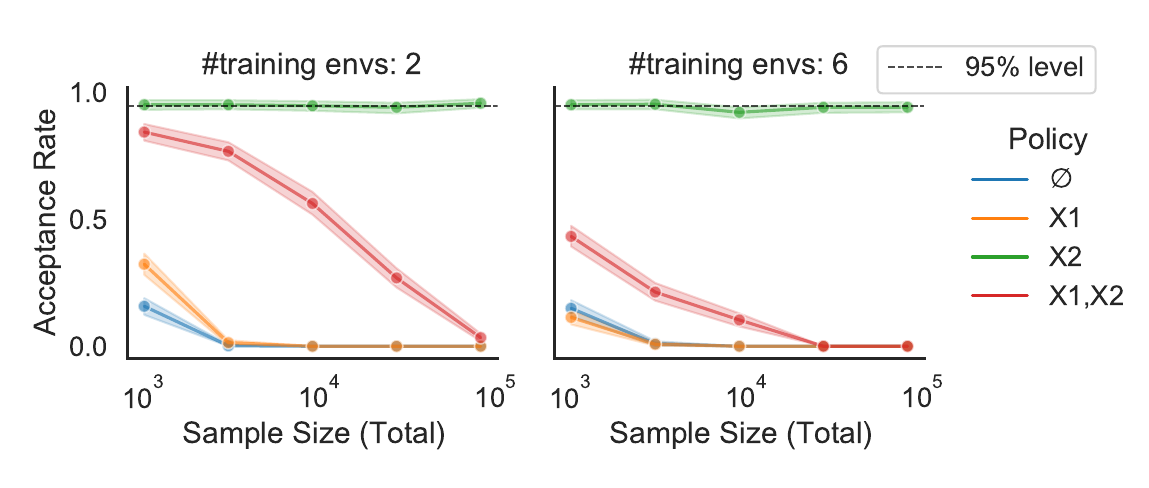}
    \caption{Acceptance rates for the off-policy invariance test with true conditional expectation. 
    }
    \label{fig:invariant_test_true_conditional}
\end{figure}

\section{Warfarin Case Study}
\label{app:warfarin}
\subsection{Initial Policy}
\label{app:warfarin:init_policy}
We generate the training data $\{(X_i, A_i, R_i, e_i)\}_{i=1}^{n}$, where $e_i \in \mathcal{E} = \{1,\dots,4\}$ under the following initial policy. We fit a linear regression to estimate the optimal warfarin dose from BMI score. Let us denote the resulting regressor by $f^{\text{BMI}}$. The initial policy $\pi^0$ then selects actions according to the following (unnormalized) distribution:
\begin{equation*}
    \pi^0(A = a \mid X^\text{BMI}) \propto \exp{\frac{1}{2}\abs{f^{\text{BMI}}(X^\text{BMI}) - m(a)}^{-1}},
\end{equation*}
where, as before, $m(a)$ denotes a median value of the optimal warfarin doses within the bucket $a$.
\subsection{Defining Sets}
\label{app:warfarin:defining_set}
The resulting defining set is \{Race, VKORC1\}. The following are the details of these variables (see also \cite{international2009estimation}):
 \begin{itemize}
     \item VKORC1: Genetic information -- vitamin K epoxide reductase complex, subunit 1.
     \item Race: Racial categories as defined by the U.S. Office of Management and Budget.
 \end{itemize}
 
\begin{figure}[!t]
    \centering
    \includegraphics[scale=.45]{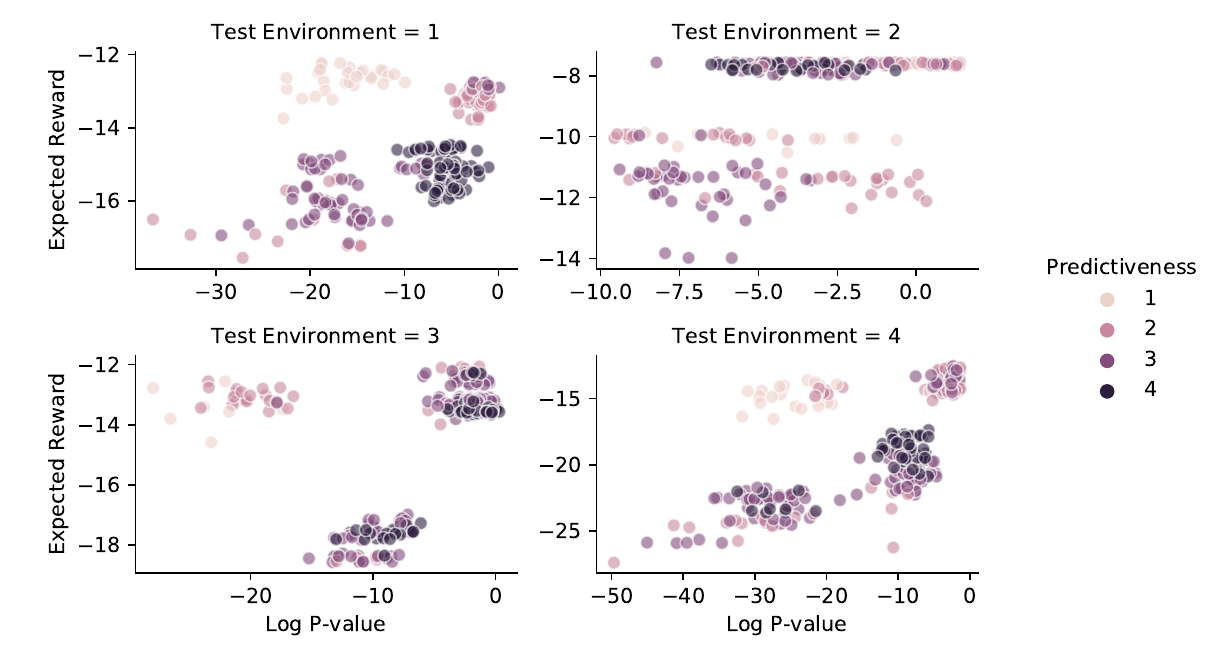}
    \caption{Analysis on the generalization performance and the degree of invariance. The y-axes represent the expected reward of policies with different subsets, while the x-axes represent their corresponding p-values return from the invariance test. The result shows that a policy that depends on a subset with a higher degree of invariance is more likely to generalize better to a new environment.}
    \label{fig:warfarin:p-value}
\end{figure}

\subsection{P-value and Generalization Analysis}
In the semi-real experiment (see Section~\ref{sec:warfarin:semi-real}), we further analyze the generalization performance of each candidate set and its corresponding p-value returned by the invariance test. To distinguish the effects of invariance and predictiveness on the generalization performance (measured by the expected reward on a test environment), we partition the subsets into four groups depending on their performance on the training environments (1 is the least predictive and 4 is the most predictive).

Within each predictiveness group, the scatter plots in Figure~\ref{fig:warfarin:p-value} display a correlation between the p-value returned by the invariance test and the expected reward under a test environment. This result indicates that a policy that depends on a subset with a higher degree of invariance (higher p-value) tends to generalizes better to a new environment. The correlation is strongest in the test environment $e = 4$ in which we could also observe the largest performance gap between invariant and non-invariant approaches, see Figure~\ref{fig:warfarin_hidden}. 

\section{}
\label{app:counter_example}
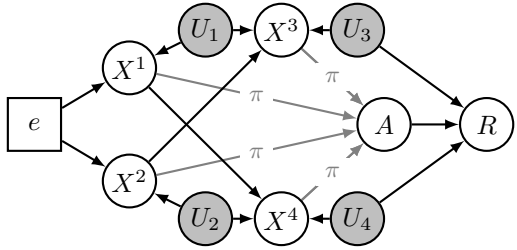
\begin{figure}[!ht]
    \centering
    \input{figs/fig_counter_example}
\caption{Example setting illustrating that Assumption~\ref{assm:U_causes_R} is required to derive the theoretical results in Proposition~\ref{prop:inv_policy_NEW_2} and Theorem~\ref{thm:inv_policy_2}.}
\label{fig:counter_example}
\end{figure}
We now discuss an example (presented in Figure~\ref{fig:counter_example}) that justifies Assumption~\ref{assm:U_causes_R}.
In this example, the variables $U_1$ and $U_2$ influence only the observed covariates $X$ but not the reward $R$. This example would lead to the following problems in Proposition~\ref{prop:inv_policy_NEW_2} and Theorem~\ref{thm:inv_policy_NEW}. 

First, the subsets $\{X^1, X^4\}$ and $\{X^2, X^3\}$ are both $d$-invariant, but no set of size $3$ or more is $d$-invariant. By symmetry, there is no guarantee that a $d$-invariant set that is optimal in the training environments will also be optimal in a new test environment because e.g. $\{X^1, X^4\}$ might be optimal on the training data while $\{X^2, X^3\}$ is optimal on the test data. This then refutes the statement in  Proposition~\ref{prop:inv_policy_NEW_2}. Assumption~\ref{assm:U_causes_R} fixes this problem as it ensures the existence of a largest $d$-invariant set which is a superset of all other $d$-invariant sets (see the proof of Lemma~\ref{lemma:stable_invariant}), and rules out this example. 

Second, there is no strongly non-$d$-invariant variable (see Definition~\ref{def:conf-removing}) in this example and hence Assumption~\ref{assm:strong_envs} does not guarantee the existence of a confounding removing environment. This implies that a set $\mathcal{E}$ of environments can be arbitrary, for instance, it could be a singleton $\mathcal{E}=\{e\}$. In that case,  Theorem~\ref{thm:inv_policy_2} would no longer hold (but Theorem~\ref{thm:inv_policy_1} remains valid). We therefore require Assumption~\ref{assm:U_causes_R} for proving the results of Proposition~\ref{prop:inv_policy_NEW_2} and the second statement of Theorem~\ref{thm:inv_policy_2}.
\end{appendices}

\end{document}

%% file: commands.tex
\usepackage{amsmath, mathtools}
\usepackage{centernot}
\usepackage{amssymb}
\usepackage{graphicx,subfigure}
\usepackage{amsthm}
\usepackage{dsfont}
\usepackage{bbm}
\usepackage{tabu}
\usepackage{paralist}
\usepackage{enumitem}
\usepackage{physics}
\usepackage{caption}
\usepackage{accents}
\usepackage[ruled,vlined]{algorithm2e}
\SetKwComment{Comment}{$\triangleright$\ }{}
\usepackage[english]{babel}
\usepackage[utf8x]{inputenc}
\usepackage{amsfonts}
\usepackage{booktabs}
\usepackage{tabu}
\usepackage[T1]{fontenc}
\usepackage{mathrsfs}
\usepackage{stmaryrd}
\usepackage{tikz}
\usetikzlibrary{backgrounds}
\pgfdeclarelayer{bg1}
\pgfdeclarelayer{bg2}
\pgfsetlayers{bg2, bg1, main}
\usepackage{xcolor}
\usetikzlibrary{arrows.meta}

\usepackage[colorinlistoftodos]{todonotes}
\usepackage[colorlinks=true, allcolors=blue]{hyperref}

\newtheorem{theorem}{Theorem}

\newtheorem{lemma}{Lemma}
\newtheorem{assumption}{Assumption}

\newtheorem{example}{Example}
\newtheorem{proposition}{Proposition}

\newtheorem{definition}{Definition}
\newtheorem{setting}{Setting}

\theoremstyle{definition}

\newtheorem{remark}{Remark}

\newcounter{subassumption}[asu]

\makeatletter
\renewcommand{\p@subassumption}{\theasu}%
\makeatother

\newcommand\Item[1][]{%
  \ifx\relax#1\relax  \item \else \item[#1] \fi
  \abovedisplayskip=0pt\abovedisplayshortskip=0pt~\vspace*{-\baselineskip}}
\let\emptyset\varnothing
\newcommand{\inv}{\operatorname{inv}}
\newcommand{\opt}{\operatorname{opt}}
\newcommand{\obs}{\operatorname{obs}}
\renewcommand{\tr}{\operatorname{tr}}
\newcommand{\tst}{\operatorname{tst}}

\newcommand{\I}{\operatorname{I}}
\newcommand{\SNI}{\operatorname{SNI}}

\renewcommand{\var}{\texttt}
\newcommand{\mathcalbf}[1]{\boldsymbol{\mathcal{#1}}}
\newcommand{\ci}{\mathrel{\perp\mspace{-10mu}\perp}}
\newcommand\numberthis{\addtocounter{equation}{1}\tag{\theequation}}
\newcommand{\nci}{\centernot{\ci}\hspace{-2pt}}
\DeclareMathOperator{\EX}{\mathbb{E}}%
\DeclareMathOperator*{\argmax}{arg\,max}
\DeclareMathOperator*{\argmin}{arg\,min}

\newcommand{\simind}{\overset{\text{ind.}}{\sim}}

\renewcommand{\P}{\mathbb{P}}

\DeclareMathOperator*{\PA}{PA}
\DeclareMathOperator*{\AN}{AN}
\DeclareMathOperator*{\DE}{DE}
\DeclareMathOperator*{\MB}{MB}
\DeclareMathOperator*{\doo}{do}

%% file: figs/fig_setting1.tex
\begin{tikzpicture}[node distance=1.5cm,
      thick, roundnode/.style={circle, draw, minimum size=7mm}, squarenode/.style={rectangle, draw, minimum size=7mm}]
    \node[roundnode, fill=white!100] (X) {$X$};
    \node[roundnode, fill=black!25] (U) [above right=0.6 and 1.1 of X] {$U$};
    \node[squarenode] (E) [left=0.8 of X]{$e$};
    \node[roundnode] (A) [below right=0.6 and 1.2 of X] {$A$};
    \node[roundnode] (R) [right=2.4 of X] {$R$};
    \begin{pgfonlayer}{bg1}
        \node[roundnode,thick] (X1) [right=0.9 of E] {$ $};
        \node[roundnode,thick,fill=black!25] (U1) [above right=0.6 and 1.1 of X1] {$ $};
    \end{pgfonlayer}
    \begin{pgfonlayer}{bg2}
        \node[roundnode,thick] (X2) [right=1.0 of E] {$ $};
        \node[roundnode,thick,fill=black!25] (U2) [above right=0.6 and 1.1 of X2] {$ $};.
    \end{pgfonlayer}
    \draw[-latex,thick] (E) -- (X);
    \draw[-latex,thick] (E) edge[bend left=20] (U);
    \draw[-latex,thick] (X2) -- (R);
    \draw[-latex,thick] (X2) -- (A);
    \draw[-latex,thick] (X2) edge[bend right] (U);
    \draw[-latex,thick] (U) edge[bend right] (X2);
    \path[-latex,draw,thick,gray] (X2) edge node[fill=white] {$\pi$} (A);
    \draw[-latex,thick] (A) -- (R);
    \draw[-latex,thick] (U2) -- (R);
    
\end{tikzpicture}

%% file: figs/fig_example1a.tex
 \begin{tikzpicture}[node distance=1.5cm,
    thick, roundnode/.style={circle, draw, inner sep=1pt,minimum size=7mm}, squarenode/.style={rectangle, draw, inner sep=1pt, minimum size=7mm}]
    \node[roundnode] (X1) at (-1.5, 1.5) {$X^1$};
    \node[roundnode] (X2) at (-1.5, 0){$X^2$};
    \node[roundnode][fill=black!25] (U) at (1.25,1.5) {$U$};
    \node[squarenode] (E) at (-2.75, 0){$e$};
    \node[roundnode] (A) at (0, 0) {$A$};
    \node[roundnode] (R) at (1.25, 0) {$R$};
    \draw[-latex] (E) edge (X2);
    \draw[-latex] (E) -- (X1);
    \draw[-latex] (X2) edge[bend right=30] (R);
    \begin{pgfonlayer}{bg1}
        \path[-latex,draw,thick, gray] (X2) edge
        node[fill=white] {$\pi$} (A);
        \path[-latex,draw,thick, gray] (X1) edge node[fill=white] {$\pi$} (A);
    \end{pgfonlayer}
    \draw[-latex, thick] (U) edge (X1);
    \draw[-latex, thick] (A) -- (R);
    \draw[-latex, thick] (U) edge (R);
\end{tikzpicture}

%% file: figs/fig_example1b.tex
\begin{tikzpicture}[node distance=1.5cm,
    thick, roundnode/.style={circle, draw, inner sep=1pt,minimum size=7mm}, squarenode/.style={rectangle, draw, inner sep=1pt, minimum size=7mm}]
    \node[roundnode] (X1) at (-1.5, 1.5) {$X^1$};
    \node[roundnode] (X2) at (-1.5, 0){$X^2$};
    \node[roundnode] (X3) at (0, 1.5) {$X^3$};
    \node[roundnode][fill=black!25] (U) at (1.5,1.5) {$U$};
    \node[squarenode] (E) at (-2.75, 0){$e$};
    \node[roundnode] (A) at (0, 0) {$A$};
    \node[roundnode] (R) at (1.5, 0) {$R$};
    \draw[-latex] (E) edge (X2);
    \draw[-latex] (E) -- (X1);
    \draw[-latex] (X2) edge[bend right=30] (R);
    \draw[-latex] (U) -- (X3);
    \draw[-latex] (X3) -- (R);
    \begin{pgfonlayer}{bg1}
        \path[-latex,draw,thick,gray] (X2) edge
        node[fill=white] {$\pi$} (A);
        \path[-latex,draw,thick,gray] (X1) edge node[fill=white] {$\pi$} (A);
        \path[-latex,draw,thick,gray] (X3) edge node[fill=white] {$\pi$} (A);
    \end{pgfonlayer}
    \draw[-latex, thick] (U) edge[bend right=30] (X1);
    \draw[-latex, thick] (A) -- (R);
    \draw[-latex, thick] (U) edge (R);
\end{tikzpicture}

%% file: figs/fig_counter_example.tex
\begin{tikzpicture}[node distance=1.5cm,
    thick, roundnode/.style={circle, draw, inner sep=1pt,minimum size=7mm}, squarenode/.style={rectangle, draw, inner sep=1pt, minimum size=7mm}]
    \node[roundnode] (X1) at (-1.5, 1.5) {$X^1$};
    \node[roundnode] (X2) at (-1.5, 0){$X^2$};
    \node[roundnode] (X3) at (.5, 2) {$X^3$};
    \node[roundnode][fill=black!25] (U1) at (-.5, 2) {$U_1$};
    \node[roundnode][fill=black!25] (U2) at (-.5, -.5) {$U_2$};
    \node[roundnode][fill=black!25] (U3) at (1.5, 2) {$U_3$};
    \node[roundnode][fill=black!25] (U4) at (1.5, -.5) {$U_4$};
    \node[squarenode] (E) at (-2.75, 0.75){$e$};
    \node[roundnode] (X4) at (0.5, -0.5) {$X^4$};
    \node[roundnode] (A) at (1.85, 0.75) {$A$};
    \node[roundnode] (R) at (3.2, 0.75) {$R$};
    \draw[-latex] (E) -- (X2);
    \draw[-latex] (E) -- (X1);
    \draw[-latex] (U1) -- (X1);
    \draw[-latex] (U1) -- (X3);
    \draw[-latex] (U2) -- (X2);
    \draw[-latex] (U2) -- (X4);
    \draw[-latex] (U3) -- (X3);
    \draw[-latex] (U3) -- (R);
    \draw[-latex] (U4) -- (X4);
    \draw[-latex] (U4) -- (R);
    \draw[-latex] (X1) -- (X4);
    \draw[-latex] (X2) -- (X3);
    \draw[-latex] (A) -- (R);
    \begin{pgfonlayer}{bg1}
    \path[-latex,draw,thick, gray] (X1) edge
        node[fill=white] {\small $\pi$} (A);
    \path[-latex,draw,thick, gray] (X2) edge
        node[fill=white] {\small $\pi$} (A);
    \path[-latex,draw,thick, gray] (X3) edge (A);
    \path[-latex,draw,thick, gray] (X4) edge (A);
    \path[-latex,draw,thick, gray] (X3) edge node[fill=white] {\small $\pi$}  (A);
    \path[-latex,draw,thick, gray] (X4) edge node[fill=white] {\small $\pi$} (A);
    \end{pgfonlayer}
\end{tikzpicture}